\documentclass[accepted]{uai2026} 
                        
\usepackage[american]{babel}

\usepackage{natbib} 
\usepackage{mathtools} 
\usepackage{booktabs} 
\usepackage{tikz} 
\usepackage{bm}
\usepackage{amssymb}
\usepackage{amsmath}
\usepackage{amsthm}
\usepackage{enumitem}
\usepackage{subcaption}
\usepackage{graphicx}
\usepackage{hyperref}

\usepackage[capitalize,noabbrev,nameinlink]{cleveref}
\hypersetup{linktoc=all}

\newlist{assumplist}{enumerate}{1}
\setlist[assumplist]{label=\textbf{\arabic*.}, ref=\theassumption.\arabic*, leftmargin=1.2em, itemsep=1pt, parsep=0pt, topsep=2pt}

\crefname{assumplisti}{Assumption}{Assumptions}
\theoremstyle{plain}
\newtheorem{theorem}{Theorem}[section]
\newtheorem{proposition}{Proposition}
\newtheorem{lemma}{Lemma}
\newtheorem{corollary}{Corollary}
\theoremstyle{definition}
\newtheorem{definition}{Definition}
\newtheorem{assumption}{Assumption}
\theoremstyle{remark}
\newtheorem{remark}{Remark}

\makeatletter

\newcommand{\printsupplementarytoc}{%
  \section*{Contents of the Supplementary Material}%
  \begingroup
    \setcounter{tocdepth}{2}%
    \setlength{\parindent}{0pt}%
    \setlength{\parskip}{0pt}%
    \renewcommand{\@pnumwidth}{2.25em}%
    \renewcommand{\@tocrmarg}{3.0em}%
    \renewcommand{\@dotsep}{3.2}%
    \renewcommand*{\l@section}[2]{%
      \ifnum \c@tocdepth >\z@
        \addpenalty{\@secpenalty}%
        \addvspace{0.55em plus 1pt}%
        \@tempdima 2.3em\relax
        \begingroup
          \parindent \z@
          \rightskip \@pnumwidth
          \parfillskip -\@pnumwidth
          \leavevmode
          \bfseries
          \advance\leftskip\@tempdima
          \hskip -\leftskip
          ##1\nobreak
          \leaders\hbox{$\m@th\mkern \@dotsep mu.\mkern \@dotsep mu$}\hfill
          \nobreak\hb@xt@\@pnumwidth{\hfil ##2}\par
        \endgroup
      \fi}%
    \renewcommand*{\l@subsection}{\@dottedtocline{2}{2.3em}{3.4em}}%
    \renewcommand*{\l@subsubsection}{\@dottedtocline{3}{5.7em}{4.2em}}%
    \begin{center}%
      \begin{minipage}{0.92\linewidth}%
        \noindent\rule{\linewidth}{0.45pt}\par
        \@starttoc{stoc}%
        \noindent\rule{\linewidth}{0.45pt}\par
      \end{minipage}%
    \end{center}%
  \endgroup
}

\newcommand{\suppsection}[1]{%
  \section{#1}%
  \addcontentsline{stoc}{section}{\protect\numberline{\thesection}#1}%
}

\newcommand{\suppsubsection}[1]{%
  \subsection{#1}%
  \addcontentsline{stoc}{subsection}{\protect\numberline{\thesubsection}#1}%
}

\makeatother

\crefname{assumption}{Assumption}{Assumptions}

\title{Provable Subspace Identification of Nonlinear Multi-view CCA}

\author[1]{Zhiwei Han}
\author[1]{Stefan Matthes}
\author[1,2]{Hao Shen}
\affil[1]{%
    Chair of Data Processing\\
    Technical University of Munich\\
    Germany
}
\affil[2]{%
    fortiss GmbH\\
    Munich\\
    Germany
}
  
\begin{document}
\maketitle

\begin{abstract}
We investigate the identifiability of nonlinear canonical 
correlation analysis (CCA) in a multi-view setup, in which 
each view is generated by applying an unknown nonlinear map 
to a linear mixture of shared latent variables plus view-private 
noise. Rather than pursuing exact unmixing, which is known to be 
ill-posed under general nonlinear mixing, we instead reframe 
multi-view CCA as a basis-invariant subspace identification problem.
Under suitable latent priors and spectral separation conditions, we prove 
that the pairwise population CCA objective recovers correlated 
signal subspaces up to view-wise orthogonal ambiguity. For $N \geq 3$ 
views, their multi-view aggregation provably isolates the jointly 
correlated subspaces shared across all views while eliminating view-private variation.
We further establish finite-sample statistical consistency guarantees by 
translating the concentration of empirical cross-covariances into explicit 
subspace error bounds via spectral perturbation theory. 
Experiments on synthetic and rendered image datasets support our theoretical 
findings and illustrate the necessity of the assumed conditions.
\end{abstract}
\section{Introduction}
Aligned multi-view data are ubiquitous, arising from multimodal sensing, 
multi-camera systems, and synchronized instruments. 
A primary objective of representation learning is to recover latent 
structures shared across views from \emph{nonlinear} observations while 
discarding view-private variation, yielding representations that are 
interpretable, robust to nuisance perturbations, and transferable across 
modalities.
Correlation-based methods, most notably canonical correlation analysis (CCA) 
and its multi-view generalizations 
\citep{kettenring1971canonical}, address this objective by learning 
whitened representations that are maximally aligned in second-order statistics. 
Such methods and their whitening-based variants are also widely used in 
self-supervised learning to prevent feature collapse and stabilize 
training 
\citep{zbontar2021barlow,weng2022investigation,ermolov2021whitening,he2022exploring,kalapos2024whitening}.

Despite their empirical success, what nonlinear CCA \emph{identifies} under 
realistic nonlinear distortions remains insufficiently characterized. 
Given the impossibility of unsupervised source recovery from general nonlinear 
mixtures \citep{hyvarinen1999nonlinear,locatello2019challenging}, 
multi-view constraints have emerged as a prominent structural prior 
to restore identifiability.
Existing CCA-based guarantees either impose post-nonlinear assumptions 
\citep{lyu2020nonlinear} or 
establish equivalence up to arbitrary invertible transforms 
\citep{lyu2022understanding}. 
While recent work establishes affine identifiability under general nonlinear mixing
in the two-view regime \citep{han2026provable},
what additional views enable in a basis-invariant sense remains unclear.

In this work, we address this question through the lens of \emph{subspace identification}. 
We study an additive multi-view generative process where each view is produced by an 
unknown smooth invertible map applied to a view-specific source. Each source decomposes 
into a linearly mixed shared latent vector plus view-private noise, 
relaxing the coordinate-wise independence assumption standard in ICA and 
aligning with the content--style modeling paradigm in causal representation learning
\citep{von2021self,yao2024multiview}. 
Because the mixing matrices are not identifiable, we target their basis-invariant 
\emph{signal subspaces} instead. 
We show that multi-view CCA acts as an \emph{intersection filter} over 
latent factors: it isolates the jointly correlated structure shared across views while 
eliminating view-private variation. We further provide finite-sample guarantees for the 
empirical estimator.

We summarize our contributions as follows:
\begin{itemize}
    \item We propose an $N$-view ($N \ge 3$) additive latent model that relaxes the
    coordinate-wise independence assumption and casts nonlinear CCA as a basis-invariant 
    subspace identification problem.
    \item We prove that generalized nonlinear CCA isolates the jointly correlated signal subspaces, formalizing the objective as an intersection filter over latent factors.
    \item We establish finite-sample consistency for empirical multi-view CCA, translating concentration of second-order statistics into explicit subspace recovery rates via spectral perturbation bounds.
    \item We validate the theory on synthetic and rendered image datasets, corroborating correlated-subspace recovery and the necessity of the stated conditions.
\end{itemize}

\section{Related Work}
\paragraph{Identifiable Disentangled Representations.}
Identifiable representation learning seeks to recover latent
structure up to unavoidable symmetries, an objective rooted
in blind source separation and independent component analysis (ICA) 
\citep{barlow1989finding,comon1994independent,cardoso1998blind}. 
Without additional assumptions, exact recovery from single-view 
i.i.d.\ observations under general nonlinear mixing is 
fundamentally impossible 
\citep{hyvarinen1999nonlinear,locatello2019challenging}. 
To restore identifiability, prior work has introduced symmetry-breaking inductive biases,
such as temporal dynamics, nonstationarity, auxiliary variables, interventions, or 
multi-view constraints 
\citep{hyvarinen2016unsupervised,hyvarinen2017nonlinear,hyvarinen2019nonlinear,khemakhem2020variational,locatello2020weakly,shu2020weakly}. 

\paragraph{Identifiability from Multiple Views.}
Multi-view observations provide cross-view constraints that 
mitigate the non-identifiability of single-view mixing. 
In linear settings, higher-order moment methods and multi-view 
ICA can isolate shared and view-specific sources up to standard 
permutation and scaling ambiguities \citep{podosinnikova2016beyond,pandeva2023multi}. 
Depending on the model class, contrastive objectives can identify 
shared factors in nonlinear mixtures up to specific transformations. 
Assuming independent sources and structured noise, these identifiability 
guarantees range from component-wise transformations to linear or 
block-level equivalence classes
\citep{zimmermann2021contrastive,gresele2020rosetta,matthes2023towards,galvez2023role}. 
Recent advances relax the independent-source assumption by exploiting causal 
latent structures, shifting the target to \emph{block-identifiability}. 
These methods are increasingly unified through partial-observability frameworks, 
where identifiability follows from underlying graphical structures and suitable 
regularity conditions
\citep{von2021self,scholkopf2016modeling,sturma2023unpaired,yao2024multiview,yao2025unifying,ahuja2022weakly,daunhawer2023multimodal,xu2024sparsity,locatello2020weakly}.

\paragraph{Identifiability of Nonlinear Multi-view CCA.}
CCA aligns views through second-order statistics after whitening.
Recent research on CCA-based identifiability has made progress 
on several related problems, including factor-level identifiability 
\citep{lyu2020nonlinear,lyu2022understanding,karakasis2023revisiting}, 
nonlinear dynamical system identification \citep{sidiropoulos2022canonical}, 
subspace identification \citep{sorensen2021generalized}, and subspace 
clustering \citep{karakasis2025subspace}. However, arbitrary nonlinear 
mixing makes nonlinear CCA substantially more challenging: existing 
guarantees are typically limited to recovery up to arbitrary invertible 
transformations, a substantially weaker conclusion than what is available 
in the linear case. Recent work establishes affine identifiability for 
nonlinear CCA in the two-view regime \citep{han2026provable} by leveraging 
Lancaster distribution priors \citep{lancaster1958structure,eagleson1964polynomial}.
It does not characterize what additional views identify from a basis-invariant 
perspective. To bridge this gap, we shift the identifiability target 
from exact mixing recovery to signal subspace identification. 
We prove that, under suitable latent priors, generalized CCA identifies 
the jointly correlated subspaces shared across views, thereby yielding 
a basis-invariant guarantee stronger than equivalence up to arbitrary 
invertible transformations.

\section{Problem Formulation}
\paragraph{Multi-view Generative Process.}
Let \(N\ge 2\) denote the number of views. We observe \(M\) i.i.d.\ 
multi-view samples
\[
\{(\mathbf{x}_1^{(m)},\ldots,\mathbf{x}_N^{(m)})\}_{m=1}^M .
\]
We suppress the sample index and write
\((\mathbf{x}_1,\ldots,\mathbf{x}_N)\) for a generic multi-view 
observation. For each view $i\in[N]:=\{1,\ldots,N\}$, the observation 
$\mathbf{x}_i\in\mathcal{X}_i\subseteq\mathbb{R}^{d_{\mathcal{X}_i}}$ is 
generated from an \emph{unobserved} view-specific latent source 
$\mathbf{s}_i\in\mathcal{S}_i\subseteq\mathbb{R}^{d_{\mathcal{S}_i}}$ 
via an \emph{unknown, generally nonlinear}
generative mapping
$\mathbf{g}_i:\mathcal{S}_i\to\mathcal{X}_i$, i.e.,
\(
\mathbf{x}_i = \mathbf{g}_i(\mathbf{s}_i).
\)
Throughout our identifiability analysis, we assume that each
\(\mathbf{g}_i\) is smooth and invertible on the support of
\(\mathbf{s}_i\).
Here, $\mathcal{S}_i$ and $\mathcal{X}_i$ denote the source and observation spaces of view $i$, 
respectively, and $d_{(\cdot)}$ is the dimension of the underlying ambient space.

Formally, we define the multi-view generative process via an additive-noise 
latent construction at the source level
as follows:
for each $i \in [N]$, we have
\begin{equation}
\label{equ:additive}
\mathbf{x}_i = \mathbf{g}_i(\mathbf{s}_i), \quad \mathbf{s}_i = \mathbf{A}_i \mathbf{c} + \bm{\epsilon}_i, \quad
\mathbf{c}\sim p_{\mathbf{c}},\quad
\bm{\epsilon}_i\sim p_{\bm{\epsilon}_i},
\end{equation}
where $\mathbf{c}\in\mathbb{R}^{d_{\mathcal{C}}}$ is a latent vector shared across views,
$\bm{\epsilon}_i\in\mathbb{R}^{d_{\mathcal{S}_i}}$ captures view-private variation, and
$\mathbf{A}_i\in\mathbb{R}^{d_{\mathcal{S}_i}\times d_{\mathcal{C}}}$ is a
view-specific mixing matrix.

Our generative process generalizes the single-view model of \citet{lyu2022provable} 
to $N$ views, 
with view-dependent mixing matrices $\{\mathbf{A}_i\}_{i=1}^N$ and 
nonlinear maps $\{\mathbf{g}_i\}_{i=1}^N$. 
Unlike standard ICA assumptions that enforce coordinate-wise independence 
\citep{hyvarinen2016unsupervised,hyvarinen2017nonlinear,hyvarinen2019nonlinear,khemakhem2020variational}, 
the shared signal term $\mathbf{A}_i\mathbf{c}$ induces structured 
dependencies within each source $\mathbf{s}_i$, 
e.g., as in causal latent models, consistent with the checkerboard-type 
dependencies of \citet{matthes2023towards}. 
Interpreting $\mathbf{c}$ and $\bm{\epsilon}_i$ as shared \emph{content} and 
view-private \emph{style}, 
our model differs from concatenation-based multi-view formulations in 
causal representation learning 
\citep{daunhawer2023multimodal,yao2024multiview}. 
Moreover, the model connects to partial observability:
if \(\ker(\mathbf A_i)\neq\{0\}\), then components of
\(\mathbf c\) lying in \(\ker(\mathbf A_i)\) do not affect
\(\mathbf s_i\) and are therefore unobservable from view \(i\).

\paragraph{Latent Structure.}
The additive construction induces nontrivial cross-view dependence among
the source variables. Since generalized CCA aggregates pairwise
correlations, our analysis first characterizes the joint distribution of
each source pair \((\mathbf s_i,\mathbf s_j)\) over 
$\mathcal{S}_i \times \mathcal{S}_j$ with 
$d_{\mathcal{S}_i}, d_{\mathcal{S}_j} \ge 2$. We then combines the
resulting pairwise structures across views.

\begin{assumption}[Latent Distributional Prior]
\label{ass:latents_iid_family}
Let \(\mathbf c\) and \(\{\bm\epsilon_i\}_{i=1}^N\) be as in
\autoref{equ:additive}.
\begin{assumplist}
    \item \label{ass:latent_factorization} \textbf{Latent Factorization.}
    The shared latent vector and private latent vectors are mutually
    independent. Within each vector, coordinates are i.i.d.:
    \begin{align*}
       &p_{\mathbf c}(\mathbf c)=\prod_{j=1}^{d_{\mathcal C}}\phi(c_j),\qquad
p_{\bm\epsilon_i}(\bm\epsilon_i)=\prod_{j=1}^{d_{\mathcal S_i}}\phi(\epsilon_{ij}),\\
        &\text{and}\qquad 
        p(\mathbf{c},\bm{\epsilon}_1,\ldots,\bm{\epsilon}_N)
        =
        p_{\mathbf{c}}(\mathbf{c})\prod_{i=1}^N p_{\bm{\epsilon}_i}(\bm{\epsilon}_i),
    \end{align*}
    where $\phi$ denotes standardized Gaussian density.
    \item \label{ass:isotropy} \textbf{Standardization.}
    The latents are centered and isotropic:
    \[
    \mathbb{E}[\mathbf{c}]=\mathbf{0}, \; \mathrm{Cov}(\mathbf{c})=\mathbf{I}_{d_{\mathcal{C}}}, \quad \mathbb{E}[\bm{\epsilon}_i]=\mathbf{0}, \; \mathrm{Cov}(\bm{\epsilon}_i)=\mathbf{I}_{d_{\mathcal{S}_i}}.
    \]

    \item \label{ass:admissible_prior} \textbf{Admissible Marginals.}
    In the main text, we use standardized Gaussian marginals, which are 
    sufficient for our main theory. Extensions to other standardized 
    Lancaster-type marginals are discussed in the supplementary material.
\end{assumplist}
\end{assumption}

\paragraph{Learning Objective.} 
The exact linear mixing matrices \(\{\mathbf A_i\}_{i=1}^N\) are 
fundamentally not identifiable as bases. Instead, we target the 
basis-invariant column-space information encoded by these matrices, 
formalized later as correlated signal subspaces in 
\autoref{def:subspaces}. Specifically, we learn view-specific encoders
\[
\mathbf f_i:\mathcal X_i\to\mathcal Z\subseteq\mathbb R^{d_{\mathcal Z}}, \qquad \forall i \in [N]
\]
and analyze the induced source-domain maps
\(\mathbf h_i:=\mathbf f_i\circ\mathbf g_i\). Our identifiability
analysis asks whether CCA forces these maps to recover the identifiable
correlated signal subspaces, up to the unavoidable orthogonal ambiguity
introduced by whitening. For simplicity, we use a common representation
dimension \(d_{\mathcal Z}\) across views.

\section{Preliminaries}
\paragraph{Notations.}
Throughout, scalars are denoted by plain letters. 
Random vectors and vector-valued functions are denoted by bold lowercase letters. 
Matrices are denoted by bold uppercase letters. 
Sets and spaces are denoted by calligraphic letters.
\subsection{Nonlinear Multi-view CCA}
CCA learns encoders whose whitened representations maximize cross-view correlation
measured by singular values of normalized cross-covariances.
\paragraph{Generalized CCA for Multi-view Learning.} 
Under the multi-view generative process in \autoref{equ:additive}, 
generalized CCA \citep{kettenring1971canonical} aggregates
pairwise objectives across all view pairs as
\begin{equation}
\label{eq:mvcca_obj}
J := \sum_{1\le i<j\le N}
\left\|
\bm\Sigma_{ii}^{-1/2}\bm\Sigma_{ij}\bm\Sigma_{jj}^{-1/2}
\right\|_*,
\end{equation}
where $\|\cdot\|_*$ is the nuclear norm,
$\bm{\Sigma}_{ii}: = \mathrm{Cov}(\mathbf{f}_i(\mathbf{x}_i))\succ0$ and
$\bm{\Sigma}_{jj}: = \mathrm{Cov}(\mathbf{f}_j(\mathbf{x}_j))\succ0$,
while
$\bm{\Sigma}_{ij}: = \mathrm{Cov}(\mathbf{f}_i(\mathbf{x}_i), \mathbf{f}_j(\mathbf{x}_j))$ denotes the cross-covariance matrix between
$\mathbf{f}_i$ and $\mathbf{f}_j$.

\paragraph{CCA in the source domain.} 
Let the induced representation mapping for view $i \in [N]$ 
be $\mathbf{h}_i := \mathbf{f}_i \circ \mathbf{g}_i : \mathcal{S}_i \to \mathcal{Z}$.
By reparameterization invariance (\autoref{lemma:rep-inv};\citep{han2026provable}), 
for all $1 \leq i < j \leq N$,
$\mathrm{Cov}(\mathbf f_i(\mathbf x_i),\mathbf f_j(\mathbf x_j))
=\mathrm{Cov}(\mathbf h_i(\mathbf s_i),\mathbf h_j(\mathbf s_j))$ and 
$\mathrm{Cov}(\mathbf f_i(\mathbf x_i))
=\mathrm{Cov}(\mathbf h_i(\mathbf s_i))$.
Consequently, optimizing \autoref{eq:mvcca_obj}
over $\{\mathbf f_i\}$ is mathematically equivalent to optimizing it over 
the induced representation mapping $\mathbf h_i$.
Under the regularity and representability conditions of 
\autoref{lemma:rep-inv}, this reparameterization allows us to analyze
population maximizers in the source domain without changing the CCA 
objective.

\subsection{Whitening and Canonicalization}
\label{sec:preliminaries}
Whitening, or its regularized variants such as soft-whitening 
\citep{zbontar2021barlow}, imposes a
non-degenerate normalization \citep{weng2022investigation} that 
makes the CCA objective well posed and prevents representational 
collapse \citep{ermolov2021whitening}.
Both properties are critical to our 
identifiability analysis. 

\paragraph{Whitened Representations.}
For each view $i \in [N]$, let 
$\bm\mu'_i:=\mathbb E[\mathbf h_i(\mathbf s_i)]$ and 
$\bm\Sigma'_{ii}:=\mathrm{Cov}(\mathbf h_i(\mathbf s_i))\succ0$.
A representation-whitening matrix 
$\mathbf B_i\in\mathbb R^{d_{\mathcal Z}\times d_{\mathcal Z}}$ satisfies
\(
\mathbf B_i\bm\Sigma'_{ii}\mathbf B_i^\top=\mathbf I
\)
and defines
\[
\tilde{\mathbf h}_i(\mathbf s_i)
:=
\mathbf B_i\big(\mathbf h_i(\mathbf s_i)-\bm\mu'_i\big),
\qquad
\mathrm{Cov}(\tilde{\mathbf h}_i(\mathbf s_i))=\mathbf I .
\]
A canonical choice is the symmetric inverse square root 
$\mathbf B_i={\bm\Sigma'_{ii}}^{-1/2}$.
Under this whitening, the normalized cross-covariance in 
\autoref{eq:mvcca_obj} becomes the cross-covariance of the whitened
source-domain representations:
\[
\tilde{\bm\Sigma}_{ij}
:=
\mathrm{Cov}\!\big(
\tilde{\mathbf h}_i(\mathbf s_i),
\tilde{\mathbf h}_j(\mathbf s_j)
\big)
=
\mathbf B_i\,\bm\Sigma'_{ij}\,\mathbf B_j^\top ,
\]
where $\bm\Sigma'_{ij}:=\mathrm{Cov}(\mathbf h_i(\mathbf s_i),\mathbf h_j(\mathbf s_j))$.
Thus the pairwise CCA term equals $\|\tilde{\bm\Sigma}_{ij}\|_*$.
While orthogonal left-multiplication of $\mathbf{B}_i$ defines the 
equivalence class for our identifiability theorem, the CCA objective 
remains strictly invariant to this ambiguity since it depends solely on 
the orthogonally invariant singular values of $\tilde{\bm\Sigma}_{ij}$.

\paragraph{Canonicalization via Whitening.}
Whitening removes scale and non-orthogonal within-view linear ambiguities,
reducing the feasible set to orthonormal directions in $L^2$; the remaining
ambiguity is orthogonal.
Fix a view $i\in[N]$ and consider the whitened representation mapping
$\tilde{\mathbf h}_i(\mathbf s_i)$.
Let $\{\xi_{i,n}\}_{n\ge1}$ be a fixed orthonormal basis of 
$L^2(P_{\mathbf s_i})$ restricted to centered square-integrable functions.
Since $\tilde{\mathbf h}_i$ is centered, the constant basis component is zero
and is omitted from the expansion.
Each component of $\tilde{\mathbf h}_i$ 
admits the following orthonormal basis expansion: for the component 
$j\in[d_{\mathcal Z}]$,
\[
\tilde{h}_{i,j}(\mathbf s_i)
=
\sum_{n\ge1}c_{i,j,n}\,\xi_{i,n}(\mathbf s_i),
\ \  c_{i,j,n}:=\langle \tilde h_{i,j},\xi_{i,n}\rangle_{L^2(\mu_i)},
\]
where $\mathbf{c}_{i,j}=[c_{i,j,1},\cdots,c_{i,j,\infty}]$ 
denotes the corresponding coefficient vector under the basis 
$\{\xi_{i,n}\}_{n \ge 1}$.
The whitening constraint is equivalent to the orthonormality 
of these coefficient vectors
across components within the same view $i$:
\begin{equation}
    \label{eq:whitening_induced_canonicalization}
\mathbf{c}_{i,p}^\top\mathbf{c}_{i,q}=\delta_{pq},\qquad
p,q\in[d_{\mathcal Z}],
\end{equation}
where $\delta_{pq}$ is the Kronecker-delta function, i.e.,
$\delta_{pq}=1$ iff $p=q$ and $0$ otherwise.
Whitening canonically enforces
(i) decorrelation between representation components and
(ii) unit normalization of each component in function space.
Equivalently, the coefficient matrix $\mathbf C_i$ with rows 
$(\mathbf{c}_{i,1}^\top, \dots, \mathbf{c}_{i,d_{\mathcal Z}}^\top)$ satisfies 
$\mathbf C_i\mathbf C_i^\top=\mathbf I$.

\section{Theory}
\label{sec:theory}
\subsection{Orthogonal Polynomial Expansion of the Canonical Density}
\label{subsec:canonical_factorization}
We first analyze the pairwise source distribution induced by the additive
multi-view model in \autoref{equ:additive}. After whitening each source, 
the normalized cross-covariance
admits an SVD whose canonical coordinates diagonalize the Gaussian dependence
between two views. This yields a product factorization of the joint density
and, via the Mehler--Hermite expansion, a spectral decomposition of nonlinear
cross-view correlations.
\paragraph{Canonical Coordinate System.}
Under the mutual independence and i.i.d.\ conditions in 
\cref{ass:latent_factorization},
let 
\[
\mathbf{W}_i := \rm{Cov}(\mathbf s_i)^{-\frac{1}{2}}=
\big(\mathbf{A}_i\mathbf{A}_i^\top + \mathbf{I}\big)^{-\frac{1}{2}},\quad
i \in [N]
\] 
be a whitening matrix for 
the view-specific source $\mathbf{s}_i$.
For any pair of views $1 \leq i < j \leq N$, the dependency structure 
of the view-specific sources $(\mathbf{s}_i, \mathbf{s}_j)$ is captured 
by their normalized cross-covariance matrix:
\begin{equation}
    \label{eq:normalized_cross_covariance_matrix}
    \mathbf{R}_{ij} := \mathrm{Cov}(\mathbf{W}_i \mathbf{s}_i, \mathbf{W}_j \mathbf{s}_j) = 
    \mathbf{W}_i \mathbf{A}_i \mathbf{A}_j^\top \mathbf{W}_j^\top.
\end{equation}
We consider the following singular value decomposition,
\[
\mathbf{R}_{ij} = \mathbf{U}_{ij} \mathbf{T}_{ij} \mathbf{V}_{ij}^\top, \quad
\mathbf{U}_{ij} \in O(d_{\mathcal{S}_i}), \quad
\mathbf{V}_{ij} \in O(d_{\mathcal{S}_j})
\]
where 
$\mathbf{T}_{ij} \in \mathbb{R}^{d_{\mathcal{S}_i} \times d_{\mathcal{S}_j}}$ 
is a rectangular diagonal matrix.
Let $r_{ij} := \mathrm{rank}(\mathbf A_i \mathbf A_j^\top)$, 
the nonzero singular values of $\mathbf{T}_{ij}$ satisfy
\[
1>t_{ij,1}\ge\cdots\ge t_{ij,r_{ij}}>0,
\]
and all remaining singular values are zero.
\begin{lemma}[Canonical Factorization]
\label{lemma:canonical_factorization}
Under the isotropic Gaussian prior in \cref{ass:isotropy}, define the
canonical coordinates
\begin{equation}
    \label{equ:canonicalizers}
\mathbf u=\mathbf U_{ij}^\top\mathbf W_i\mathbf s_i, 
\qquad
\mathbf v=\mathbf V_{ij}^\top\mathbf W_j\mathbf s_j,
\end{equation}
where $\mathbf U_{ij}$ and $\mathbf V_{ij}$ are orthogonal canonicalizers
of the whitened cross-covariance between $\mathbf s_i$ and $\mathbf s_j$.
By the change-of-variables formula for this invertible linear
transformation,
\begin{align*}
p_{\mathbf{u},\mathbf{v}}(\mathbf{u},\mathbf{v}) 
&= p_{\mathbf{s}_i,\mathbf{s}_j}
\bigl(
\mathbf{W}_i^{-1} \mathbf{U}_{ij} \mathbf{u},
\mathbf{W}_j^{-1} \mathbf{V}_{ij} \mathbf{v}
\bigr) \\
&\quad \cdot 
|\det \mathbf{W}_i|^{-1}
|\det \mathbf{W}_j|^{-1}.
\end{align*}
Moreover, the joint density factorizes as
\[
p_{\mathbf u,\mathbf v}(\mathbf u,\mathbf v)
=
\underbrace{\prod_{k=1}^{r_{ij}} \phi_{t_{ij,k}}(u_k,v_k)}_{\mathcal K_{ij}(\mathbf u[1:r_{ij}],\mathbf v[1:r_{ij}])}
\prod_{m=r_{ij}+1}^{d_{\mathcal S_i}}\phi(u_m)
\prod_{n=r_{ij}+1}^{d_{\mathcal S_j}}\phi(v_n)\]
where $\phi$ is the standard normal density and $\phi_t$ is the
standardized bivariate normal density with correlation $t$.
\end{lemma}
\begin{remark}[Information-Theoretic Interpretation]
The factorization in \autoref{lemma:canonical_factorization} implies 
that the cross-view interaction decouples into $r_{ij}$ independent 
Gaussian coordinates:
\[
v_k = t_{ij, k}\,u_k + \sqrt{1-t_{ij, k}^2}\,\varepsilon_k,
\quad \varepsilon_k\sim\mathcal N(0,1), \quad k \in [r_{ij}],
\]
where the canonical correlation $t_{ij, k}$ directly governs the effective
\textit{signal-to-noise ratio} (SNR)
$\gamma_{ij, k} := t_{ij, k}^2 / (1-t_{ij, k}^2)$ for the pair $(u_k, v_k)$.
By definition, $r_{ij}=\mathrm{rank}(\mathbf A_i\mathbf A_j^\top)\le \min(d_{\mathcal S_i},d_{\mathcal S_j})$, 
we have $t_{ij,k}=0$ for all $k>r_{ij}$.
This indicates that only $r_{ij}$ correlated canonical 
coordinates share cross-view information, i.e., $\mathrm{SNR} > 0$, 
while the remaining modes are independent noise. 
\end{remark}
To cast this probabilistic correlation structure into a 
functional-analytic framework, we expand the product bivariate 
Gaussian density
$\mathcal{K}_{ij}$ in \autoref{lemma:canonical_factorization} 
using normalized multivariate Hermite polynomials as shown in 
\citet{mehler1866ueber} and \citet{han2026provable}.
\begin{lemma}[Normalized Multivariate Mehler--Hermite Expansion]
\label{lemma:coupling_density}
Let $r:=r_{ij}$ and let
\(
\psi_n(z):=\frac{1}{\sqrt{n!}}He_n(z),
n\in\mathbb N_0,
\)
be the normalized probabilists' Hermite polynomials. For
$\mathbf n=(n_1,\ldots,n_r)\in\mathbb N_0^r$, define
normalized multivariate Hermite polynomials as
\[
\Psi_{\mathbf n}(\mathbf u)
:=
\prod_{k=1}^{r}\psi_{n_k}(u_k),
\qquad
\phi_r(\mathbf u)
:=
\prod_{k=1}^{r}\phi(u_k).
\]
Then $\{\Psi_{\mathbf n}\}_{\mathbf n\in\mathbb N_0^r}$ is an
orthonormal basis of $L^2(\nu_r)$, where
$\nu_r(d\mathbf u)=\phi_r(\mathbf u)d\mathbf u$ is the
multi-dimensional Gaussian measure.

Let $(\mathbf U,\mathbf V)$ be a centered Gaussian pair in canonical
coordinates such that
\(
\operatorname{Cov}(\mathbf U)=\operatorname{Cov}(\mathbf V)=I_r\)
and \(
\operatorname{Cov}(\mathbf U,\mathbf V)
=
\operatorname{diag}(\rho_1,\ldots,\rho_r),
\)
where $\rho_k:=t_{ij,k}\in(-1,1)$. If $\mathcal K_{ij}$ denotes the joint
density of $(\mathbf U,\mathbf V)$, then
\[
\mathcal K_{ij}(\mathbf u,\mathbf v)
=
\phi_r(\mathbf u)\phi_r(\mathbf v)
\sum_{\mathbf n\in\mathbb N_0^r}
t_{\mathbf n}\,
\Psi_{\mathbf n}(\mathbf u)
\Psi_{\mathbf n}(\mathbf v),
\]
where
\(
t_{\mathbf n}:=\prod_{k=1}^{r}\rho_k^{\,n_k}.
\)
\end{lemma}
Together with reparameterization invariance (\autoref{lemma:rep-inv})
and whitening-induced canonicalization
(\autoref{eq:whitening_induced_canonicalization}),
\autoref{lemma:coupling_density} yields a functional-analytic framework
for studying nonlinear representation maps in canonical coordinates.
Leveraging the orthonormality of Hermite polynomials under the Gaussian
measure \citep{dunkl2014orthogonal}, this framework diagonalizes
cross-view correlations in the Hermite basis, thereby enabling a
principled separation of first-order linear modes from higher-order
nonlinear Hermite components.

\subsection{Subspace Identifiability}
\label{subsec:subspace_identifiability}
After whitening, CCA-like objectives remain 
invariant to within-view orthogonal transformations, 
such as rotations.
Since precise recovery of mixing matrices 
$\{\mathbf{A}_i\}_{i=1}^N$ is fundamentally impossible, 
we instead target the basis-invariant signal 
subspaces they span. 
\begin{definition}[Correlated Signal Subspaces]
\label{def:subspaces}
For each view $i \in [N]$, the \emph{signal subspace} is defined as 
$\mathcal{U}_i := \mathrm{col}(\mathbf{A}_i) \subseteq \mathbb{R}^{d_{\mathcal{S}_i}}$. 
For any pair of distinct views $1 \leq i < j \leq N$, let 
$r_{ij} := \mathrm{rank}(\mathbf{A}_i\mathbf{A}_j^\top)$. We define 
the \emph{pairwise correlated signal subspace} of view $i$ with respect to view $j$ as 
\[
\mathcal{U}_{i \mid j} := \mathrm{col}\!\big(\mathbf{U}_{ij}(:,1:r_{ij})\big), 
\quad 
\mathcal{U}_{j \mid i} := \mathrm{col}\!\big(\mathbf{V}_{ij}(:,1:r_{ij})\big).
\]
Note that $\mathcal{U}_{i \mid j} \subseteq \mathcal{U}_{i}$ and 
$\mathcal{U}_{j \mid i} \subseteq \mathcal{U}_{j}$, with equality holding 
if and only if $r_{ij} = \mathrm{rank}(\mathbf A_i)$ and $r_{ij} = \mathrm{rank}(\mathbf A_j)$, 
respectively. 
Finally, the \emph{multi-view jointly correlated subspace} for view $i$ is defined 
as the intersection of its pairwise correlated subspaces: 
\(
\mathcal{U}_i^{\mathrm{mv}} := \bigcap_{j \neq i} \mathcal{U}_{i \mid j}.
\)
\end{definition}
We then give the identifiability of the correlated signal subspace 
 in two-view setup.
\begin{theorem}[Infinite-Dimensional Two-view Subspace Identifiability]
\label{thm:two_view_identifiability_infinite}
Fix a pair of distinct views $i\ne j$ with 
$r_{ij}:=\mathrm{rank}(\mathbf A_i\mathbf A_j^\top)>0$. 
Consider the population two-view CCA objective for this pair in the 
infinite-dimensional whitened function class.
Suppose the observed data are 
generated by the multi-view generative process defined in 
\autoref{equ:additive}, where the generating functions 
$\mathbf{g}_i$ and $\mathbf{g}_j$ are smooth and invertible and 
\autoref{ass:latents_iid_family} holds. Let 
$(\tilde{\mathbf{f}}_i^\star, \tilde{\mathbf{f}}_j^\star)$ be any 
whitened population maximizer of the two-view CCA 
objective in \autoref{eq:mvcca_obj}, and define the composition mappings 
$\tilde{\mathbf{h}}_\ell^\star := \tilde{\mathbf{f}}_\ell^\star \circ \mathbf{g}_\ell$ 
for $\ell \in \{i, j\}$. 

Assuming an infinite feature dimension, i.e., $d_\mathcal{Z} \to \infty$, 
there exist orthogonal matrices $\mathbf{O}_i, \mathbf{O}_j \in O(r_{ij})$ 
such that, 
\begin{equation}
\label{eq:two_view_identifiability_rankaware}
\begin{aligned}
    \mathbf{P}_{r_{ij}} \tilde{\mathbf{h}}_i^\star(\mathbf{s}_i) &= \mathbf{O}_i \mathbf{U}_{ij}(:,1:r_{ij})^\top \mathbf{W}_i \mathbf{s}_i, \\
    \mathbf{P}_{r_{ij}} \tilde{\mathbf{h}}_j^\star(\mathbf{s}_j) &= \mathbf{O}_j \mathbf{V}_{ij}(:,1:r_{ij})^\top \mathbf{W}_j \mathbf{s}_j,
\end{aligned}
\end{equation}
where $\mathbf P_{r_{ij}}$ is a rank-$r_{ij}$ coordinate selector, 
defined up to the unavoidable within-view orthogonal ambiguity of CCA, 
that extracts the coordinates corresponding to the pairwise correlated 
linear Hermite modes. Consequently, the maximizers identify the 
$r_{ij}$-dimensional whitened correlated subspaces $\mathcal{U}_{i\mid j}$ 
and $\mathcal{U}_{j\mid i}$ up to orthogonal transformations.
\end{theorem}
\begin{assumption}[First-Order Canonical Dominance]
\label{ass:first_order_dominance}
For any pair of views $1 \leq i < j \leq N$ with 
$\mathrm{rank}(\mathbf{A}_i \mathbf{A}_j^\top) \ge 2$, 
the strictly positive canonical correlations satisfy $t_{ij,r_{ij}} > t_{ij,1}^2$.
\end{assumption}
\begin{corollary}[Finite-Dimensional Two-view Subspace Identifiability]
\label{corollary:two_view_identifiability_finite}
Suppose the conditions of \autoref{thm:two_view_identifiability_infinite} 
hold, along with \autoref{ass:first_order_dominance}. For any finite representation 
dimension $d_{\mathcal{Z}} \ge r_{ij}$, the identifiability guarantees 
in \autoref{eq:two_view_identifiability_rankaware} hold exactly, where 
the rank-$r_{ij}$ coordinate selector takes the explicit form 
$\mathbf{P}_{r_{ij}} := [\mathbf{I}_{r_{ij}} \ \mathbf{0}] \in \mathbb{R}^{r_{ij} \times d_{\mathcal{Z}}}$.
\end{corollary}
\begin{remark}[Spectral separation of higher-order Hermite modes]
The Mehler--Hermite expansion (\autoref{lemma:coupling_density})  
decomposes the cross-view coupling into linear and higher-order 
(nonlinear) polynomial modes. \autoref{ass:first_order_dominance} 
mathematically ensures that the weakest linear correlation 
($t_{ij, r_{ij}}$) strictly exceeds the strongest possible 
higher-order correlation, which is bounded by $t_{ij, 1}^2$. 
This spectral gap guarantees that in the finite-dimensional 
regime (\autoref{corollary:two_view_identifiability_finite}), the 
generalized CCA objective strictly prioritizes the 
first-order linear Hermite modes, securely isolating them in the leading $r_{ij}$ 
coordinates. 
In a canonical representative of the whitened solution, the remaining 
coordinates, i.e., those with indices larger than $r_{ij}$, span 
higher-order Hermite modes orthogonal to the recovered first-order 
linear subspace \citep{han2026provable}.
\end{remark}

\begin{theorem}[Infinite-Dimensional Multi-View Subspace Identifiability]
\label{thm:multi_view_identifiability_infinite}
Suppose the conditions of \autoref{thm:two_view_identifiability_infinite} 
hold for $N \ge 3$ view ensembles and let $\{\tilde{\mathbf{f}}_i^\star\}_{i=1}^N$
be any whitened population maximizer of the generalized CCA objective 
in \autoref{eq:mvcca_obj}.
For each view $i$, let $r_i := \dim(\mathcal{U}_i^{\mathrm{mv}})$ and let $\mathbf{U}_i^{\mathrm{mv}} \in \mathbb{R}^{d_{\mathcal{S}_i} \times r_i}$ be an orthonormal basis spanning the multi-view jointly correlated subspace $\mathcal{U}_i^{\mathrm{mv}}$. 

Assuming an infinite feature dimension, i.e., $d_{\mathcal Z}\to\infty$,
there exist matrices $\mathbf L_i\in\mathbb R^{r_i\times d_{\mathcal Z}}$
with orthonormal rows and orthogonal matrices $\mathbf O_i\in O(r_i)$ such that
\begin{equation}
\label{eq:multi_view_identifiability_infinite}
    \mathbf L_i \tilde{\mathbf h}_i^\star(\mathbf s_i)
    =
    \mathbf O_i(\mathbf U_i^{\mathrm{mv}})^\top
    \mathbf W_i\mathbf s_i,
    \quad \forall i\in[N].
\end{equation}
Here, $\mathbf L_i$ extracts the jointly correlated component from the 
whitened representation up to an allowable within-view orthogonal 
post-transformation. If a canonical representative of the CCA solution 
is chosen, $\mathbf L_i$ can be taken as a coordinate selector.
\end{theorem}
\begin{remark}[Subspace dimensionality and partial observability]
Because the multi-view jointly correlated subspace 
$\mathcal{U}_i^{\mathrm{mv}}$ acts as a strict intersection filter, 
it isolates only those latent signals that are mutually visible 
across the entire $N$-view ensemble. Consequently, its dimension 
$r_i$ is strictly bounded by, and typically lower than, the pairwise 
ranks $r_{ij}$. Furthermore, for latent factors shared only among 
a strict subset of views, the generalized CCA objective lacks 
sufficient global constraints to enforce alignment, causing the 
corresponding unshared representation dimensions to inherently drift 
\citep{yao2024multiview}. The finite-dimensional counterpart of 
\autoref{eq:multi_view_identifiability_infinite} follows analogously to \autoref{corollary:two_view_identifiability_finite}.
\end{remark}
\subsection{Finite-Sample Statistical Consistency}
We now quantify the statistical error incurred when the population
second-order moments of fixed population-level CCA representations are
replaced by empirical estimates. Throughout this subsection, the encoders
$\{\mathbf f_i\}_{i=1}^N$ are treated as fixed and satisfy the population
identifiability conditions established above. Thus, the rank-$r_{ij}$
singular subspace of the population normalized cross-covariance coincides
with the pairwise correlated source subspace $\mathcal U_{i\mid j}$, up to
the orthogonal ambiguity characterized in \autoref{corollary:two_view_identifiability_finite}.
Since our identifiability results are driven by the singular/eigenspaces of
normalized cross-covariances, the analysis reduces to (i) concentration of
empirical covariances in operator norm and (ii) stability of 
singular-/eigenspaces under perturbations.
Our analysis below controls statistical estimation error only; optimization
and uniform generalization error for data-trained neural encoders are outside
the scope of our analysis.

\paragraph{Empirical Normalized Cross-Covariances.}
Given $n$ i.i.d.\ samples 
$\{(\mathbf{x}_1^{(t)},\dots,\mathbf{x}_N^{(t)})\}_{t=1}^n$ 
and fixed encoders 
$\{\mathbf{f}_i\}_{i=1}^N$, let $\mathbf{z}_i^{(t)} := \mathbf{f}_i(\mathbf{x}_i^{(t)})$. 
With $\bar{\mathbf{z}}_i$ and $\bar{\mathbf{z}}_j$ denoting 
the empirical means of the representations for views $i$ and $j$, 
define the sample cross-covariance matrices as
 $\widehat{\bm\Sigma}_{ij} := \frac{1}{n}\sum_{t=1}^n(\mathbf{z}_i^{(t)}-\bar{\mathbf{z}}_i)(\mathbf{z}_j^{(t)}-\bar{\mathbf{z}}_j)^\top$.
The empirical normalized cross-covariance is
\begin{equation}
    \label{eq:emp_Rij}
\widehat{\mathbf{R}}_{ij} := \widehat{\bm\Sigma}_{ii}^{-1/2} \widehat{\bm\Sigma}_{ij} \widehat{\bm\Sigma}_{jj}^{-1/2}, \quad 1 \le i < j \le N,
\end{equation}
with population counterpart $\mathbf{R}_{ij} := \bm\Sigma_{ii}^{-1/2}\bm\Sigma_{ij}\bm\Sigma_{jj}^{-1/2}$. Let $\widehat{\mathcal{U}}_{i\mid j}$ denote the rank-$r_{ij}$ left singular subspace of $\widehat{\mathbf{R}}_{ij}$. By \autoref{corollary:two_view_identifiability_finite}, this subspace estimates the pairwise correlated subspace $\mathcal{U}_{i\mid j}$ up to orthogonal ambiguity.

Additionally, we impose a mild tail condition on the \emph{population-whitened} representations
$\tilde{\mathbf z}_i := \bm\Sigma_{ii}^{-1/2}(\mathbf z_i-\mathbb E[\mathbf z_i])$.

\begin{assumption}[Sub-Gaussian whitened representations]
\label{ass:subgaussian_rep}
There exists $\kappa>0$ such that for all $i\in[N]$,
$\sup_{\|\mathbf u\|_2=1}\|\mathbf u^\top \tilde{\mathbf z}_i\|_{\psi_2}\le \kappa$.
\end{assumption}

\autoref{ass:subgaussian_rep} yields operator-norm concentration for
$\widehat{\bm\Sigma}_{ij}$ and controls the error of the empirical whitening factors
$\widehat{\bm\Sigma}_{ii}^{-1/2}$, which together imply concentration of the normalized
cross-covariances $\widehat{\mathbf R}_{ij}$ in \autoref{eq:emp_Rij}. To translate this
moment error into a \emph{subspace} error, we require a spectral separation of the target
singular subspaces: for each pair $(i,j)$, define
$\Delta_{ij}:=\sigma_{r_{ij}}(\mathbf R_{ij})-\sigma_{r_{ij}+1}(\mathbf R_{ij})>0$.
Under \autoref{ass:first_order_dominance}, whenever $r_{ij}\ge 2$ one has the explicit lower
bound $\Delta_{ij}\ge t_{ij,r_{ij}}-t_{ij,1}^2$, reflecting the strict dominance of linear
modes over higher-order Mehler--Hermite components.
We then provide finite-sample guarantees for subspace recovery.

\begin{theorem}[Finite-sample subspace recovery]
\label{thm:finite_sample_pairwise}
Suppose \autoref{ass:subgaussian_rep} holds, $d_{\mathcal Z}<\infty$, 
and $n \ge C_0\kappa^4(d_{\mathcal Z}+\log(N/\delta))$ so that the 
empirical within-view covariance matrices are nonsingular with high 
probability.
Then for any $\delta\in(0,1)$, with probability at least $1-\delta$,
simultaneously for all $1\le i<j\le N$,
\begin{equation}
\label{eq:Rij_conc}
\big\|\widehat{\mathbf R}_{ij}-\mathbf R_{ij}\big\|_2
\;\le\;
C\,\kappa^2\sqrt{\frac{d_{\mathcal Z}+\log(N/\delta)}{n}},
\end{equation}
for a universal constant $C>0$. Moreover, on the event
$\|\widehat{\mathbf R}_{ij}-\mathbf R_{ij}\|_2\le \Delta_{ij}/2$,
\begin{equation}
\label{eq:pairwise_sintheta}
\big\|\sin\Theta(\widehat{\mathcal U}_{i\mid j},\mathcal U_{i\mid j})\big\|_2
\;\le\;
\frac{2\|\widehat{\mathbf R}_{ij}-\mathbf R_{ij}\|_2}{\Delta_{ij}},
\end{equation}
and the same bound holds for the right singular subspace
$\widehat{\mathcal U}_{j\mid i}$.
\end{theorem}

\begin{remark}[Scaling with $n$]
Combining \autoref{eq:Rij_conc}--\autoref{eq:pairwise_sintheta} 
gives the parametric
$O_{\mathbb P}(n^{-1/2})$ rate. In particular,
$\|\sin\Theta\|_2\le\varepsilon$ is achieved once
$n \gtrsim \kappa^4(d_{\mathcal Z}+\log(N/\delta))/(\Delta_{ij}^2\varepsilon^2)$.
\end{remark}

To recover the multi-view jointly correlated subspace
$\mathcal U_i^{\mathrm{mv}}$ from noisy pairwise
estimates, we use an averaged-projector intersection filter. Let
$\mathbf P_{i\mid j}$ and $\widehat{\mathbf P}_{i\mid j}$ denote the orthogonal projectors
onto $\mathcal U_{i\mid j}$ and $\widehat{\mathcal U}_{i\mid j}$, and define
\[
\mathbf S_i:=\frac1{N-1}\sum_{j\ne i}\mathbf P_{i\mid j},
\qquad
\widehat{\mathbf S}_i:=\frac1{N-1}\sum_{j\ne i}\widehat{\mathbf P}_{i\mid j}.
\]
By the spectral characterization of averaged projectors, 
$\mathcal U_i^{\mathrm{mv}}$ is the eigenspace of $\mathbf S_i$ 
associated with eigenvalue $1$. We assume an intersection eigengap
\[
\Gamma_i:=1-\lambda_{r_i+1}(\mathbf S_i)>0,
\]
and estimate $\mathcal U_i^{\mathrm{mv}}$ by the top-$r_i$ eigenspace of
$\widehat{\mathbf S}_i$.
\begin{corollary}[Consistency of the intersection filter]
\label{cor:finite_sample_multiview}
Under the conditions of \autoref{thm:finite_sample_pairwise}, with probability at least
$1-\delta$, for every $i\in[N]$,
\begin{align*}
\|\sin\Theta(\widehat{\mathcal{U}}_i^{\mathrm{mv}}, \mathcal{U}_i^{\mathrm{mv}})\|_2 
&\le \frac{\|\widehat{\mathbf{S}}_i - \mathbf{S}_i\|_2}{\Gamma_i} \\
&\le \frac{2}{\Gamma_i} \max_{j\ne i} \|\sin\Theta(\widehat{\mathcal{U}}_{i\mid j}, \mathcal{U}_{i\mid j})\|_2.
\end{align*}
Consequently, $\widehat{\mathcal U}_i^{\mathrm{mv}}$ is consistent at rate
$O_{\mathbb P}(n^{-1/2})$, governed by $\{\Delta_{ij}\}$ and $\Gamma_i$.
\end{corollary}

All proofs are deferred to the supplementary material.
\section{Experiments}
\label{sec:experiments}
In this section, we empirically test the main predictions of our
identifiability analysis using controlled generative experiments.
The experiments are designed to assess whether correlation-based
multi-view objectives recover the ground-truth canonical source
subspaces under known nonlinear view transformations. We first use
fully synthetic data, where the latent factors, canonical spectra,
and representation dimensions can be controlled exactly. We then use
3DIdent as a rendered stress test with annotated generative factors.
Our setup extends established disentanglement learning benchmarks 
\citep{zimmermann2021contrastive,matthes2023towards,han2026provable} 
to the multi-view ($N=3$) regime considered in this work.

\subsection{Experimental Setup}
\paragraph{Datasets.} 
We evaluate on two settings that provide ground-truth access to
generative factors:
\begin{itemize}
    \item \textbf{Synthetic Data}: We extend the single-generator setup of 
        \citet{zimmermann2021contrastive} to an $N=3$ multi-view regime. 
        Observations are produced by applying independent nonlinear 
        generating
        functions to view-specific sources sampled from our additive 
        model (\autoref{equ:additive}), enabling exact control over 
        the latent distributions and mixing matrices.

    \item \textbf{3DIdent} \citep{daunhawer2023multimodal}: 
        A physically rendered benchmark with 11 annotated generative factors.
        Since the factor grid is discrete, 3DIdent does not exactly instantiate
        our continuous source-level model. We therefore use it as a controlled
        rendered stress test. Specifically, we first sample continuous latent
        vectors according to \autoref{equ:additive}, normalize the factor
        coordinates to a common scale, and then select the rendered instance
        whose factor tuple is nearest in Euclidean distance. This procedure
        preserves controlled access to ground-truth factor labels while testing
        the learned representations on visually nontrivial observations.
\end{itemize}

\paragraph{Construction of Linear Mixing Matrices.}
To strictly control the pairwise canonical correlations of
$\mathbf{R}_{ij}$ in \autoref{eq:normalized_cross_covariance_matrix}, 
we construct the mixing matrices $\{\mathbf{A}_i\}_{i=1}^3$ to match 
target singular values $\{t_{ij,k}\}_{k=1}^{r_{ij}} \subset [0,1)$ for the 
population normalized cross-covariance matrices,
\[
\mathbf{R}_{ij} = (\mathbf{A}_i \mathbf{A}_i^\top + \mathbf{I})^{-1/2} \mathbf{A}_i \mathbf{A}_j^\top (\mathbf{A}_j \mathbf{A}_j^\top + \mathbf{I})^{-1/2}.
\]
By assigning a shared right singular matrix $\mathbf{P} \in O(d)$ 
across all $\mathbf{A}_i$, the singular values of $\mathbf{R}_{ij}$ 
factorize elementwise as $t_{ij,k} = g_{i,k} g_{j,k}$, where 
$g_{i,k} := \sigma_{i,k} / \sqrt{1+\sigma_{i,k}^2}$ and $\sigma_{i,k}$ 
are the 
singular values of $\mathbf{A}_i$. For $N=3$ views, this algebraic
system uniquely yields $g_{i,k} = (t_{ij,k} t_{ki,k} / t_{jk,k})^{1/2}$, 
subsequently fixing $\sigma_{i,k} = g_{i,k} / \sqrt{1-g_{i,k}^2}$. 
We then construct 
$\mathbf{A}_i := \mathbf{U}_i \mathrm{diag}(\sigma_{i,1}, \dots, \sigma_{i,r_{ij}}) \mathbf{P}^\top$, sampling independent random orthogonal matrices $\mathbf{U}_i \in O(d)$. 
This guarantees the exact prescribed canonical 
spectra while randomizing the signal subspace orientations.

\paragraph{Methods.}
We compare neural GCCA, the empirical objective aligned with our
population analysis, against representative self-supervised objectives
using the same encoder architecture and representation dimension:
Barlow Twins, W-MSE, and InfoNCE.
\begin{itemize}
    \item \textbf{Barlow Twins} \citep{zbontar2021barlow}: Enforces invariance and minimizes redundancy through cross-correlation diagonalization.  
    \item \textbf{W-MSE} \citep{ermolov2021whitening}: Aligns whitened multi-view neural representations by minimizing mean-squared error.  
    \item \textbf{InfoNCE} \citep{oord2018representation}: The standard 
    contrastive learning objective, maximizing mutual information 
    between views, shown to be identifiable under mild conditions
    \citep{zimmermann2021contrastive,matthes2023towards}.
    \item \textbf{Generalized CCA (GCCA)} \citep{kettenring1971canonical}: 
    The correlation-based multi-view objective studied in this work,
    implemented as a neural version of SUMCOR generalized CCA and
    corresponding to \autoref{eq:mvcca_obj}.
\end{itemize}

To approximate the view-specific encoding functions 
$\{\mathbf f_i\}_{i=1}^N$, we parameterize them as neural 
networks, leveraging their universal approximation capability. 
We use neural-network-parameterized GCCA as a practical empirical
approximation to the population GCCA objective when data are generated
via online sampling from the underlying distributions. This empirical
training procedure is not covered by the finite-sample perturbation
theorem, which analyzes fixed finite-dimensional representations.
All baselines are implemented 
using their official repositories when available, otherwise we re-implement following 
the authors' specifications.
\paragraph{Metrics.}
We assess subspace recovery via principal angles between 
the span of the learned whitened representation and the ground-truth canonical subspace. 
Specifically, we report the mean principal angle ($PA_{\text{mean}}$) 
and the maximum principal angle ($PA_{\text{max}}$), 
following standard subspace identification practice 
\citep{ma2020subspace,gao2017sparse,cai2018rate}.  

In all experiments, view-specific encoders are trained using Adam 
with a fixed learning rate of $10^{-4}$. The batch size is set to 
1024 for synthetic data and 256 for 3DIdent. 
On the synthetic dataset with $d_{\mathcal S}=5$, 
100{,}001 training steps require approximately 2.5 hours, 
whereas 40{,}000 steps on 3DIdent take about 20 hours 
on a single RTX A5000 GPU. 
All tables report mean $\pm$ standard deviation over five random seeds.
\begin{table}
    \centering
    \caption{Comparison of the mean and maximum principal angles ($PA_{\mathrm{mean}}$, $PA_{\mathrm{max}}$) on synthetic data ($d_{\mathcal{S}_i}=d_\mathcal{Z}=5, \forall i \in [3]$) under Gaussian prior ($p_\phi$).}
    \label{tab:synthetic_identifiability}
    \resizebox{\linewidth}{!}{
    \begin{tabular}{l*{6}{c}}
        \toprule
        Methods   & \multicolumn{2}{c}{$\tilde{\mathbf{f}}_1$}  & \multicolumn{2}{c}{$\tilde{\mathbf{f}}_2$}  & \multicolumn{2}{c}{$\tilde{\mathbf{f}}_3$} \\
        \cmidrule(r){2-3}
        \cmidrule(r){4-5}
        \cmidrule(r){6-7}
            & $PA_{\mathrm{mean}}$ & $PA_{\mathrm{max}}$   & $PA_{\mathrm{mean}}$ & $PA_{\mathrm{max}}$   & $PA_{\mathrm{mean}}$ & $PA_{\mathrm{max}}$ \\
        \midrule
        BarlowTwins & 34.12 $\pm$0.21 & 86.41 $\pm$0.95 & 31.85 $\pm$0.18 & 81.22 $\pm$0.82 & 33.40 $\pm$0.24 & 84.95 $\pm$0.91 \\
        InfoNCE     & 4.95 $\pm$0.06 & 8.15 $\pm$0.52 & 5.12 $\pm$0.09 & 8.42 $\pm$0.38 & \textbf{4.88 $\pm$0.08} & \textbf{7.98 $\pm$0.47} \\
        W-MSE       & 5.08 $\pm$0.08 & 8.31 $\pm$0.41 & \textbf{4.91 $\pm$0.05} & \textbf{7.95 $\pm$0.58} & 5.03 $\pm$0.07 & 8.24 $\pm$0.44 \\
        GCCA        & \textbf{4.85 $\pm$0.05} & \textbf{7.92 $\pm$0.35} & 4.98 $\pm$0.08 & 8.11 $\pm$0.42 & 4.92 $\pm$0.06 & 8.05 $\pm$0.51 \\
        \bottomrule
    \end{tabular}}
\end{table}
\begin{table}
    \centering
    \caption{Comparison of the mean and maximum principal angles ($PA_{\mathrm{mean}}$, $PA_{\mathrm{max}}$) on 3DIdent ($d_{{\mathcal S}_i}=d_{\mathcal Z}=11, \forall i \in [3]$) under Gaussian prior ($p_\phi$).}
    \label{tab:3d_identifiability}
    \resizebox{\linewidth}{!}{
    \begin{tabular}{l*{6}{c}}
        \toprule
        Methods   & \multicolumn{2}{c}{$\tilde{\mathbf{f}}_1$}  & \multicolumn{2}{c}{$\tilde{\mathbf{f}}_2$}  & \multicolumn{2}{c}{$\tilde{\mathbf{f}}_3$} \\
        \cmidrule(r){2-3}
        \cmidrule(r){4-5}
        \cmidrule(r){6-7}
            & $PA_{\mathrm{mean}}$ & $PA_{\mathrm{max}}$   & $PA_{\mathrm{mean}}$ & $PA_{\mathrm{max}}$   & $PA_{\mathrm{mean}}$ & $PA_{\mathrm{max}}$ \\
        \midrule
        BarlowTwins & 34.21 $\pm$0.45 & 86.42 $\pm$1.85 & 36.15 $\pm$0.52 & 82.75 $\pm$2.15 & 35.82 $\pm$0.61 & 88.14 $\pm$1.92 \\
        InfoNCE     & 8.05 $\pm$0.18 & 11.21 $\pm$0.65 & 8.12 $\pm$0.22 & 10.85 $\pm$0.82 & 7.98 $\pm$0.15 & 11.42 $\pm$0.91 \\
        W-MSE       & 8.15 $\pm$0.25 & 10.95 $\pm$0.78 & 7.92 $\pm$0.19 & 11.15 $\pm$0.55 & 8.11 $\pm$0.21 & 10.88 $\pm$0.88 \\
        GCCA        & \textbf{7.82 $\pm$0.16} & \textbf{10.51 $\pm$0.62} & \textbf{7.85 $\pm$0.18} & \textbf{10.45 $\pm$0.71} & \textbf{7.80 $\pm$0.15} & \textbf{10.48 $\pm$0.68} \\
        \bottomrule
    \end{tabular}}
\end{table}
\subsection{Validation of Theoretical Findings}
\paragraph{Subspace Identifiability.}
We empirically validate subspace identifiability under our additive 
generative model in \autoref{equ:additive}. 
For simplicity, we present the results using Gaussian prior. 
Results for other admissible non-Gaussian distributions are 
deferred to the supplementary material. 
Setting the latent dimensions 
to $d_{\mathcal{C}} = d_{\mathcal{S}_i} = 5$, we independently sample 
the shared content $\mathbf{c}$ and view-private noise 
$\{\bm{\epsilon}_i\}_{i=1}^N$ from the selected distributions 
(details in the supplementary material). 
To invert this process, the view-specific 
encoders $\{\mathbf{f}_i\}_{i=1}^N$ are implemented as residual MLPs 
equipped with batch normalization to ensure stable optimization.

\autoref{tab:synthetic_identifiability} and \autoref{tab:3d_identifiability}
show that GCCA, InfoNCE, and W-MSE all recover low-dimensional
source subspaces with small principal angles, while Barlow Twins
exhibits substantially larger subspace errors. On 3DIdent, GCCA
achieves the lowest errors across all views and metrics. On synthetic
data, GCCA is competitive with InfoNCE and W-MSE, with all three
methods attaining comparable subspace recovery. These results are
consistent with the view that whitening- and correlation-based
objectives are well aligned with the subspace structure predicted by
our theory, while the exact empirical ranking also reflects
optimization and objective-specific effects outside the population
identifiability analysis.
\begin{figure}[t]
  \centering
    \includegraphics[width=\linewidth]{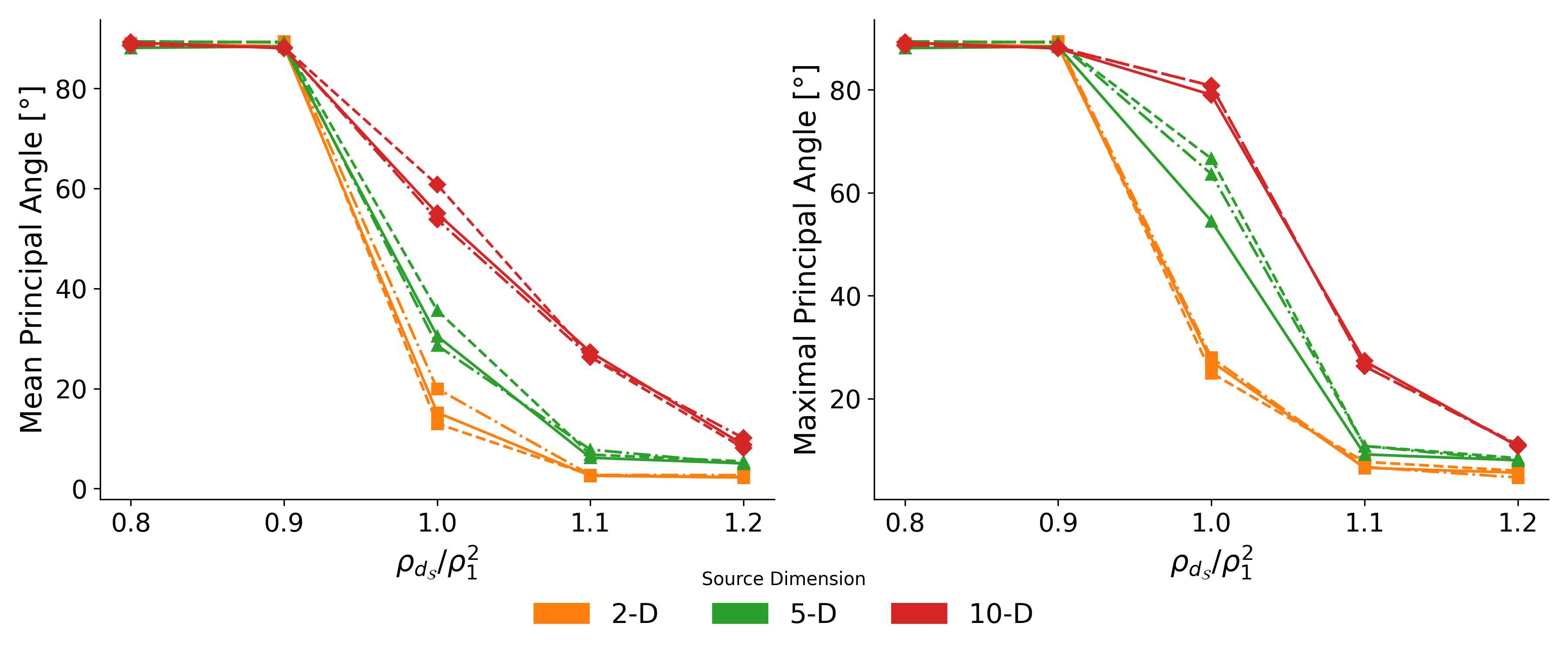}
\caption{Ablation over the first-order canonical ratio 
$\rho_{d_{\mathcal{S}}}/\rho_1^2$ ($d_{\mathcal{S}}=d_{\mathcal{Z}}$). 
Left and right panels display the mean 
($PA_{\mathrm{mean}}\downarrow$) and maximum 
($PA_{\mathrm{max}}\downarrow$) principal angles, respectively. 
Colors indicate source dimensions, while solid, dashed, and 
dash-dotted lines denote the view-specific encoders 
$\tilde{\mathbf{f}}_1$, $\tilde{\mathbf{f}}_2$, and $\tilde{\mathbf{f}}_3$.}
\label{fig:first_order_canonical_dominance}
\end{figure}
\paragraph{Ablation on First-order Canonical Dominance.} 
Let $\rho_{d_{\mathcal S}}/\rho_1^2$ denote the first-order canonical dominance ratio. 
We ablate \autoref{ass:first_order_dominance} by varying this ratio while 
keeping all other parameters fixed. 
\autoref{fig:first_order_canonical_dominance} reports the resulting mean 
and maximum principal angles on the synthetic dataset 
when the first-order canonical ratio are 
uniformly sampled from $[\rho_{d_{\mathcal S}},\rho_1]$.
When $\rho_{d_{\mathcal S}}/\rho_1^2<1$, at least one canonical direction 
violates the dominance condition, leading to persistently large 
$PA_{\max}$ and incomplete subspace recovery. 
Once the ratio exceeds $1$, both $PA_{\mathrm{mean}}$ and 
$PA_{\max}$ decrease sharply, indicating full 
identifiability of all canonical directions. 
The transition is more abrupt at lower source dimensions, 
reflecting more concentrated recoverability in 
low-dimensional regimes.
\begin{table}[t]
    \centering
    \caption{Mean principal angle ($PA_{\mathrm{mean}}\downarrow$) in the overcomplete regime ($d_{\mathcal{S}_i}=5$, $d_{\mathcal{Z}}=7, \forall i \in [3]$).}    
    \label{tab:dimension_mismatch_overcomplete}
    \resizebox{\linewidth}{!}{%
    \begin{tabular}{c*{5}{c}}  
        \toprule
        Encoder &
        Gaussian & Binomial & Gamma & Poisson & Hypergeometric \\
        \midrule
        $\tilde{\mathbf{f}}_1$ & 5.42 $\pm$0.08 & 5.58 $\pm$0.12 & 5.47 $\pm$0.05 & 5.61 $\pm$0.14 & 5.53 $\pm$0.09 \\
        $\tilde{\mathbf{f}}_2$ & 5.51 $\pm$0.11 & 5.39 $\pm$0.07 & 5.65 $\pm$0.15 & 5.48 $\pm$0.06 & 5.55 $\pm$0.10 \\
        $\tilde{\mathbf{f}}_3$ & 5.45 $\pm$0.09 & 5.62 $\pm$0.13 & 5.50 $\pm$0.08 & 5.41 $\pm$0.11 & 5.59 $\pm$0.07 \\
        \bottomrule
    \end{tabular}}
\end{table}
\begin{figure}[t]
  \centering
 
    \includegraphics[width=\linewidth]{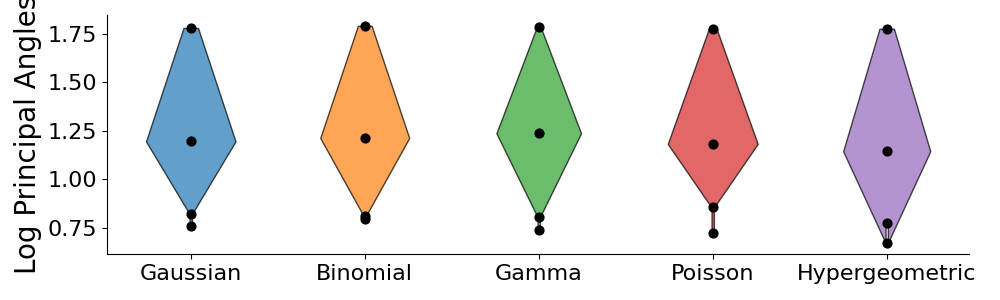}
\caption{Log principal angles of encoder $\mathbf{f}_1$ in 
the under-complete setup ($d_{\mathcal S_i}=5,\ d_{\mathcal Z}=4, \forall i \in [3]$). 
Black dots denote principal angles and the shaded region indicates 
the log-standard deviation. }
  \label{fig:log_principal_angle_dimension_mismatch}
\end{figure}
\paragraph{Ablation on Dimension Mismatch.} 
We study identifiability under dimension mismatch on synthetic dataset. 
In the under-complete regime ($d_{\mathcal S}>d_{\mathcal Z}$), 
\autoref{fig:log_principal_angle_dimension_mismatch} shows the log-principal-angle 
distribution across source distributions. 
Across all cases, several directions attain small angles, 
while at least one direction remains bounded away from zero, 
indicating partial but not full subspace recovery. 
The variability is limited, suggesting robustness to distributional choice, 
yet the missing dimension prevents complete identifiability.

In the over-complete regime ($d_{\mathcal S}<d_{\mathcal Z}$),
\autoref{tab:dimension_mismatch_overcomplete} reports consistently
small mean principal angles across source distributions, indicating
that the learned representations contain directions closely aligned
with the source subspace. However, the additional representation
coordinates are not uniquely determined by the first-order subspace
criterion. Therefore, the over-complete setting supports subspace
recovery but not global coordinate-level identifiability.
Future work will investigate this higher-order 
structure more systematically.

\section{Discussion and Conclusion}
We reframed nonlinear multi-view CCA as a basis-invariant subspace 
identification problem. By considering an additive multi-view generative 
process, we proved that for $N \ge 3$ views, the population 
generalized CCA acts as a provable intersection filter, 
recovering the jointly correlated signal subspaces up to an orthogonal transformation 
while filtering out view-private components.
The first-order canonical dominance condition theoretically guarantees 
the spectral separation between first-order shared Hermite components and 
higher-order components induced by nonlinear transformations
via a multivariate Mehler--Hermite expansion. 
This spectral gap further enables finite-sample subspace recovery when the relevant 
population operators are estimated from finite-dimensional feature representations.

While our theory assumes 
matched representation and source dimensions and full-rank source covariance
experiments demonstrate the behavior of CCA in capacity-mismatched regimes. 
Future work will systematically investigate the geometric isolation 
of higher-order Hermite components in these redundant dimensions, 
and extend this identifiability framework to rank-deficient source structures, 
e.g., partial observability, 
further bridging multivariate statistics with 
self-supervised learning.



\bibliography{main}

\onecolumn

\title{Supplementary Material}
\maketitle
\appendix

\printsupplementarytoc

\suppsection{Overview of the Supplementary Material}
\label{app:overview_supplementary}

This supplementary material is organized as follows. 

We first collect
auxiliary tools used throughout the analysis, including the whitened
encoder classes (\autoref{definition:whitend_encoder}), the population CCA
objective (\autoref{equ:CCA_objective}), pushforward identities
(\autoref{lemma:pushforward}), reparameterization invariance
(\autoref{lemma:rep-inv}), maximizer correspondence
(\autoref{lemma:rep-inv}, item~\ref{item:maximizer_correspondence}), and the canonical factorization
of the pairwise source density (\autoref{lemma:canonical_factorization}). 

We then provide the proof of
\autoref{lemma:coupling_density}, which establishes the normalized
multivariate Mehler--Hermite expansion underlying the spectral analysis.

The subsequent sections contain complete proofs of the main theoretical
claims. In particular, we prove the infinite-dimensional two-view subspace
identifiability result in
\autoref{thm:two_view_identifiability_infinite}, the finite-dimensional
two-view specialization in
\autoref{corollary:two_view_identifiability_finite}, the
infinite-dimensional multi-view identifiability result in
\autoref{thm:multi_view_identifiability_infinite}, the consistency of the
multi-view intersection filter in
\autoref{cor:finite_sample_multiview}, and the finite-sample subspace
recovery guarantee in
\autoref{thm:finite_sample_pairwise}. These proofs make explicit how
reparameterization invariance, whitening-induced canonicalization,
orthogonal-polynomial diagonalization, spectral separation, and standard
matrix perturbation arguments combine to yield the stated identifiability
and consistency guarantees.

Finally, we report additional experimental results complementing the main
text. These results further evaluate the behavior of the proposed
correlation-based multi-view objective across different latent
distributional settings and provide additional evidence for the role of the
stated spectral and dimensional assumptions.

\suppsection{Auxiliary Tools from Prior Works}

\suppsubsection{Reparameterization Invariance}
\begin{definition}[Whitened Encoder Classes with Post-orthogonal Closure]\label{definition:whitend_encoder}
Let $\mathcal{H}_{\mathcal{X}}\subset L^2(P_{\mathbf x};\mathbb R^{d_{\mathcal Z}})$ be Hilbert spaces of square-integrable vector-valued functions.
For base encoders $\forall \mathbf f\in\mathcal{H}_{\mathcal{X}}$ with covariance  matrix $\mathbf{\Sigma}_{\mathbf f}$ symmetric positive definite, choose any 
$\mathbf W_{\mathbf f}\in\mathrm{GL}(d_{\mathcal Z})$ such that $\mathbf W_{\mathbf f}\,\mathbf{\Sigma}_{\mathbf f}\,\mathbf W_{\mathbf f}^\top=\mathbf I_{d_{\mathcal Z}}$.
Define the whitened encoder class on domain $\mathcal{X}$ by
\begin{equation*}
\tilde{\mathcal F}_{\mathcal{X}}:= \Big\{\, \mathbf W_{\mathbf f}\big(\mathbf f-\mathbb E[\mathbf f(\mathbf{x})]\big)\ :\ \mathbf f\in\mathcal{H}_{\mathcal{X}},\ \mathbf{\Sigma}_{\mathbf f}\succ 0 \,\Big\},
\end{equation*}
and $\mathbf Q\,\tilde{\mathbf f}\in\tilde{\mathcal F}_{\mathcal{X}}, \forall \tilde{\mathbf f}\in\tilde{\mathcal F}_{\mathcal{X}}, \mathbf Q \in O(d_{\mathcal Z})$. Analogously, we define whitened encoder class $\tilde{\mathcal F}'_{\mathcal{X}'}$ for all base encoders $\mathbf{f}' \in \mathcal{H}_{\mathcal{X}'}\subset L^2(P_{\mathbf x'};\mathbb R^{d_{\mathcal Z}})$.
\end{definition}
\begin{assumption}[Whitened Encoder Classes with Post-orthogonal Closure]\label{assump:function-class}
Let $\mathcal{H}_{\mathcal{X}}\subset L^2(P_{\mathbf x};\mathbb R^{d_{\mathcal Z}})$ be Hilbert spaces of square-integrable vector-valued functions.
For base encoders $\forall \mathbf f\in\mathcal{H}_{\mathcal{X}}$ with covariance  matrix $\mathbf{\Sigma}_{\mathbf f}$ symmetric positive definite, choose any 
$\mathbf W_{\mathbf f}\in\mathrm{GL}(d_{\mathcal Z})$ such that $\mathbf W_{\mathbf f}\,\mathbf{\Sigma}_{\mathbf f}\,\mathbf W_{\mathbf f}^\top=\mathbf I_{d_{\mathcal Z}}$.
Define the whitened encoder class on domain $\mathcal{X}$ by
\begin{equation*}
\tilde{\mathcal F}_{\mathcal{X}}:= \Big\{\, \mathbf W_{\mathbf f}\big(\mathbf f-\mathbb E[\mathbf f(\mathbf{x})]\big)\ :\ \mathbf f\in\mathcal{H}_{\mathcal{X}},\ \mathbf{\Sigma}_{\mathbf f}\succ 0 \,\Big\},
\end{equation*}
and $\mathbf Q\,\tilde{\mathbf f}\in\tilde{\mathcal F}_{\mathcal{X}}, \forall \tilde{\mathbf f}\in\tilde{\mathcal F}_{\mathcal{X}}, \mathbf Q \in O(d_{\mathcal Z})$. Analogously, we define whitened encoder class $\tilde{\mathcal F}'_{\mathcal{X}'}$ for all base encoders $\mathbf{f}' \in \mathcal{H}_{\mathcal{X}'}\subset L^2(P_{\mathbf x'};\mathbb R^{d_{\mathcal Z}})$.
\end{assumption}
\begin{definition}[CCA Population Objective.]
Given the whitened encoder class in \autoref{assump:function-class}, nonlinear CCA solves the following population optimization problem:
\begin{equation}
\label{equ:CCA_objective}
\max_{\tilde{\mathbf f}\in \tilde{\mathcal F}_{\mathcal{X}},\;\tilde{\mathbf f'}\in \tilde{\mathcal F'}_{\mathcal{X}'}}\; J( \tilde{\mathbf f}, \tilde{\mathbf f'})=\sum_{i=1}^{d_\mathcal{Z}}
\sigma_i\big(\mathrm{Cov}(\tilde{\mathbf f}(\mathbf{x}),\,\tilde{\mathbf f'}(\mathbf{x'}))\big),
\end{equation}
where $\sigma_i(\cdot)$ denotes the $i$-th singular value and $\mathrm{Cov}(\cdot)$ the population covariance.
\end{definition}
\begin{lemma}[Pushforward identities and whitening preservation]
\label{lemma:pushforward}
Let ground-truth latent pair $(\mathbf{s}, \mathbf{s}')$ be square-integrable random vectors on $\mathcal S\times\mathcal S$ with joint law $P_{\mathbf s\mathbf s'}$.
Let $\mathbf g:\mathcal S\!\to\!\mathcal X$ and $\mathbf g':\mathcal S\!\to\!\mathcal X'$ be Borel-measurable and observation pair
$(\mathbf{x},\mathbf{x'})=(g(\mathbf{s}),g'(\mathbf{s}'))$ with joint law $P_{\mathbf x\mathbf x'}=(\mathbf g, \mathbf g')_\# P_{\mathbf s\mathbf s'}$.
Then the following properties hold.
\begin{enumerate}[label=\arabic*.]
\item \textbf{Square-integrability.} For any Borel-measurable $\tilde{\mathbf f}\in \tilde{\mathcal{F}}_\mathcal{X}$  with finite second moments on $\mathcal X$, the composition $\tilde{\mathbf f}\!\circ \mathbf g $ also have finite second moments on $\mathcal S$. Analogously, the same property also applies to all $\tilde{\mathbf f'}\in \tilde{\mathcal{F}}'_{\mathcal{X}'}$.

\item \textbf{Expectation and Covariance Preservation.} For all square-integrable $\tilde{\mathbf f}\in \tilde{\mathcal{F}}_\mathcal{X}$, $\tilde{\mathbf f'}\in \tilde{\mathcal{F}}'_{\mathcal{X}'}$,
\[
\mathbb E[\tilde{\mathbf f}(\mathbf{x})]=\mathbb E[\tilde{\mathbf f}\circ\mathbf g(\mathbf{s})],\quad
\mathbb E[\tilde{\mathbf f'}(\mathbf{x'})]=\mathbb E[\tilde{\mathbf f'}\circ \mathbf g'(\mathbf{s}')],
\]
\[
\mathrm{Cov}\big(\tilde{\mathbf f}(\mathbf{x})\big)=\mathrm{Cov}\big(\tilde{\mathbf f}\!\circ \mathbf g(\mathbf{s})\big),\quad
\mathrm{Cov}\big(\tilde{\mathbf f}'(\mathbf{x'})\big)=\mathrm{Cov}\big(\tilde{\mathbf f}'\!\circ \mathbf g'(\mathbf{s}')\big),
\]
\[
\mathrm{Cov}\!\big(\tilde{\mathbf f}(\mathbf{x}),\,\tilde{\mathbf f}'(\mathbf{x'})\big)
=\mathrm{Cov}\!\big(\tilde{\mathbf f}\!\circ \mathbf g(\mathbf{s}),\,\tilde{\mathbf f}'\!\circ \mathbf g'(\mathbf{s}')\big).
\]

\item \textbf{Whitening Preservation.} If $\tilde{\mathbf f}\in\tilde{\mathcal F}_\mathcal X$ and $\tilde{\mathbf f}'\in\tilde{\mathcal{F}}'_{\mathcal{X}'}$ are whitened with respect to $P_{\mathbf x}$ and $P_{\mathbf x'}$, respectively. Then
$\tilde{\mathbf f}\!\circ \mathbf g$ and
$\tilde{\mathbf f}'\!\circ \mathbf g'$ are whitened under $P_{\mathbf s}$ and $P_{\mathbf s'}$.
\end{enumerate}
\end{lemma}
\begin{proof} \ 

\textbf{1. Square-integrability.}

Let $\tilde{\mathbf f} \in \tilde{\mathcal{F}}_{\mathcal{X}}, \tilde{\mathbf f}' \in \tilde{\mathcal{F}}'_{\mathcal{X}'}$ 
be Borel-measurable and square-integrable under $P_{\mathbf x}$ and $P_{\mathbf x'}$, 
respectively. That is,
\[
\int_{\mathcal X}\|\tilde{\mathbf f}(\mathbf a)\|^2\,dP_{\mathbf x}(\mathbf a)<\infty, 
\qquad
\int_{\mathcal X'}\|\tilde{\mathbf f}'(\mathbf a)\|^2\,dP_{\mathbf x'}(\mathbf b)<\infty.
\]
Because $P_{\mathbf x}=\mathbf g_\# P_{\mathbf s}$ and 
$P_{\mathbf x'}=\mathbf g'_\# P_{\mathbf s'}$ are the pushforward meausre of $P_{\mathbf s}$ and $P_{\mathbf s'}$. The pushforward identity gives, 
for any nonnegative measurable $\bm\varphi \in L^1(P_{\mathbf x}), \bm\psi\in L^1(P_{\mathbf x'})$,
\begin{equation}
   \label{equ:1_1}
    \int_{\mathcal X} \bm\varphi(\mathbf a)\, dP_{\mathbf x}(\mathbf a)
=\int_{\mathcal S} (\bm\varphi\circ\mathbf g)(\mathbf b)\, dP_{\mathbf s}(\mathbf b)\quad\text{and}\quad \int_{\mathcal X'} \bm\psi(\mathbf a)\, dP_{\mathbf x'}(\mathbf a)
=\int_{\mathcal S} (\bm\psi\circ\mathbf g)(\mathbf b)\, dP_{\mathbf s'}(\mathbf b).
\end{equation}
Applying this to $\bm\varphi(\mathbf x)=\|\mathbf f(\mathbf x)\|^2$ yields
\begin{align*}
\label{equ:1_2}
    \mathbb E\big[\|(\tilde{\mathbf f}\!\circ \mathbf g)(\mathbf s)\|^2\big]
=\int_{\mathcal S}\|(\tilde{\mathbf f}\circ \mathbf g)(\mathbf a)\|^2\,dP_{\mathbf s}(\mathbf a)
=\int_{\mathcal X}\|\tilde{\mathbf f}(\mathbf a)\|^2\,dP_{\mathbf x}(\mathbf a)<\infty,\\
\mathbb E\big[\|(\tilde{\mathbf f}'\!\circ \mathbf g')(\mathbf s')\|^2\big]
=\int_{\mathcal S}\|(\tilde{\mathbf f}'\circ \mathbf g')(\mathbf b)\|^2\,dP_{\mathbf s'}(\mathbf b)
=\int_{\mathcal X'}\|\tilde{\mathbf f}'(\mathbf b)\|^2\,dP_{\mathbf x}(\mathbf b)<\infty.
\end{align*}
Thus $\tilde{\mathbf f}\!\circ \mathbf g\in L^2(P_{\mathbf s};\mathbb R^{d_\mathcal{Z}})$ and 
$\tilde{\mathbf f}'\!\circ \mathbf g'\in L^2(P_{\mathbf s'};\mathbb R^{d_\mathcal{Z}})$, as claimed.

\textbf{2. Expectation and covariance preservation.}

Let $\tilde{\mathbf f} \in \tilde{\mathcal{F}}_{\mathcal{X}}, \tilde{\mathbf f}' \in \tilde{\mathcal{F}}'_{\mathcal{X}'}$ be
Borel and square-integrable under $P_{\mathbf x}$ and $P_{\mathbf x'}$, respectively.

\emph{Expectations.}
Apply \autoref{equ:1_1} componentwise with $\bm\varphi=\tilde{\mathbf f}_j$ and $\bm\psi=\tilde{\mathbf f}'_j$:
\[
\mathbb E[\tilde{\mathbf f}(\mathbf x)] \;=\; \mathbb E[(\tilde{\mathbf f}\!\circ \mathbf g)(\mathbf s)],\qquad
\mathbb E[\tilde{\mathbf f}'(\mathbf x')] \;=\; \mathbb E[(\tilde{\mathbf f}'\!\circ \mathbf g')(\mathbf s')].
\]

\emph{Second moments and (marginal) covariances.}
Using \autoref{equ:1_1} with $\bm\varphi(\mathbf x)=\tilde{\mathbf f}(\mathbf x)\tilde{\mathbf f}(\mathbf x)^\top$ (integrable by Cauchy-Schwarz), $\forall i, j \in [1, \cdots, d_\mathcal{Z}]$,
\[
\mathbb E\!\big[\tilde{\mathbf f}(\mathbf x)\tilde{\mathbf f}(\mathbf x)^\top\big]
\;=\; \mathbb E\!\big[(\tilde{\mathbf f}\!\circ \mathbf g)(\mathbf s)\,(\tilde{\mathbf f}\!\circ \mathbf g)(\mathbf s)^\top\big].
\]
Subtracting outer products of the means yields
\[
\mathrm{Cov}\big(\tilde{\mathbf f}(\mathbf x)\big)
= \mathbb E\!\big[\tilde{\mathbf f}(\mathbf x)\tilde{\mathbf f}(\mathbf x)^\top\big]
  - \mathbb E[\tilde{\mathbf f}(\mathbf x)]\,\mathbb E[\tilde{\mathbf f}(\mathbf x)]^\top
= \mathrm{Cov}\big((\tilde{\mathbf f}\!\circ \mathbf g)(\mathbf s)\big).
\]
The same argument applies to $\tilde{\mathbf f}'$:
\[
\mathrm{Cov}\big(\tilde{\mathbf f}'(\mathbf x')\big)
= \mathrm{Cov}\big((\tilde{\mathbf f}'\!\circ \mathbf g')(\mathbf s')\big).
\]

\emph{Cross-covariances.}
For the joint law $P_{\mathbf x\mathbf x'}=(\mathbf g,\mathbf g')_\# P_{\mathbf s\mathbf s'}$ and any integrable $\Phi:\mathcal X\times\mathcal X'\to\mathbb R$, the well-established generalization of the pushforward measure to joint mappings is as follows:
\begin{equation}\label{equ:1-2}
\int_{\mathcal X\times\mathcal X'} \Phi(\mathbf a,\mathbf b)\,dP_{\mathbf x\mathbf x'}(\mathbf a,\mathbf b)
\;=\; \int_{\mathcal S\times\mathcal S'} \Phi(\mathbf g(\mathbf c),\mathbf g'(\mathbf d))\,dP_{\mathbf s\mathbf s'}(\mathbf c,\mathbf d).
\end{equation}

By \autoref{equ:1-2} with the matrix-valued integrand
$\Phi(\mathbf x)=\tilde{\mathbf f}(\mathbf x)\,\tilde{\mathbf f}'(\mathbf x')^\top$ (integrable by Cauchy--Schwarz), $\forall i, j \in [1, \cdots, d_\mathcal{Z}]$,
\[
\mathbb E\!\big[\tilde{\mathbf f}(\mathbf x)\,\tilde{\mathbf f}'(\mathbf x')^\top\big]
\;=\; \mathbb E\!\big[(\tilde{\mathbf f}\!\circ \mathbf g)(\mathbf s)\,(\tilde{\mathbf f}'\!\circ \mathbf g')(\mathbf s')^\top\big].
\]
Subtracting the product of means (already matched above) gives
\[
\mathrm{Cov}\!\big(\tilde{\mathbf f}(\mathbf x),\,\tilde{\mathbf f}'(\mathbf x')\big)
= \mathrm{Cov}\!\big((\tilde{\mathbf f}\!\circ \mathbf g)(\mathbf s),\,(\tilde{\mathbf f}'\!\circ \mathbf g')(\mathbf s')\big).
\]
All three identities are thus established.

\textbf{3. Whitening preservation.}
If $\mathbb E[\tilde{\mathbf f}(\mathbf{x})]=\mathbf 0$ and $\mathrm{Cov}(\tilde{\mathbf f}(\mathbf{x}))=\mathbf I$, then by step 2, we get
$\mathbb E[\tilde{\mathbf f}\!\circ g(\mathbf{s})]=\mathbf 0$ and
$\mathrm{Cov}(\tilde{\mathbf f}\!\circ g(\mathbf{s}))=\mathbf I$.
The argument for $\tilde{\mathbf f}'$ is identical.
\end{proof}
\begin{lemma}[Bijection under composition]
\label{lemma:comp-bijection}
Let $\mathcal S,\mathcal X,\mathcal X'$ be standard Borel spaces and $\mathbf{g}, \mathbf{g}'$ be Borel-measurable and injective.
Assume the pushforward measures $P_{\mathbf x}=\mathbf g_\# P_{\mathbf s}$ and
$P_{\mathbf x'}=\mathbf g'_\# P_{\mathbf s'}$. Define the whitened representable latent classes as in \autoref{definition:whitend_encoder}. Then:
\begin{enumerate}[label=\arabic*.]
\item \textbf{Composition isometries.}
The composition operators
\[
C_{\mathbf g}:L^2(P_{\mathbf x};\mathbb R^{d_{\mathcal Z}})\to L^2(P_{\mathbf s};\mathbb R^{d_{\mathcal Z}}),\quad 
\tilde{\mathbf f}\mapsto \tilde{\mathbf f}\circ \mathbf g,
\]
\[
C_{\mathbf g'}:L^2(P_{\mathbf x'};\mathbb R^{d_{\mathcal Z}})\to L^2(P_{\mathbf s'};\mathbb R^{d_{\mathcal Z}}),\quad 
\tilde{\mathbf f}'\mapsto \tilde{\mathbf f}'\circ \mathbf g',
\]
are linear isometries. 
\item \textbf{Bijection on whitened classes.}
The map
\[
\Psi:\ \tilde{\mathcal F}_{\mathcal X}\times \tilde{\mathcal F}'_{\mathcal X'}
\longrightarrow
\tilde{\mathcal F}_{\mathcal S}\times \tilde{\mathcal F}'_{\mathcal S'},\qquad
(\tilde{\mathbf f},\tilde{\mathbf f}')
\longmapsto
(\tilde{\mathbf f}\!\circ \mathbf g,\ \tilde{\mathbf f}'\!\circ \mathbf g')
\]
is a bijection (modulo null sets), with inverse
\[
\Psi^{-1}:\ (\tilde{\mathbf h},\tilde{\mathbf h}')
\longmapsto
(\tilde{\mathbf h}\!\circ \mathbf g^{-1},\ \tilde{\mathbf h}'\!\circ \mathbf g'^{-1}),
\]
where $\mathbf g^{-1}$ and $\mathbf g'^{-1}$ are the Borel inverses on the Borel images 
$\mathbf g(\mathcal S)$ and $\mathbf g'(\mathcal S')$, respectively. For any $Q,Q'\in O(d_{\mathcal Z})$,
\[
\Psi(Q\tilde{\mathbf f},Q'\tilde{\mathbf f}')=(Q(\tilde{\mathbf f}\!\circ \mathbf g),\ Q'(\tilde{\mathbf f}'\!\circ \mathbf g')).
\]
\end{enumerate}
\end{lemma}
\begin{proof} \ 

\textbf{1. Isometry.}

Since $\forall \tilde{\mathbf f}\in \tilde{F}_{\mathcal X}$,
\[
\| \tilde{\mathbf f}\circ \mathbf g\|_{L^2(P_{\mathbf s})}^2
=\int_{\mathcal S}\| (\tilde{\mathbf f} \circ \mathbf g)(\mathbf s)\|^2\,dP_{\mathbf s}(\mathbf s)
=\int_{\mathcal X}\| \tilde{\mathbf f}(\mathbf x)\|^2\,d(\mathbf g_\# P_{\mathbf s})(\mathbf x)
=\|\tilde{\mathbf f}\|_{L^2(P_{\mathbf x})}^2,
\]
so $C_{\mathbf g}$ is a linear isometry. The argument for $C_{\mathbf g'}$ is identical.

\textbf{2. Measurable inverses on images.}

Since $\mathcal S,\mathcal X$ are standard Borel and $\mathbf g$ is Borel and injective, 
the Lusin-Souslin theorem implies $\mathbf g(\mathcal S)$ is Borel in $\mathcal X$ and 
$\mathbf g^{-1}:\mathbf g(\mathcal S)\to \mathcal S$ is Borel. The same holds for $\mathbf g'$.

\textbf{3. Surjectivity.}

Let $\tilde{\mathbf h}\in\tilde{\mathcal F}_{\mathcal S}$. Define 
$\tilde{\mathbf f}:=\tilde{\mathbf h}\!\circ \mathbf g^{-1}$ on $\mathbf g(\mathcal S)$ and extend arbitrarily to a Borel function on $\mathcal X$.
Then, by Step 2, $\tilde{\mathbf f}\in\tilde{\mathcal F}_{\mathcal X}$ and
$(\tilde{\mathbf f}\!\circ \mathbf g)(s)=\tilde{\mathbf h}(s)$.
The same construction gives the primed component. Thus $\Psi$ is surjective

\textbf{4. Injectivity.}

If $\Phi(\tilde{\mathbf f}_1,\tilde{\mathbf f}_1')=\Phi(\tilde{\mathbf f}_2,\tilde{\mathbf f}_2')$, then 
$\tilde{\mathbf f}_1\!\circ \mathbf g=\tilde{\mathbf f}_2\!\circ \mathbf g$, hence
$\tilde{\mathbf f}_1=\tilde{\mathbf f}_2$. by applying $\mathbf g^{-1}$ on $\mathbf g(\mathcal S)$; similarly for the primed side.
Therefore $\Psi$ is injective.

Combining the steps yields the stated bijection and properties.
\end{proof}
\begin{lemma}[Continuity of the Whitening Map]
\label{lemma:whitening-continuity}
Let $\{\,\mathbf f_n\,\}_{n\ge1}\subset L^2(P_{\mathbf x};\mathbb R^{d_\mathcal{Z}})$ be a sequence of
square-integrable vector-valued functions such that $\mathbf f_n\to\mathbf f$ in $L^2(P_{\mathbf x})$.
Denote
\[
\bm\mu_n := \mathbb E[\mathbf f_n(\mathbf x)], \qquad
\bm\mu := \mathbb E[\mathbf f(\mathbf x)], \qquad
\mathbf\Sigma_{\mathbf f_n} := \mathrm{Cov}(\mathbf f_n(\mathbf x)), \qquad
\mathbf\Sigma_{\mathbf f} := \mathrm{Cov}(\mathbf f(\mathbf x)).
\]
Assume $\mathbf\Sigma_{\mathbf f}\succ 0$, and let
\[
\mathbf W_{\mathbf f_n} := \mathbf\Sigma_{\mathbf f_n}^{-1/2}, \qquad
\mathbf W_{\mathbf f} := \mathbf\Sigma_{\mathbf f}^{-1/2}.
\]
Then, for all $n$ sufficiently large,
$\mathbf\Sigma_{\mathbf f_n}\succ 0$ and
\begin{equation}\label{equ:whitened-convergence}
\mathbf W_{\mathbf f_n}\big(\mathbf f_n - \bm\mu_n\big)
\;\xrightarrow[n\to\infty]{L^2}\;
\mathbf W_{\mathbf f}\big(\mathbf f - \bm\mu\big).
\end{equation}
Moreover, if $\mathbf U_{\mathbf f_n}\in\mathrm{GL}(d_\mathcal{Z})$
are any whiteners satisfying $\mathbf U_{\mathbf f_n}\mathbf\Sigma_{\mathbf f_n}\mathbf U_{\mathbf f_n}^\top = \mathbf I_{d_\mathcal{Z}}$,
then there exist orthogonal matrices $\mathbf Q_n\in O(d_\mathcal{Z})$
such that
\[
\mathbf Q_n\,\mathbf U_{\mathbf f_n}(\mathbf f_n-\bm\mu_n)
\;\xrightarrow[n\to\infty]{L^2}\;
\mathbf W_{\mathbf f}(\mathbf f-\bm\mu).
\]
\end{lemma}

\begin{proof} \ 

\textbf{1. Convergence of moments.}

Since $\mathbf f_n \to \mathbf f$ in $L^2(P_{\mathbf x})$, we have
$\|\mathbf f_n - \mathbf f\|_{L^2}^2
   = \mathbb E[\|\mathbf f_n - \mathbf f\|^2] \to 0$.
By the Cauchy-Schwarz inequality,
\[
\|\bm\mu_n - \bm\mu\|
= \|\mathbb E[\mathbf f_n - \mathbf f]\|
\le \mathbb E[\|\mathbf f_n - \mathbf f\|]
\le \|\mathbf f_n - \mathbf f\|_{L^2}
\;\longrightarrow\; 0.
\]

For the second moment matrices,
each entry of $\mathbb E[\mathbf f_n\mathbf f_n^\top]$ converges to the corresponding entry of
$\mathbb E[\mathbf f\mathbf f^\top]$ because
\[
\|\mathbf f_n\mathbf f_n^\top - \mathbf f\mathbf f^\top\|_{\mathrm{F}}
\;\le\;
(\|\mathbf f_n\|+\|\mathbf f\|)\,\|\mathbf f_n-\mathbf f\|,
\]
and taking expectations yields
\[
\|\mathbb E[\mathbf f_n\mathbf f_n^\top] - \mathbb E[\mathbf f\mathbf f^\top]\|
\le
\mathbb E[\|\mathbf f_n\mathbf f_n^\top - \mathbf f\mathbf f^\top\|]
\le
(\|\mathbf f_n\|_{L^2}+\|\mathbf f\|_{L^2})\,\|\mathbf f_n-\mathbf f\|_{L^2}
\;\longrightarrow\; 0,
\]
so $\mathbb E[\mathbf f_n\mathbf f_n^\top]\to\mathbb E[\mathbf f\mathbf f^\top]$ in operator norm.

Combining the above,
\[
\mathbf\Sigma_{\mathbf f_n}
=\mathbb E[(\mathbf f_n-\bm\mu_n)(\mathbf f_n-\bm\mu_n)^\top]
=\mathbb E[\mathbf f_n\mathbf f_n^\top]
   - \bm\mu_n\bm\mu_n^\top
\;\xrightarrow[n\to\infty]{\|\cdot\|}\;
\mathbb E[\mathbf f\mathbf f^\top] - \bm\mu\bm\mu^\top
=\mathbf\Sigma_{\mathbf f}.
\]

\textbf{2. Positive definiteness and bounded whitening operators.}

Since $\mathbf\Sigma_{\mathbf f}\succ0$, all its eigenvalues are positive.
Let $\sigma_{\min}>0$ denote its smallest eigenvalue, i.e.,
\[
\sigma_{\min} := \sigma_{\min}(\mathbf\Sigma_{\mathbf f})
\quad\text{so that}\quad
\mathbf\Sigma_{\mathbf f}\succeq \sigma_{\min}\mathbf I.
\]
Because $\mathbf\Sigma_{\mathbf f_n}\to\mathbf\Sigma_{\mathbf f}$ in operator norm,
we have
\[
\|\mathbf\Sigma_{\mathbf f_n}-\mathbf\Sigma_{\mathbf f}\|
\;\xrightarrow[n\to\infty]{}\;0.
\]
By Weyl's eigenvalue perturbation inequality,
\[
\forall i \in [1,\cdots,d_\mathcal{Z}],\quad |\sigma_i(\mathbf\Sigma_{\mathbf f_n})-\sigma_i(\mathbf\Sigma_{\mathbf f})|
\le \|\mathbf\Sigma_{\mathbf f_n}-\mathbf\Sigma_{\mathbf f}\|.
\]
Hence for sufficiently large $n$,
\[
\sigma_{\min}(\mathbf\Sigma_{\mathbf f_n})
\ge \sigma_{\min}(\mathbf\Sigma_{\mathbf f})-\|\mathbf\Sigma_{\mathbf f_n}-\mathbf\Sigma_{\mathbf f}\|
\ge \tfrac{\sigma_{\min}}{2}>0.
\]
Therefore each $\mathbf\Sigma_{\mathbf f_n}$ is symmetric positive definite and satisfies
\[
\mathbf\Sigma_{\mathbf f_n}\succeq \tfrac{\sigma_{\min}}{2}\mathbf I.
\]
Taking inverse square roots preserves the Löwner order on the SPD cone, so
\[
\|\mathbf W_{\mathbf f_n}\|
=\|\mathbf\Sigma_{\mathbf f_n}^{-1/2}\|
=\big(\sigma_{\min}(\mathbf\Sigma_{\mathbf f_n})\big)^{-1/2}
\le (\sigma_{\min}/2)^{-1/2}.
\]
Thus the sequence $\{\mathbf W_{\mathbf f_n}\}$ is uniformly bounded in operator norm.

\textbf{3. Continuity of matrix square roots.}

Since $\mathbf\Sigma_{\mathbf f}\succ0$ and 
$\mathbf\Sigma_{\mathbf f_n}\to \mathbf\Sigma_{\mathbf f}$ in operator norm, 
there exist constants
\[
m := \tfrac{\sigma_{\min}(\mathbf\Sigma_{\mathbf f})}{2} > 0,
\qquad 
M := \|\mathbf\Sigma_{\mathbf f}\| + 1,
\]
and an integer $N$ such that for all $n\ge N$,
the spectra satisfy $\sigma(\mathbf\Sigma_{\mathbf f_n}),\,\sigma(\mathbf\Sigma_{\mathbf f}) \subset [m,M]$.

Fix $\varepsilon>0$.  
By the Weierstrass approximation theorem, there exists a polynomial $p$ such that
\[
\sup_{\sigma\in[m,M]}\big|\sigma^{-1/2}-p(\sigma)\big| < \varepsilon/3.
\]
Using the continuous functional calculus for self-adjoint matrices, we have
\[
\big\|\mathbf\Sigma_{\mathbf f_n}^{-1/2}-\mathbf\Sigma_{\mathbf f}^{-1/2}\big\|
\;\le\;
\big\|\mathbf\Sigma_{\mathbf f_n}^{-1/2}-p(\mathbf\Sigma_{\mathbf f_n})\big\|
+\big\|p(\mathbf\Sigma_{\mathbf f_n})-p(\mathbf\Sigma_{\mathbf f})\big\|
+\big\|p(\mathbf\Sigma_{\mathbf f})-\mathbf\Sigma_{\mathbf f}^{-1/2}\big\|.
\]
The first and third terms are each bounded by 
$\sup_{\sigma\in[m,M]}|\sigma^{-1/2}-p(\sigma)|<\varepsilon/3$.
For the middle term, note that $p(A)=\sum_k a_k A^k$ is a finite polynomial,
hence norm-continuous; since $\mathbf\Sigma_{\mathbf f_n}\to\mathbf\Sigma_{\mathbf f}$,
we have $\|p(\mathbf\Sigma_{\mathbf f_n})-p(\mathbf\Sigma_{\mathbf f})\|<\varepsilon/3$ for all large $n$.
Therefore $\|\mathbf\Sigma_{\mathbf f_n}^{-1/2}-\mathbf\Sigma_{\mathbf f}^{-1/2}\|<\varepsilon$ for $n$ sufficiently large, i.e.
\[
\mathbf W_{\mathbf f_n}=\mathbf\Sigma_{\mathbf f_n}^{-1/2}
\;\xrightarrow[n\to\infty]{\|\cdot\|}
\mathbf\Sigma_{\mathbf f}^{-1/2}=\mathbf W_{\mathbf f}.
\]


\textbf{4. $L^2$-continuity of whitening.}

Define the difference
\[
\bm\Delta_n
:= \mathbf W_{\mathbf f_n}(\mathbf f_n-\bm\mu_n)
   - \mathbf W_{\mathbf f}(\mathbf f-\bm\mu)
   = (\mathbf W_{\mathbf f_n}-\mathbf W_{\mathbf f})(\mathbf f-\bm\mu)
     + \mathbf W_{\mathbf f_n}\big[(\mathbf f_n-\bm\mu_n)-(\mathbf f-\bm\mu)\big].
\]
Taking $L^2$ norms and applying Cauchy-Schwarz,
\[
\|\bm\Delta_n\|_{L^2}
\le \|\mathbf W_{\mathbf f_n}-\mathbf W_{\mathbf f}\|\,\|\mathbf f-\bm\mu\|_{L^2}
   + \|\mathbf W_{\mathbf f_n}\|\,\|(\mathbf f_n-\bm\mu_n)-(\mathbf f-\bm\mu)\|_{L^2}.
\]
The first term tends to $0$ because $\mathbf W_{\mathbf f_n}\to\mathbf W_{\mathbf f}$
and $\|\mathbf f-\bm\mu\|_{L^2}<\infty$.
For the second term, $\|\mathbf W_{\mathbf f_n}\|$ is uniformly bounded and
\[
\|(\mathbf f_n-\bm\mu_n)-(\mathbf f-\bm\mu)\|_{L^2}
\le \|\mathbf f_n-\mathbf f\|_{L^2} + \|\bm\mu_n-\bm\mu\| \;\to\; 0.
\]
Thus $\|\bm\Delta_n\|_{L^2}\to0$, proving \autoref{equ:whitened-convergence}.

\textbf{5. Orthogonal alignment for general whiteners.}

Let $\mathbf U_{\mathbf f_n}\in\mathrm{GL}(d_\mathcal{Z})$ satisfy
$\mathbf U_{\mathbf f_n}\mathbf\Sigma_{\mathbf f_n}\mathbf U_{\mathbf f_n}^\top = \mathbf I_{d_\mathcal{Z}}$.
Using the polar decomposition,
\[
\mathbf U_{\mathbf f_n}\mathbf\Sigma_{\mathbf f_n}^{1/2}
= \mathbf Q_n,\qquad \mathbf Q_n\in O(d_\mathcal{Z}),
\]
which implies $\mathbf Q_n\mathbf U_{\mathbf f_n}
=\mathbf\Sigma_{\mathbf f_n}^{-1/2}=\mathbf W_{\mathbf f_n}$.
Therefore,
\[
\mathbf Q_n\,\mathbf U_{\mathbf f_n}(\mathbf f_n-\bm\mu_n)
= \mathbf W_{\mathbf f_n}(\mathbf f_n-\bm\mu_n)
\;\xrightarrow[n\to\infty]{L^2}\;
\mathbf W_{\mathbf f}(\mathbf f-\bm\mu)
\]
by the first part, completing the proof.
\end{proof}
\begin{lemma}[Reparameterization Invariance and Representational Universality of CCA]
\label{lemma:rep-inv}
Let $\mathcal S,\mathcal X,\mathcal X'$ be standard Borel spaces and $\mathbf{g}, \mathbf{g}'$ be Borel-measurable and injective. In addition, let $\tilde{\mathcal F}_{\mathcal{X}},\tilde{\mathcal F}'_{\mathcal{X}'}$ be the whitened
encoder classes of \autoref{definition:whitend_encoder} and assume that the base encoders $\mathbf{f}, \mathbf{f}'$ underlying $\tilde{\mathcal F}_{\mathcal X},\tilde{\mathcal F}'_{\mathcal{X}'}$
are dense in $L^2(P_{\mathbf{x}};\mathbb R^{d_{\mathcal{Z}}})$ and $L^2(P_{\mathbf{x}'}; \mathbb R^{d_{\mathcal{Z}}})$, respectively.
Define whitened representable latent classes:
\[
\tilde{\mathcal F}_{\mathcal{S}}:=\{\tilde{\mathbf f}\!\circ \mathbf{g}:\ \tilde{\mathbf f}\in\tilde{\mathcal F}_{\mathcal{X}}\},
\qquad
\tilde{\mathcal F}'_{\mathcal{S}}:=\{\tilde{\mathbf f}'\!\circ \mathbf{g}':\ \tilde{\mathbf f}'\in\tilde{\mathcal F}'_{\mathcal{X}'}\},
\]
and the whitened feasible latent classes:
\begin{align*}
    \hat{\mathcal F}_{\mathcal S} := \{\, \hat{\mathbf h} \in L^2(P_{\mathbf s};d_\mathcal Z): 
\mathbb E[\hat{\mathbf h}(\mathbf s)]=0,\ \mathrm{Cov}(\hat{\mathbf h}(\mathbf s))=\mathbf I_{d_\mathcal Z} \,\},\\
\hat{\mathcal F'_\mathcal S} := \{\, \hat{\mathbf h} \in L^2(P_{\mathbf s'};d_\mathcal Z): 
\mathbb E[\hat{\mathbf h}(\mathbf s')]=0,\ \mathrm{Cov}(\hat{\mathbf h}(\mathbf s'))=\mathbf I_{d_\mathcal Z} \,\}.
\end{align*}

Further define the population CCA objective with respect to whitened feasible latent classes $\hat{\mathcal F}_\mathcal S,\hat{\mathcal F'_\mathcal S}$ on their domains $\mathcal{S}\times\mathcal{S}$:
\begin{equation}
\label{equ:CCA_objective_source}
J_\mathcal{S}(\hat{\mathbf h}, \hat{\mathbf h'})=\sum_{i=1}^{d_\mathcal{Z}}
\sigma_i\big(\mathrm{Cov}(\hat{\mathbf h}(\mathbf{s}),\,\hat{\mathbf h'}(\mathbf{s'}))\big),
\end{equation}
where $\sigma_i(\cdot)$ denotes the $i$-th singular value. Then the following properties hold.

\begin{enumerate}[label=\arabic*., ref=\arabic*]
\item \textbf{Objective Preservation.}
$\forall (\tilde{\mathbf f},\tilde{\mathbf f}')\in \tilde{\mathcal{F}}_{\mathcal{X}}\times\tilde{\mathcal{F}}'_{\mathcal{X}'}$,
\(
J(\tilde{\mathbf f},\tilde{\mathbf f}')\;=\;
J_\mathcal S\big(\tilde{\mathbf f}\!\circ \mathbf{g},\ \tilde{\mathbf f}'\!\circ \mathbf{g}'\big).
\)

\item \textbf{Maximizer Correspondence.}\label{item:maximizer_correspondence}
If $(\tilde{\mathbf{f}}^*, \tilde{\mathbf{f}'}^*)$ is a maximizer of $J$ in 
\autoref{equ:CCA_objective} over 
$ \tilde{\mathcal{F}}_\mathcal{X}\times \tilde{\mathcal{F}}_{\mathcal{X}'}$, 
then its composition 
$(\tilde{\mathbf{f}}\circ \mathbf{g}, \tilde{\mathbf{f}'}\circ \mathbf{g}')$ 
is also a maximizer of $J_\mathcal{S}$ in \autoref{equ:CCA_objective_source} 
over $ \tilde{\mathcal{F}}_\mathcal{S}\times \tilde{\mathcal{F}}_{\mathcal{S}'}$, 
and conversely.

\item \textbf{Representation Universality.} For any $\hat{\mathbf h}\in\hat{\mathcal F}_{\mathcal S}$, $\hat{\mathbf h}'\in\hat{\mathcal F}'_{\mathcal S}$, and any $\epsilon,\epsilon'>0$, 
there exist $\tilde{\mathbf h}\in\tilde{\mathcal F}_{\mathcal S}$ and $\tilde{\mathbf h}'\in\tilde{\mathcal F}'_{\mathcal S}$ such that
\begin{align*}
    \mathbb E\!\left[\|\hat{\mathbf h}(\mathbf s)-\tilde{\mathbf h}(\mathbf s)\|^2\right]& < \epsilon^2 \quad \text{and} \quad
\mathbb E\!\left[\|\hat{\mathbf h}'(\mathbf s')-\tilde{\mathbf h}'(\mathbf s')\|^2\right] < {\epsilon'}^2.
\end{align*}
\end{enumerate}
\end{lemma}
\begin{proof}
We work in $L^2$, identifying functions that agree $P$-a.s. Let $J$ be the CCA objective on $\mathcal X\times\mathcal X'$
from \autoref{equ:CCA_objective} and $J_\mathcal S$ the latent CCA objective on $\mathcal S\times\mathcal S'$
from \autoref{equ:CCA_objective_source}.

\medskip\noindent
\textbf{1. Objective preservation.}

By the whitening preservation of \autoref{lemma:pushforward}.3, the two cross-covariance
matrices $\mathrm{Cov}(\tilde{\mathbf f},\tilde{\mathbf f}'), \mathrm{Cov}(\tilde{\mathbf f}\!\circ \mathbf g,\ \tilde{\mathbf f}'\!\circ \mathbf g')$ coincide, hence so do their sum of the singular values:
\[
J(\tilde{\mathbf f},\tilde{\mathbf f}') \;=\; J_\mathcal S(\tilde{\mathbf f}\!\circ \mathbf g,\ \tilde{\mathbf f}'\!\circ \mathbf g').
\]

\medskip\noindent
\textbf{2. Maximizer correspondence (between representable classes).}

Since the composition map
\[
\Psi:\ \tilde{\mathcal F}_{\mathcal X}\times \tilde{\mathcal F}'_{\mathcal X'} \to
\tilde{\mathcal F}_{\mathcal S}\times \tilde{\mathcal F}'_{\mathcal S},\qquad
\Psi(\tilde{\mathbf f},\tilde{\mathbf f}')=(\tilde{\mathbf f}\!\circ \mathbf g,\ \tilde{\mathbf f}'\!\circ \mathbf g').
\] is bijective, given by \autoref{lemma:comp-bijection}.2, and the objective preservation in \autoref{lemma:rep-inv}.1 implies 
\[
\Phi\!\left(\operatorname*{argmax}_{\tilde{\mathcal F}_{\mathcal X}\times \tilde{\mathcal F}'_{\mathcal X'}} J\right)
\;=\;
\operatorname*{argmax}_{\tilde{\mathcal F}_{\mathcal S}\times \tilde{\mathcal F}'_{\mathcal S}} J_\mathcal S.
\]
If $(\tilde{\mathbf f}^*,\tilde{\mathbf f}'^{*})$ is a maximizer on
$\tilde{\mathcal F}_{\mathcal X}\times \tilde{\mathcal F}'_{\mathcal X'}$, then
$(\tilde{\mathbf f}^*\!\circ \mathbf g,\tilde{\mathbf f}'^{*}\!\circ \mathbf g')$ is also a maximizer on
$\tilde{\mathcal F}_{\mathcal S}\times \tilde{\mathcal F}'_{\mathcal S}$, and conversely.

\paragraph{3. Representation Universality.}\ 

Fix $\hat{\mathbf h}\in\hat{\mathcal F}_{\mathcal S}$ and $\epsilon>0$. 
We show there exists $\tilde{\mathbf h}\in\tilde{\mathcal F}_{\mathcal S}$ with 
$\mathbb E\!\left[\|\hat{\mathbf h}(\mathbf s)-\tilde{\mathbf h}(\mathbf s)\|^2\right]<\epsilon^2$.
(The primed side is identical.)

\medskip\noindent
\textit{(i) Pull back $\hat{\mathbf h}$ to the observation domain.}
By \autoref{lemma:comp-bijection}.2 (measurable inverses on images), the Borel inverse
$\mathbf g^{-1}:\mathbf g(\mathcal S)\to\mathcal S$ exists. Define on $\mathbf g(\mathcal S)$
\[
\hat{\mathbf h}_{\mathcal X}:=\hat{\mathbf h}\circ \mathbf g^{-1},
\]
and extend it arbitrarily to a Borel function on $\mathcal X$ (the extension does not affect $L^2(P_{\mathbf x})$ since $P_{\mathbf x}$ is supported on $\mathbf g(\mathcal S)$).
By \autoref{lemma:pushforward}.2-3 (expectation/covariance preservation and whitening preservation),
$\hat{\mathbf h}\in\hat{\mathcal F}_{\mathcal S}$ implies 
\[
\hat{\mathbf h}_{\mathcal X}\in \hat{\mathcal F}_{\mathcal X}
:=\big\{\mathbf u\in L^2(P_{\mathbf x};\mathbb R^{d_{\mathcal Z}}):\ 
\mathbb E[\mathbf u(\mathbf x)]=\mathbf 0,\ \mathrm{Cov}(\mathbf u(\mathbf x))=\mathbf I\big\}.
\]

\medskip\noindent
\textit{(ii) Approximate in $L^2(P_{\mathbf x})$ by base encoders.}
By the density assumption in \autoref{definition:whitend_encoder}, there exists a sequence
$\{\mathbf f_n\}_{n\ge1}\subset\mathcal H_{\mathcal X}$ such that
\[
\|\mathbf f_n-\hat{\mathbf h}_{\mathcal X}\|_{L^2(P_{\mathbf x})}\xrightarrow[n\to\infty]{}0.
\]
Write $\bm\mu_n=\mathbb E[\mathbf f_n(\mathbf x)]$ and 
$\mathbf\Sigma_{\mathbf f_n}=\mathrm{Cov}(\mathbf f_n(\mathbf x))$.
Since $\hat{\mathbf h}_{\mathcal X}\in L^2$ and $\mathbf f_n\to\hat{\mathbf h}_{\mathcal X}$ in $L^2$,
\autoref{lemma:whitening-continuity}.1 implies $\bm\mu_n\to\mathbf 0$ and 
$\mathbf\Sigma_{\mathbf f_n}\to\mathbf I$ in operator norm; in particular,
$\mathbf\Sigma_{\mathbf f_n}\succ0$ for all large $n$.

\medskip\noindent
\textit{(iii) Whiten the approximants and pass limits through whitening.}
Let $\mathbf W_{\mathbf f_n}=\mathbf\Sigma_{\mathbf f_n}^{-1/2}$ and define the whitened encoders
\[
\tilde{\mathbf f}_n:=\mathbf W_{\mathbf f_n}\big(\mathbf f_n-\bm\mu_n\big)\in\tilde{\mathcal F}_{\mathcal X}.
\]
By \autoref{lemma:whitening-continuity}.4,
\[
\tilde{\mathbf f}_n \xrightarrow[n\to\infty]{L^2(P_{\mathbf x})} 
\mathbf\Sigma_{\hat{\mathbf h}_{\mathcal X}}^{-1/2}\big(\hat{\mathbf h}_{\mathcal X}-\mathbb E[\hat{\mathbf h}_{\mathcal X}]\big)
=\hat{\mathbf h}_{\mathcal X},
\]
since $\hat{\mathbf h}_{\mathcal X}$ is already whitened (mean $0$, covariance $\mathbf I$).
Hence there exists $N$ such that $\|\tilde{\mathbf f}_N-\hat{\mathbf h}_{\mathcal X}\|_{L^2(P_{\mathbf x})}<\epsilon$.

\medskip\noindent
\textit{(iv) Push forward to the latent domain and conclude.}
Set $\tilde{\mathbf h}:=\tilde{\mathbf f}_N\circ \mathbf g\in\tilde{\mathcal F}_{\mathcal S}$ by definition of the representable class.
By \autoref{lemma:comp-bijection}.1 (composition isometry),
\[
\|\tilde{\mathbf h}-\hat{\mathbf h}\|_{L^2(P_{\mathbf s})}
= \|\tilde{\mathbf f}_N\circ \mathbf g - \hat{\mathbf h}_{\mathcal X}\circ \mathbf g\|_{L^2(P_{\mathbf s})}
= \|\tilde{\mathbf f}_N-\hat{\mathbf h}_{\mathcal X}\|_{L^2(P_{\mathbf x})}
< \epsilon.
\]
Therefore $\mathbb E[\|\hat{\mathbf h}(\mathbf s)-\tilde{\mathbf h}(\mathbf s)\|^2]<\epsilon^2$.
Since $\epsilon>0$ was arbitrary, the claim follows. The proof on the primed side is identical,
yielding $\tilde{\mathbf h}'\in\tilde{\mathcal F}'_{\mathcal S}$ with
$\mathbb E[\|\hat{\mathbf h}'(\mathbf s')-\tilde{\mathbf h}'(\mathbf s')\|^2]<{\epsilon'}^2$ for any $\epsilon'>0$.
\end{proof}

\suppsubsection{Normalized Hermite-Mehler expansion of joint Gaussian distribution}
\begin{lemma}[Hermite--Mehler expansion of a bivariate Gaussian pair]
\label{lemma:hermite_expansion_bivariate_gaussian}
Let $(S,S')$ be jointly Gaussian with means $\mu_S,\mu_{S'}$, variances
$\sigma_S^2,\sigma_{S'}^2$, and correlation coefficient $t\in(-1,1)$.
Define the standardized coordinates
\[
U:=\frac{S-\mu_S}{\sigma_S},
\qquad
V:=\frac{S'-\mu_{S'}}{\sigma_{S'}}.
\]
Let $He_n$ denote the probabilists' Hermite polynomials and define
\[
\psi_n(z):=\frac{1}{\sqrt{n!}}\,He_n(z),
\qquad n\in\mathbb N_0.
\]
Let
\[
\nu(dz)=\phi(z)\,dz
=
\frac{1}{\sqrt{2\pi}}e^{-z^2/2}\,dz
\]
denote the standard Gaussian measure on $\mathbb R$. Then
$\{\psi_n\}_{n\ge 0}$ is an orthonormal basis of $L^2(\nu)$, namely,
\[
\int_{\mathbb R}\psi_m(z)\psi_n(z)\,\nu(dz)
=
\delta_{mn}.
\]
The joint density of $(U,V)$ admits the Hermite--Mehler expansion
\begin{equation}
\label{eq:Hermite-expansion-standardized-density}
p_{U,V}(u,v)
=
\phi(u)\phi(v)
\sum_{n=0}^{\infty} t^n\,\psi_n(u)\psi_n(v),
\end{equation}
Equivalently,
\begin{equation}
\label{eq:Hermite-expansion-density}
p_{S,S'}(s,s')
=\frac{1}{\sigma_S\sigma_{S'}}p_{U,V}(u,v)
=
\frac{1}{\sigma_S\sigma_{S'}}
\phi(u)\phi(v)
\sum_{n=0}^{\infty} t^n\,\psi_n(u)\psi_n(v).
\end{equation}
Finally, the Hermite polynomials diagonalize the covariance operator of the standardized pair $(U,V)$:
\[
\mathbb E[\psi_m(U)\psi_n(V)]
=
t^n\,\delta_{mn},
\qquad m,n\in\mathbb N_0.
\]
\end{lemma}

\begin{proof}
The classical Mehler formula for the probabilists' Hermite polynomials states
that, for $|t|<1$,
\[
\sum_{n=0}^{\infty}\frac{t^n}{n!}He_n(u)He_n(v)
=
\frac{1}{\sqrt{1-t^2}}
\exp\!\left(
\frac{2tuv-t^2(u^2+v^2)}{2(1-t^2)}
\right).
\]
Multiplying both sides by $\phi(u)\phi(v)$ gives
\begin{align*}
\phi(u)\phi(v)
\sum_{n=0}^{\infty}\frac{t^n}{n!}He_n(u)He_n(v)
&=
\frac{1}{2\pi\sqrt{1-t^2}}
\exp\!\left(
-\frac{u^2-2tuv+v^2}{2(1-t^2)}
\right).
\end{align*}
The right-hand side is precisely the density of a centered bivariate
Gaussian vector with unit variances and correlation coefficient $t$.
Hence,
\[
p_{U,V}(u,v)
=
\phi(u)\phi(v)
\sum_{n=0}^{\infty}\frac{t^n}{n!}He_n(u)He_n(v).
\]
Since
\[
\psi_n(u)\psi_n(v)
=
\frac{1}{n!}He_n(u)He_n(v),
\]
we obtain
\[
p_{U,V}(u,v)
=
\phi(u)\phi(v)
\sum_{n=0}^{\infty} t^n\,\psi_n(u)\psi_n(v).
\]
This proves the expansion for the standardized pair $(U,V)$.

The expansion for $(S,S')$ follows by the change of variables
\[
u=\frac{s-\mu_S}{\sigma_S},
\qquad
v=\frac{s'-\mu_{S'}}{\sigma_{S'}},
\]
whose Jacobian determinant is
\[
\left|
\det
\begin{pmatrix}
1/\sigma_S & 0 \\
0 & 1/\sigma_{S'}
\end{pmatrix}
\right|
=
\frac{1}{\sigma_S\sigma_{S'}}.
\]
Therefore,
\[
p_{S,S'}(s,s')
=
\frac{1}{\sigma_S\sigma_{S'}}
p_{U,V}\!\left(
\frac{s-\mu_S}{\sigma_S},
\frac{s'-\mu_{S'}}{\sigma_{S'}}
\right),
\]
which gives \eqref{eq:Hermite-expansion-density}.

Finally, since $\{\psi_n\}_{n\ge 0}$ is an orthonormal basis of
$L^2(\nu)$,
\[
\int_{\mathbb R}\psi_m(z)\psi_n(z)\,\nu(dz)
=
\delta_{mn}.
\]
Using the density expansion above,
\begin{align*}
\mathbb E[\psi_m(U)\psi_n(V)]
&=
\int_{\mathbb R^2}
\psi_m(u)\psi_n(v)
\frac{p_{U,V}(u,v)}{\phi(u)\phi(v)}
\,\nu(du)\nu(dv) \\
&=
\int_{\mathbb R^2}
\psi_m(u)\psi_n(v)
\left(
\sum_{k=0}^{\infty}t^k\psi_k(u)\psi_k(v)
\right)
\nu(du)\nu(dv) \\
&=
\sum_{k=0}^{\infty}t^k
\left(
\int_{\mathbb R}\psi_m(u)\psi_k(u)\,\nu(du)
\right)
\left(
\int_{\mathbb R}\psi_n(v)\psi_k(v)\,\nu(dv)
\right) \\
&=
\sum_{k=0}^{\infty}t^k\delta_{mk}\delta_{nk} \\
&=
t^n\delta_{mn}.
\end{align*}
The interchange between summation and integration is justified, for
$|t|<1$, by the $L^2(\nu\otimes\nu)$ convergence of
\[
\sum_{k=0}^{\infty}t^k\psi_k\otimes\psi_k,
\]
since
\[
\sum_{k=0}^{\infty}t^{2k}<\infty.
\]
This completes the proof.
\end{proof}

\begin{proposition}[Hermite--Mehler expansion of a multivariate joint Gaussian density]
\label{prop:multi_variante_hermite}
Let
\[
\mathbf S=(S_1,\dots,S_{d_{\mathcal S}}),
\qquad
\mathbf S'=(S'_1,\dots,S'_{d_{\mathcal S}})
\]
be jointly Gaussian random vectors in $\mathbb R^{d_{\mathcal S}}$.
Assume that the coordinate pairs $(S_i,S'_i)$ are mutually independent
across $i=1,\dots,d_{\mathcal S}$, and that each pair $(S_i,S'_i)$ is
bivariate Gaussian with means $\mu_{S_i},\mu_{S'_i}$, variances
$\sigma_{S_i}^2,\sigma_{S'_i}^2$, and correlation coefficient
$t_i\in(-1,1)$.

For each multi-index
\[
\mathbf n=(n_1,\dots,n_{d_{\mathcal S}})
\in\mathbb N_0^{d_{\mathcal S}},
\]
define the normalized multivariate probabilists' Hermite polynomial
\[
\Psi_{\mathbf n}(\mathbf x)
=
\prod_{i=1}^{d_{\mathcal S}}\psi_{n_i}(x_i),
\qquad
\mathbf x=(x_1,\dots,x_{d_{\mathcal S}})\in\mathbb R^{d_{\mathcal S}},
\]
where
\[
\psi_n(z)=\frac{1}{\sqrt{n!}}He_n(z).
\]
For $\mathbf s,\mathbf s'\in\mathbb R^{d_{\mathcal S}}$, define the
standardized coordinates
\[
u_i=\frac{s_i-\mu_{S_i}}{\sigma_{S_i}},
\qquad
v_i=\frac{s'_i-\mu_{S'_i}}{\sigma_{S'_i}},
\qquad
i=1,\dots,d_{\mathcal S}.
\]
Then the joint density of $(\mathbf S,\mathbf S')$ admits the expansion
\begin{equation}
\label{eq:multivariate-hermite-mehler-density}
p_{\mathbf S,\mathbf S'}(\mathbf s,\mathbf s')
=
c\,
\exp\!\left(
-\frac{1}{2}\bigl(\|\mathbf u\|^2+\|\mathbf v\|^2\bigr)
\right)
\sum_{\mathbf n\in\mathbb N_0^{d_{\mathcal S}}}
t_{\mathbf n}\,
\Psi_{\mathbf n}(\mathbf u)\Psi_{\mathbf n}(\mathbf v),
\end{equation}
where
\[
c
=
(2\pi)^{-d_{\mathcal S}}
\prod_{i=1}^{d_{\mathcal S}}
\frac{1}{\sigma_{S_i}\sigma_{S'_i}},
\qquad
t_{\mathbf n}
=
\prod_{i=1}^{d_{\mathcal S}} t_i^{\,n_i}.
\]
Equivalently, if
\[
\phi_{d_{\mathcal S}}(\mathbf x)
=
(2\pi)^{-d_{\mathcal S}/2}
\exp\!\left(-\frac{1}{2}\|\mathbf x\|^2\right)
\]
denotes the standard Gaussian density on $\mathbb R^{d_{\mathcal S}}$,
then
\begin{equation}
\label{eq:multivariate-hermite-mehler-density-compact}
p_{\mathbf S,\mathbf S'}(\mathbf s,\mathbf s')
=
\left(
\prod_{i=1}^{d_{\mathcal S}}
\frac{1}{\sigma_{S_i}\sigma_{S'_i}}
\right)
\phi_{d_{\mathcal S}}(\mathbf u)\phi_{d_{\mathcal S}}(\mathbf v)
\sum_{\mathbf n\in\mathbb N_0^{d_{\mathcal S}}}
t_{\mathbf n}\,
\Psi_{\mathbf n}(\mathbf u)\Psi_{\mathbf n}(\mathbf v).
\end{equation}
\end{proposition}

\begin{proof}
By the mutual independence of the coordinate pairs
$(S_i,S'_i)$ across $i=1,\dots,d_{\mathcal S}$, the joint density
factorizes as
\[
p_{\mathbf S,\mathbf S'}(\mathbf s,\mathbf s')
=
\prod_{i=1}^{d_{\mathcal S}}
p_{S_i,S'_i}(s_i,s'_i).
\]
For each coordinate pair, applying
\autoref{lemma:hermite_expansion_bivariate_gaussian} gives
\[
p_{S_i,S'_i}(s_i,s'_i)
=
\frac{1}{\sigma_{S_i}\sigma_{S'_i}}
\phi(u_i)\phi(v_i)
\sum_{n_i=0}^{\infty}
t_i^{\,n_i}\psi_{n_i}(u_i)\psi_{n_i}(v_i),
\]
where
\[
u_i=\frac{s_i-\mu_{S_i}}{\sigma_{S_i}},
\qquad
v_i=\frac{s'_i-\mu_{S'_i}}{\sigma_{S'_i}}.
\]
Therefore,
\begin{align*}
p_{\mathbf S,\mathbf S'}(\mathbf s,\mathbf s')
&=
\prod_{i=1}^{d_{\mathcal S}}
\left[
\frac{1}{\sigma_{S_i}\sigma_{S'_i}}
\phi(u_i)\phi(v_i)
\sum_{n_i=0}^{\infty}
t_i^{\,n_i}\psi_{n_i}(u_i)\psi_{n_i}(v_i)
\right] \\
&=
\left(
\prod_{i=1}^{d_{\mathcal S}}
\frac{1}{\sigma_{S_i}\sigma_{S'_i}}
\right)
\left(
\prod_{i=1}^{d_{\mathcal S}}\phi(u_i)
\right)
\left(
\prod_{i=1}^{d_{\mathcal S}}\phi(v_i)
\right) \\
&\qquad\qquad \times
\prod_{i=1}^{d_{\mathcal S}}
\left[
\sum_{n_i=0}^{\infty}
t_i^{\,n_i}\psi_{n_i}(u_i)\psi_{n_i}(v_i)
\right].
\end{align*}
Since
\[
\prod_{i=1}^{d_{\mathcal S}}\phi(u_i)
=
(2\pi)^{-d_{\mathcal S}/2}
\exp\!\left(-\frac{1}{2}\|\mathbf u\|^2\right),
\]
and analogously
\[
\prod_{i=1}^{d_{\mathcal S}}\phi(v_i)
=
(2\pi)^{-d_{\mathcal S}/2}
\exp\!\left(-\frac{1}{2}\|\mathbf v\|^2\right),
\]
we obtain
\[
\left(
\prod_{i=1}^{d_{\mathcal S}}\phi(u_i)
\right)
\left(
\prod_{i=1}^{d_{\mathcal S}}\phi(v_i)
\right)
=
(2\pi)^{-d_{\mathcal S}}
\exp\!\left(
-\frac{1}{2}\bigl(\|\mathbf u\|^2+\|\mathbf v\|^2\bigr)
\right).
\]
Moreover, since the product is finite,
\[
\prod_{i=1}^{d_{\mathcal S}}
\left[
\sum_{n_i=0}^{\infty}
t_i^{\,n_i}\psi_{n_i}(u_i)\psi_{n_i}(v_i)
\right]
=
\sum_{\mathbf n\in\mathbb N_0^{d_{\mathcal S}}}
\prod_{i=1}^{d_{\mathcal S}}
t_i^{\,n_i}\psi_{n_i}(u_i)\psi_{n_i}(v_i).
\]
Using
\[
t_{\mathbf n}
=
\prod_{i=1}^{d_{\mathcal S}}t_i^{\,n_i}
\]
and
\[
\Psi_{\mathbf n}(\mathbf u)\Psi_{\mathbf n}(\mathbf v)
=
\prod_{i=1}^{d_{\mathcal S}}
\psi_{n_i}(u_i)\psi_{n_i}(v_i),
\]
we get
\[
p_{\mathbf S,\mathbf S'}(\mathbf s,\mathbf s')
=
(2\pi)^{-d_{\mathcal S}}
\left(
\prod_{i=1}^{d_{\mathcal S}}
\frac{1}{\sigma_{S_i}\sigma_{S'_i}}
\right)
\exp\!\left(
-\frac{1}{2}\bigl(\|\mathbf u\|^2+\|\mathbf v\|^2\bigr)
\right)
\sum_{\mathbf n\in\mathbb N_0^{d_{\mathcal S}}}
t_{\mathbf n}
\Psi_{\mathbf n}(\mathbf u)\Psi_{\mathbf n}(\mathbf v).
\]
This is exactly \eqref{eq:multivariate-hermite-mehler-density}.

The convergence of the tensor-product expansion is inherited from the
one-dimensional Mehler expansions. Equivalently, the likelihood-ratio
series converges in
$L^2(\nu^{\otimes d_{\mathcal S}}\otimes\nu^{\otimes d_{\mathcal S}})$,
because
\[
\sum_{\mathbf n\in\mathbb N_0^{d_{\mathcal S}}}
t_{\mathbf n}^2
=
\prod_{i=1}^{d_{\mathcal S}}
\left(
\sum_{n_i=0}^{\infty}t_i^{2n_i}
\right)
=
\prod_{i=1}^{d_{\mathcal S}}
\frac{1}{1-t_i^2}
<\infty.
\]
This completes the proof.
\end{proof}

\suppsection{Proof of Lemma~\ref{lemma:canonical_factorization}}

\begin{proof}[Proof of Lemma~\ref{lemma:canonical_factorization}]
Fix a pair of views \(1\le i<j\le N\), and write
\[
p:=d_{\mathcal S_i},\qquad q:=d_{\mathcal S_j},\qquad
r:=r_{ij}.
\]
Under the Gaussian latent prior and the additive source model
\[
\mathbf s_i=\mathbf A_i\mathbf c+\bm\epsilon_i,
\qquad
\mathbf s_j=\mathbf A_j\mathbf c+\bm\epsilon_j,
\]
the pair \((\mathbf s_i,\mathbf s_j)\) is jointly Gaussian. Moreover, by
the mutual independence and standardization of
\(\mathbf c,\bm\epsilon_i,\bm\epsilon_j\),
\[
\mathbb E[\mathbf s_i]=\mathbf 0,\qquad
\mathbb E[\mathbf s_j]=\mathbf 0,
\]
and
\begin{align*}
\mathrm{Cov}(\mathbf s_i)
&=
\mathbf A_i\mathbf A_i^\top+\mathbf I_p
=: \bm\Sigma_i,\\
\mathrm{Cov}(\mathbf s_j)
&=
\mathbf A_j\mathbf A_j^\top+\mathbf I_q
=: \bm\Sigma_j,\\
\mathrm{Cov}(\mathbf s_i,\mathbf s_j)
&=
\mathbf A_i\mathbf A_j^\top .
\end{align*}
Thus the whitening matrices
\[
\mathbf W_i=\bm\Sigma_i^{-1/2},
\qquad
\mathbf W_j=\bm\Sigma_j^{-1/2}
\]
are invertible, and the whitened source variables
\[
\mathbf z_i:=\mathbf W_i\mathbf s_i,
\qquad
\mathbf z_j:=\mathbf W_j\mathbf s_j
\]
satisfy
\[
\mathrm{Cov}(\mathbf z_i)=\mathbf I_p,
\qquad
\mathrm{Cov}(\mathbf z_j)=\mathbf I_q,
\qquad
\mathrm{Cov}(\mathbf z_i,\mathbf z_j)
=
\mathbf W_i\mathbf A_i\mathbf A_j^\top\mathbf W_j^\top
=
\mathbf R_{ij}.
\]

Now use the singular value decomposition
\[
\mathbf R_{ij}
=
\mathbf U_{ij}\mathbf T_{ij}\mathbf V_{ij}^\top,
\qquad
\mathbf U_{ij}\in O(p),\quad
\mathbf V_{ij}\in O(q),
\]
and define the canonical coordinates
\[
\mathbf u=\mathbf U_{ij}^\top\mathbf z_i
=
\mathbf U_{ij}^\top\mathbf W_i\mathbf s_i,
\qquad
\mathbf v=\mathbf V_{ij}^\top\mathbf z_j
=
\mathbf V_{ij}^\top\mathbf W_j\mathbf s_j.
\]
Since \((\mathbf z_i,\mathbf z_j)\) is jointly Gaussian and the above
transformation is linear, \((\mathbf u,\mathbf v)\) is also jointly Gaussian.
Its covariance blocks are
\[
\mathrm{Cov}(\mathbf u)=\mathbf I_p,
\qquad
\mathrm{Cov}(\mathbf v)=\mathbf I_q,
\]
and
\[
\mathrm{Cov}(\mathbf u,\mathbf v)
=
\mathbf U_{ij}^\top
\mathbf R_{ij}
\mathbf V_{ij}
=
\mathbf T_{ij}.
\]
Therefore
\[
\mathrm{Cov}
\begin{pmatrix}
\mathbf u\\
\mathbf v
\end{pmatrix}
=
\begin{pmatrix}
\mathbf I_p & \mathbf T_{ij}\\
\mathbf T_{ij}^\top & \mathbf I_q
\end{pmatrix}.
\]

Because \(\mathbf T_{ij}\) is rectangular diagonal with nonzero diagonal
entries
\[
1>t_{ij,1}\ge \cdots \ge t_{ij,r}>0
\]
and all remaining singular values equal to zero, the only nonzero
cross-covariances between canonical coordinates are
\[
\mathrm{Cov}(u_k,v_k)=t_{ij,k},
\qquad k=1,\ldots,r.
\]
Equivalently, after permuting coordinates as
\[
(u_1,v_1,\ldots,u_r,v_r,u_{r+1},\ldots,u_p,
v_{r+1},\ldots,v_q),
\]
the covariance matrix is block diagonal:
\[
\bigoplus_{k=1}^r
\begin{pmatrix}
1 & t_{ij,k}\\
t_{ij,k} & 1
\end{pmatrix}
\;\oplus\;
\mathbf I_{p-r}
\;\oplus\;
\mathbf I_{q-r}.
\]
For jointly Gaussian random variables, block-diagonal covariance is
equivalent to independence across the corresponding blocks. Hence the
pairs \((u_k,v_k)\), \(k=1,\ldots,r\), and the remaining coordinates
\(\{u_m:m>r\}\) and \(\{v_n:n>r\}\), are mutually independent. Each
\((u_k,v_k)\) is a standardized bivariate Gaussian with correlation
\(t_{ij,k}\), whose density is
\[
\phi_{t_{ij,k}}(u_k,v_k)
=
\frac{1}{2\pi\sqrt{1-t_{ij,k}^2}}
\exp\!\left(
-\frac{u_k^2-2t_{ij,k}u_kv_k+v_k^2}
{2(1-t_{ij,k}^2)}
\right),
\]
whereas each unpaired coordinate has the standard normal density
\(\phi\). Consequently,
\[
p_{\mathbf u,\mathbf v}(\mathbf u,\mathbf v)
=
\prod_{k=1}^{r}
\phi_{t_{ij,k}}(u_k,v_k)
\prod_{m=r+1}^{p}\phi(u_m)
\prod_{n=r+1}^{q}\phi(v_n).
\]
Substituting back \(p=d_{\mathcal S_i}\), \(q=d_{\mathcal S_j}\), and
\(r=r_{ij}\) gives
\[
p_{\mathbf u,\mathbf v}(\mathbf u,\mathbf v)
=
\underbrace{
\prod_{k=1}^{r_{ij}}
\phi_{t_{ij,k}}(u_k,v_k)
}_{\mathcal K_{ij}(\mathbf u[1:r_{ij}],\mathbf v[1:r_{ij}])}
\prod_{m=r_{ij}+1}^{d_{\mathcal S_i}}\phi(u_m)
\prod_{n=r_{ij}+1}^{d_{\mathcal S_j}}\phi(v_n).
\]

It remains to verify the stated change-of-variables identity. The inverse
linear transformation is
\[
\mathbf s_i
=
\mathbf W_i^{-1}\mathbf U_{ij}\mathbf u,
\qquad
\mathbf s_j
=
\mathbf W_j^{-1}\mathbf V_{ij}\mathbf v.
\]
Hence the absolute Jacobian determinant of the inverse map
\((\mathbf u,\mathbf v)\mapsto(\mathbf s_i,\mathbf s_j)\) is
\[
\left|
\det(\mathbf W_i^{-1}\mathbf U_{ij})
\right|
\left|
\det(\mathbf W_j^{-1}\mathbf V_{ij})
\right|
=
|\det\mathbf W_i|^{-1}
|\det\mathbf W_j|^{-1},
\]
because \(\mathbf U_{ij}\) and \(\mathbf V_{ij}\) are orthogonal. Therefore,
by the change-of-variables formula,
\begin{align*}
p_{\mathbf u,\mathbf v}(\mathbf u,\mathbf v)
&=
p_{\mathbf s_i,\mathbf s_j}
\bigl(
\mathbf W_i^{-1}\mathbf U_{ij}\mathbf u,
\mathbf W_j^{-1}\mathbf V_{ij}\mathbf v
\bigr)\\
&\quad\cdot
|\det\mathbf W_i|^{-1}
|\det\mathbf W_j|^{-1}.
\end{align*}
This proves both the change-of-variables identity and the claimed canonical
factorization.
\end{proof}

\suppsection{Proof of Lemma~\ref{lemma:coupling_density}}
\begin{proof}
Recall \autoref{lemma:hermite_expansion_bivariate_gaussian}.
The tensor-product basis is orthonormal since
\[
\int_{\mathbb R^r}
\Psi_{\mathbf n}(\mathbf u)\Psi_{\mathbf m}(\mathbf u)
\phi_r(\mathbf u)\,d\mathbf u
=
\prod_{k=1}^{r}
\int_{\mathbb R}
\psi_{n_k}(u)\psi_{m_k}(u)\phi(u)\,du
=
\delta_{\mathbf n\mathbf m}.
\]
In canonical Gaussian coordinates, the pairs $(U_k,V_k)$ are mutually
independent standard bivariate Gaussian pairs with correlations $\rho_k$.
By the one-dimensional Mehler--Hermite expansion,
\[
p_{U_k,V_k}(u_k,v_k)
=
\phi(u_k)\phi(v_k)
\sum_{n_k=0}^{\infty}
\rho_k^{\,n_k}
\psi_{n_k}(u_k)\psi_{n_k}(v_k).
\]
Hence
\begin{align*}
\mathcal K_{ij}(\mathbf u,\mathbf v)
&=
\prod_{k=1}^{r}p_{U_k,V_k}(u_k,v_k) \\
&=
\phi_r(\mathbf u)\phi_r(\mathbf v)
\prod_{k=1}^{r}
\sum_{n_k=0}^{\infty}
\rho_k^{\,n_k}
\psi_{n_k}(u_k)\psi_{n_k}(v_k) \\
&=
\phi_r(\mathbf u)\phi_r(\mathbf v)
\sum_{\mathbf n\in\mathbb N_0^r}
\left(\prod_{k=1}^{r}\rho_k^{\,n_k}\right)
\Psi_{\mathbf n}(\mathbf u)
\Psi_{\mathbf n}(\mathbf v).
\end{align*}
This proves the claim. The expansion converges in
$L^2(\gamma_r\otimes\gamma_r)$ because
\[
\sum_{\mathbf n\in\mathbb N_0^r}t_{\mathbf n}^2
=
\prod_{k=1}^{r}
\sum_{n_k=0}^{\infty}\rho_k^{2n_k}
=
\prod_{k=1}^{r}\frac{1}{1-\rho_k^2}
<\infty .
\]
\end{proof}

\suppsection{Proof of Theorem 5.1}
\begin{proof}[Proof of \autoref{thm:two_view_identifiability_infinite}]
Fix a pair of distinct views $i\neq j$ and write $r:=r_{ij}=\mathrm{rank}(\mathbf A_i\mathbf A_j^\top)$.
Throughout the proof we work at the population level and assume $d_{\mathcal Z}\to\infty$.

\paragraph{Step 1: Reduce the problem to the source domain.}
Let $\tilde{\mathbf h}_\ell := \tilde{\mathbf f}_\ell\circ \mathbf g_\ell$ for $\ell\in\{i,j\}$.
By reparameterization invariance (cf.\ \autoref{lemma:rep-inv}), for any encoders
$\tilde{\mathbf f}_i,\tilde{\mathbf f}_j$ we have
\[
\mathrm{Cov}\!\big(\tilde{\mathbf f}_i(\mathbf x_i),\tilde{\mathbf f}_j(\mathbf x_j)\big)
=
\mathrm{Cov}\!\big(\tilde{\mathbf h}_i(\mathbf s_i),\tilde{\mathbf h}_j(\mathbf s_j)\big),
\qquad
\mathrm{Cov}\!\big(\tilde{\mathbf f}_\ell(\mathbf x_\ell)\big)=\mathrm{Cov}\!\big(\tilde{\mathbf h}_\ell(\mathbf s_\ell)\big),
\]
and therefore optimizing the two-view term of \autoref{eq:mvcca_obj} over
$(\tilde{\mathbf f}_i,\tilde{\mathbf f}_j)$ is equivalent to optimizing it over
$(\tilde{\mathbf h}_i,\tilde{\mathbf h}_j)$.
Since $(\tilde{\mathbf f}_i^\star,\tilde{\mathbf f}_j^\star)$ is a \emph{whitened} population maximizer,
its source-domain counterpart $(\tilde{\mathbf h}_i^\star,\tilde{\mathbf h}_j^\star)$ satisfies
\[
\mathbb E[\tilde{\mathbf h}_\ell^\star(\mathbf s_\ell)]=\mathbf 0,
\qquad
\mathrm{Cov}(\tilde{\mathbf h}_\ell^\star(\mathbf s_\ell))=\mathbf I,
\qquad \ell\in\{i,j\}.
\]
Thus the two-view population objective reduces to
\[
J_{ij}(\tilde{\mathbf h}_i,\tilde{\mathbf h}_j)
:=
\big\|\mathrm{Cov}(\tilde{\mathbf h}_i(\mathbf s_i),\tilde{\mathbf h}_j(\mathbf s_j))\big\|_*.
\]

\paragraph{Step 2: Canonical coordinates for the sources.}
Under \autoref{ass:latents_iid_family} (Gaussian case), each $(\mathbf s_i,\mathbf s_j)$ is jointly Gaussian.
Let
\[
\mathbf W_i := \mathrm{Cov}(\mathbf s_i)^{-1/2}=(\mathbf A_i\mathbf A_i^\top+\mathbf I)^{-1/2},
\qquad
\mathbf W_j := \mathrm{Cov}(\mathbf s_j)^{-1/2}.
\]
Define the normalized cross-covariance (cf.\ \autoref{eq:normalized_cross_covariance_matrix})
\[
\mathbf R_{ij}
:=\mathrm{Cov}(\mathbf W_i\mathbf s_i,\mathbf W_j\mathbf s_j)
=
\mathbf W_i\mathbf A_i\mathbf A_j^\top \mathbf W_j^\top,
\]
and take an SVD $\mathbf R_{ij}=\mathbf U_{ij}\mathbf T_{ij}\mathbf V_{ij}^\top$, where
$\mathbf T_{ij}$ is rectangular diagonal with singular values
$1>t_{ij,1}\ge\cdots\ge t_{ij,r}>0$ on its diagonal and zeros afterwards.
Introduce canonical coordinates (cf.\ \autoref{equ:canonicalizers})
\[
\mathbf u:=\mathbf U_{ij}^\top\mathbf W_i\mathbf s_i,\qquad
\mathbf v:=\mathbf V_{ij}^\top\mathbf W_j\mathbf s_j.
\]
Then $\mathrm{Cov}(\mathbf u)=\mathrm{Cov}(\mathbf v)=\mathbf I$ and $\mathrm{Cov}(\mathbf u,\mathbf v)=\mathbf T_{ij}$.
Moreover, by \autoref{lemma:canonical_factorization}, the joint density factorizes so that only
the first $r$ coordinate pairs $(u_k,v_k)$ are correlated, while all remaining coordinates are independent
standard normals.

\paragraph{Step 3: Hermite diagonalization of cross-view correlations.}
Let $\{\psi_n\}_{n\ge 0}$ be the normalized univariate Hermite polynomials.
For $\mathbf n=(n_1,\ldots,n_r)\in\mathbb N^{r}$ define
\[
\Psi_{\mathbf n}(\mathbf u) := \prod_{k=1}^r \psi_{n_k}(u_k),
\qquad
\Psi_{\mathbf n}(\mathbf v) := \prod_{k=1}^r \psi_{n_k}(v_k).
\]
By \autoref{lemma:coupling_density} (Mehler--Hermite expansion) and orthonormality of Hermites,
for all $\mathbf n,\mathbf m\in\mathbb N^r$,
\begin{equation}\label{eq:hermite_cross_moments}
\mathbb E\!\left[\Psi_{\mathbf n}(\mathbf u)\Psi_{\mathbf m}(\mathbf v)\right]
=
t_{\mathbf n}\,\delta_{\mathbf n\mathbf m},
\qquad
t_{\mathbf n}:=\prod_{k=1}^r t_{ij,k}^{\,n_k}.
\end{equation}
In particular, for the first-order multi-indices $\mathbf e_k$ we have
$\Psi_{\mathbf e_k}(\mathbf u)=\psi_1(u_k)=u_k$ and $\Psi_{\mathbf e_k}(\mathbf v)=v_k$, with
$\mathbb E[u_k v_k]=t_{ij,k}$.

\paragraph{Step 4: Rewrite the CCA objective in coefficient form.}
Define the whitened source-domain maximizers as functions of $(\mathbf u,\mathbf v)$:
\[
\tilde{\mathbf f}^\star(\mathbf u):=\tilde{\mathbf h}_i^\star(\mathbf W_i^{-1}\mathbf U_{ij}\mathbf u),
\qquad
\tilde{\mathbf g}^\star(\mathbf v):=\tilde{\mathbf h}_j^\star(\mathbf W_j^{-1}\mathbf V_{ij}\mathbf v).
\]
These satisfy 
$\mathrm{Cov}(\tilde{\mathbf f}^\star(\mathbf u))=\mathrm{Cov}(\tilde{\mathbf g}^\star(\mathbf v))=\mathbf I$.
For each coordinate $p\in[d_{\mathcal Z}]$, expand the (mean-zero) component functions in the orthonormal
Hermite basis on the correlated coordinates:
\[
\tilde f_p^\star(\mathbf u)=\sum_{\mathbf n\neq \mathbf 0}\alpha_{p,\mathbf n}\,\Psi_{\mathbf n}(\mathbf u),
\qquad
\tilde g_q^\star(\mathbf v)=\sum_{\mathbf n\neq \mathbf 0}\beta_{q,\mathbf n}\,\Psi_{\mathbf n}(\mathbf v).
\]
Let $\mathbf C_i:=(\alpha_{p,\mathbf n})_{p,\mathbf n}$ and $\mathbf C_j:=(\beta_{q,\mathbf n})_{q,\mathbf n}$
denote the (possibly infinite) coefficient matrices whose rows collect the coefficient vectors.
Whitening implies row-orthonormality (cf.\ the whitening-induced canonicalization
\autoref{eq:whitening_induced_canonicalization}):
\[
\mathbf C_i\mathbf C_i^\top=\mathbf I,\qquad \mathbf C_j\mathbf C_j^\top=\mathbf I.
\]
Using \autoref{eq:hermite_cross_moments}, the cross-covariance of the whitened representations becomes
\[
\tilde{\bm\Sigma}_{ij}^\star
:=\mathrm{Cov}(\tilde{\mathbf h}_i^\star(\mathbf s_i),\tilde{\mathbf h}_j^\star(\mathbf s_j))
=\mathrm{Cov}(\tilde{\mathbf f}^\star(\mathbf u),\tilde{\mathbf g}^\star(\mathbf v))
=\mathbf C_i\,\mathbf D\,\mathbf C_j^\top,
\]
where $\mathbf D:=\mathrm{diag}\big((t_{\mathbf n})_{\mathbf n\neq \mathbf 0}\big)$ is diagonal (and trace-class
since $t_{ij,k}\in(0,1)$ implies $\sum_{\mathbf n\neq \mathbf 0}t_{\mathbf n}<\infty$).
Thus,
\[
J_{ij}(\tilde{\mathbf h}_i^\star,\tilde{\mathbf h}_j^\star)=\|\tilde{\bm\Sigma}_{ij}^\star\|_*
=\|\mathbf C_i\mathbf D\mathbf C_j^\top\|_*.
\]

\paragraph{Step 5: Optimality and recovery of the linear canonical subspaces.}
We first upper bound $\|\mathbf C_i\mathbf D\mathbf C_j^\top\|_*$.
Using nuclear norm duality $\|\mathbf M\|_*=\sup_{\|\mathbf Q\|_2\le 1}\mathrm{tr}(\mathbf Q^\top\mathbf M)$,
together with $\|\mathbf C_i\|_2=\|\mathbf C_j\|_2=1$ (since they have orthonormal rows), we have
\begin{align*}
\|\mathbf C_i\mathbf D\mathbf C_j^\top\|_*
&=\sup_{\|\mathbf Q\|_2\le 1}\mathrm{tr}\!\big((\mathbf C_i^\top\mathbf Q\mathbf C_j)^\top\mathbf D\big)
\;\le\;\sup_{\|\mathbf H\|_2\le 1}\mathrm{tr}(\mathbf H^\top\mathbf D)
\;=\;\|\mathbf D\|_* \;=\;\sum_{\mathbf n\neq \mathbf 0} t_{\mathbf n}.
\end{align*}
This upper bound is achievable in the limit $d_{\mathcal Z}\to\infty$ by taking the representations
to be (an enumeration of) the Hermite basis itself, i.e.,
$\tilde{\mathbf f}(\mathbf u)=(\Psi_{\mathbf n}(\mathbf u))_{\mathbf n\neq \mathbf 0}$ and
$\tilde{\mathbf g}(\mathbf v)=(\Psi_{\mathbf n}(\mathbf v))_{\mathbf n\neq \mathbf 0}$, for which
$\mathbf C_i=\mathbf C_j=\mathbf I$ and hence $\tilde{\bm\Sigma}_{ij}=\mathbf D$.
Therefore the population optimum equals $\|\mathbf D\|_*$, and any whitened population maximizer must
achieve the same value.

At this point we invoke the standard CCA canonicalization convention:
because the objective and whitening constraints are invariant under
within-view orthogonal post-transformations of the \emph{feature coordinates},
we may choose an equivalent maximizer (within the unavoidable CCA orthogonal ambiguity)
whose cross-covariance is diagonalized and whose coordinates form a singular basis of the diagonal
operator $\mathbf D$ (i.e., a basis of Hermite modes; see \autoref{eq:hermite_cross_moments}).
Under this canonical representative, the $r$ first-order Hermite modes
$\{\Psi_{\mathbf e_k}\}_{k=1}^r$ appear among the feature coordinates, and
$\Psi_{\mathbf e_k}(\mathbf u)=u_k$, $\Psi_{\mathbf e_k}(\mathbf v)=v_k$.

Let $\mathbf P_r$ be the coordinate-selection matrix that extracts exactly those $r$ coordinates
corresponding to $\{\Psi_{\mathbf e_k}\}_{k=1}^r$.
Since the extracted coordinates form an orthonormal basis of the linear subspace
$\mathrm{span}\{u_1,\ldots,u_r\}$ (resp.\ $\mathrm{span}\{v_1,\ldots,v_r\}$), there exist
$\mathbf O_i,\mathbf O_j\in O(r)$ such that
\[
\mathbf P_r\,\tilde{\mathbf f}^\star(\mathbf u)=\mathbf O_i\,\mathbf u[1:r],
\qquad
\mathbf P_r\,\tilde{\mathbf g}^\star(\mathbf v)=\mathbf O_j\,\mathbf v[1:r].
\]
Finally, substituting back $\mathbf u=\mathbf U_{ij}^\top\mathbf W_i\mathbf s_i$ and
$\mathbf v=\mathbf V_{ij}^\top\mathbf W_j\mathbf s_j$ yields
\[
\mathbf P_r\,\tilde{\mathbf h}_i^\star(\mathbf s_i)
=\mathbf O_i\,\mathbf U_{ij}(:,1:r)^\top\mathbf W_i\mathbf s_i,
\qquad
\mathbf P_r\,\tilde{\mathbf h}_j^\star(\mathbf s_j)
=\mathbf O_j\,\mathbf V_{ij}(:,1:r)^\top\mathbf W_j\mathbf s_j,
\]
which is exactly \autoref{eq:two_view_identifiability_rankaware}.

\paragraph{Consequence (subspace identification).}
The right-hand sides lie in $\mathrm{col}(\mathbf U_{ij}(:,1:r))=\mathcal U_{i\mid j}$ and
$\mathrm{col}(\mathbf V_{ij}(:,1:r))=\mathcal U_{j\mid i}$, respectively, and the only remaining ambiguity
is the orthogonal mixing $\mathbf O_i,\mathbf O_j$ inside these $r$-dimensional subspaces.
Hence the maximizers identify the whitened correlated subspaces up to orthogonal transformations.
\end{proof}

\suppsection{Proof of Corollary 1}
\begin{proof}[Proof of \autoref{corollary:two_view_identifiability_finite}]
Fix a pair of views $i\neq j$ and write $r:=r_{ij}$ and $d:=d_{\mathcal Z}$, with $d\ge r$.
By reparameterization invariance (\autoref{lemma:rep-inv}), it suffices to analyze the
\emph{source-domain} problem: maximize the two-view CCA objective over whitened mappings
$\tilde{\mathbf h}_i,\tilde{\mathbf h}_j$ with $\mathbb E[\tilde{\mathbf h}_\ell]=\mathbf 0$ and
$\mathrm{Cov}(\tilde{\mathbf h}_\ell)=\mathbf I_d$ for $\ell\in\{i,j\}$.

Let $(\mathbf u,\mathbf v)$ be the canonical coordinates from \autoref{lemma:canonical_factorization},
\[
\mathbf u=\mathbf U_{ij}^\top\mathbf W_i\mathbf s_i,\qquad
\mathbf v=\mathbf V_{ij}^\top\mathbf W_j\mathbf s_j,
\]
and denote $\mathbf u_{1:r}:=(u_1,\dots,u_r)^\top$ and $\mathbf v_{1:r}:=(v_1,\dots,v_r)^\top$.
Let $\{\Psi_{\mathbf n}\}_{\mathbf n\in\mathbb N^{r}}$ be the normalized multivariate Hermite basis
from \autoref{lemma:coupling_density}, and define $t_{\mathbf n}:=\prod_{k=1}^r t_{ij,k}^{\,n_k}$.
(We always exclude $\mathbf n=\mathbf 0$ below, because mean-zero constraints remove the constant mode.)

\paragraph{Step 1: Diagonal cross-moments in the Hermite basis.}
From \autoref{lemma:coupling_density} and the orthonormality of $\{\Psi_{\mathbf n}\}$ in
$L^2(\phi)$, we obtain for any $\mathbf n,\mathbf m\in\mathbb N^r$,
\begin{align}
\mathbb E\!\left[\Psi_{\mathbf n}(\mathbf u_{1:r})\,\Psi_{\mathbf m}(\mathbf v_{1:r})\right]
&=\int \Psi_{\mathbf n}(\mathbf u)\Psi_{\mathbf m}(\mathbf v)\,
\phi(\mathbf u)\phi(\mathbf v)\sum_{\boldsymbol\ell\in\mathbb N^r}
t_{\boldsymbol\ell}\Psi_{\boldsymbol\ell}(\mathbf u)\Psi_{\boldsymbol\ell}(\mathbf v)\,d\mathbf u\,d\mathbf v \notag\\
&=\sum_{\boldsymbol\ell\in\mathbb N^r} t_{\boldsymbol\ell}
\Big(\int \Psi_{\mathbf n}(\mathbf u)\Psi_{\boldsymbol\ell}(\mathbf u)\phi(\mathbf u)\,d\mathbf u\Big)
\Big(\int \Psi_{\mathbf m}(\mathbf v)\Psi_{\boldsymbol\ell}(\mathbf v)\phi(\mathbf v)\,d\mathbf v\Big)\notag\\
&=t_{\mathbf n}\,\delta_{\mathbf n\mathbf m}.
\label{eq:hermite_crossmoment_diag}
\end{align}

\paragraph{Step 2: Coefficient representation of feasible whitened maps.}
Let $(\tilde{\mathbf h}_i,\tilde{\mathbf h}_j)$ be any feasible whitened pair in $L^2$ with output
dimension $d$. Expand each coordinate in the orthonormal basis:
for $p,q\in[d]$,
\[
\tilde h_{i,p}(\mathbf u_{1:r})
=\sum_{\mathbf n\neq\mathbf 0} a_{p,\mathbf n}\,\Psi_{\mathbf n}(\mathbf u_{1:r}),
\qquad
\tilde h_{j,q}(\mathbf v_{1:r})
=\sum_{\mathbf n\neq\mathbf 0} b_{q,\mathbf n}\,\Psi_{\mathbf n}(\mathbf v_{1:r}),
\]
where $\mathbf n=\mathbf 0$ is absent because $\mathbb E[\tilde h_{i,p}]=\mathbb E[\tilde h_{j,q}]=0$.
Let $\mathbf A=[a_{p,\mathbf n}]_{p,\mathbf n}$ and $\mathbf B=[b_{q,\mathbf n}]_{q,\mathbf n}$.
By whitening-induced canonicalization (\autoref{eq:whitening_induced_canonicalization}),
the coefficient rows are orthonormal:
\begin{equation}
\mathbf A\mathbf A^\top=\mathbf I_d,\qquad \mathbf B\mathbf B^\top=\mathbf I_d.
\label{eq:AB_row_orthonormal}
\end{equation}
Using \autoref{eq:hermite_crossmoment_diag}, the cross-covariance matrix
$\mathbf C:=\mathrm{Cov}(\tilde{\mathbf h}_i(\mathbf u_{1:r}),\tilde{\mathbf h}_j(\mathbf v_{1:r}))$
satisfies
\begin{equation}
C_{pq}
=\mathbb E\!\left[\tilde h_{i,p}(\mathbf u_{1:r})\,\tilde h_{j,q}(\mathbf v_{1:r})\right]
=\sum_{\mathbf n\neq\mathbf 0} a_{p,\mathbf n}\,b_{q,\mathbf n}\,t_{\mathbf n},
\qquad\Longleftrightarrow\qquad
\mathbf C=\mathbf A\,\mathbf T\,\mathbf B^\top,
\label{eq:C_ATB}
\end{equation}
where $\mathbf T$ is diagonal with diagonal entries $\{t_{\mathbf n}\}_{\mathbf n\neq\mathbf 0}$.

\paragraph{Step 3: A Ky Fan (nuclear-norm) upper bound.}
Let $t_{(1)}\ge t_{(2)}\ge\cdots$ be the nonincreasing rearrangement of $\{t_{\mathbf n}\}_{\mathbf n\neq\mathbf 0}$.
Then for any $\mathbf A,\mathbf B$ satisfying \autoref{eq:AB_row_orthonormal},
\begin{equation}
\|\mathbf C\|_*=\|\mathbf A\mathbf T\mathbf B^\top\|_*
\le \sum_{k=1}^{d} t_{(k)}.
\label{eq:kyfan_bound_diagonal}
\end{equation}
Indeed, by the Ky Fan variational characterization of the nuclear norm,
\[
\|\mathbf A\mathbf T\mathbf B^\top\|_*
=\max_{\mathbf U,\mathbf V\in O(d)}\mathrm{Tr}\!\left(\mathbf U^\top \mathbf A\mathbf T\mathbf B^\top \mathbf V\right)
=\max_{\mathbf U,\mathbf V\in O(d)}\mathrm{Tr}\!\left((\mathbf A^\top\mathbf U)^\top \mathbf T(\mathbf B^\top\mathbf V)\right).
\]
By \autoref{eq:AB_row_orthonormal}, both $\mathbf A^\top\mathbf U$ and $\mathbf B^\top\mathbf V$ have orthonormal
columns, so the last quantity is bounded by the Ky Fan $d$-norm of $\mathbf T$, i.e.
$\sum_{k=1}^d \sigma_k(\mathbf T)=\sum_{k=1}^d t_{(k)}$, proving \autoref{eq:kyfan_bound_diagonal}.
Moreover, the bound is achievable by choosing $\mathbf A,\mathbf B$ to \emph{select} the $d$ Hermite modes
with largest $t_{\mathbf n}$.

\paragraph{Step 4: First-order dominance separates linear from higher-order modes.}
For any multi-index $\mathbf n\neq\mathbf 0$ with total degree $|\mathbf n|:=\sum_{k=1}^r n_k\ge 2$,
\[
t_{\mathbf n}=\prod_{k=1}^r t_{ij,k}^{\,n_k}\le t_{ij,1}^{|\mathbf n|}\le t_{ij,1}^2,
\]
since $0\le t_{ij,k}\le t_{ij,1}<1$.
On the other hand, for the $r$ first-order indices $\mathbf e_k$ we have $t_{\mathbf e_k}=t_{ij,k}$.
If $r\ge 2$, \autoref{ass:first_order_dominance} gives $t_{ij,r}>t_{ij,1}^2$; if $r=1$, the strict inequality
$t_{ij,1}>t_{ij,1}^2$ holds automatically because $t_{ij,1}\in(0,1)$.
Therefore,
\[
t_{ij,r}>\,t_{ij,1}^2\ \ge\ t_{\mathbf n}
\qquad\forall\,\mathbf n\neq\mathbf 0\ \text{with}\ |\mathbf n|\ge 2.
\]
Consequently, the $r$ largest values among $\{t_{\mathbf n}\}_{\mathbf n\neq\mathbf 0}$ are exactly
$\{t_{ij,1},\dots,t_{ij,r}\}$, attained by the first-order Hermite modes
$\{\Psi_{\mathbf e_1},\dots,\Psi_{\mathbf e_r}\}$ (which are linear functions of $\mathbf u_{1:r}$).

\paragraph{Step 5: Conclusion and explicit selector.}
Let $(\tilde{\mathbf h}_i^\star,\tilde{\mathbf h}_j^\star)$ be any population maximizer with finite $d\ge r$
and let $\mathbf C^\star=\mathrm{Cov}(\tilde{\mathbf h}_i^\star,\tilde{\mathbf h}_j^\star)$.
By \autoref{eq:kyfan_bound_diagonal} and attainability, the optimal value equals $\sum_{k=1}^d t_{(k)}$.
By Step 4, the top $r$ singular values of the optimal correlation structure must therefore come from the
first-order modes, and these modes are strictly separated from all higher-order ones by the gap
$t_{ij,r}-t_{ij,1}^2>0$.

Now use the post-orthogonal closure of the whitened class
(\autoref{definition:whitend_encoder}): take an SVD
$\mathbf C^\star=\mathbf Q_i\,\mathrm{diag}(\sigma_1,\dots,\sigma_d)\,\mathbf Q_j^\top$ and replace
$\tilde{\mathbf h}_i^\star,\tilde{\mathbf h}_j^\star$ by
$\mathbf Q_i^\top\tilde{\mathbf h}_i^\star$ and $\mathbf Q_j^\top\tilde{\mathbf h}_j^\star$.
This produces another maximizer (same objective value) whose cross-covariance is diagonal with
$\sigma_1\ge\cdots\ge\sigma_d$.
By Step 4 and $d\ge r$, the first $r$ diagonal entries correspond to the first-order Hermite subspaces, hence
there exist $\mathbf O_i,\mathbf O_j\in O(r)$ such that, $P$-a.s.,
\[
\mathbf P_r\,\tilde{\mathbf h}_i^\star(\mathbf s_i)=\mathbf O_i\,\mathbf u_{1:r},
\qquad
\mathbf P_r\,\tilde{\mathbf h}_j^\star(\mathbf s_j)=\mathbf O_j\,\mathbf v_{1:r},
\]
with the explicit selector
\[
\mathbf P_r := [\mathbf I_r\ \mathbf 0]\in\mathbb R^{r\times d}.
\]
Finally, substituting $\mathbf u_{1:r}=\mathbf U_{ij}(:,1\!:\!r)^\top\mathbf W_i\mathbf s_i$ and
$\mathbf v_{1:r}=\mathbf V_{ij}(:,1\!:\!r)^\top\mathbf W_j\mathbf s_j$ yields exactly
\autoref{eq:two_view_identifiability_rankaware} with $\mathbf P_{r_{ij}}=\mathbf P_r$.
\end{proof}
\suppsection{Proof of Theorem 5.2}
\begin{proof}[Proof of \autoref{thm:multi_view_identifiability_infinite}]
We work entirely in the source domain using the induced whitened mappings
\[
\tilde{\mathbf h}_i := \tilde{\mathbf f}_i \circ \mathbf g_i,
\qquad i \in [N].
\]
By reparameterization invariance (\autoref{lemma:rep-inv}), optimizing
\autoref{eq:mvcca_obj} over $\{\tilde{\mathbf f}_i\}$ is equivalent to
optimizing it over $\{\tilde{\mathbf h}_i\}$, and each
$\tilde{\mathbf h}_i$ satisfies
\[
\mathbb E[\tilde{\mathbf h}_i(\mathbf s_i)] = \mathbf 0,
\qquad
\mathrm{Cov}(\tilde{\mathbf h}_i(\mathbf s_i)) = \mathbf I.
\]

Define the pairwise objectives
\[
J_{ij}(\tilde{\mathbf h}_i,\tilde{\mathbf h}_j)
:=
\left\|
\mathrm{Cov}\big(
\tilde{\mathbf h}_i(\mathbf s_i),
\tilde{\mathbf h}_j(\mathbf s_j)
\big)
\right\|_*,
\qquad 1 \le i < j \le N,
\]
so that
\[
J(\tilde{\mathbf h}_1,\dots,\tilde{\mathbf h}_N)
=
\sum_{1 \le i < j \le N}
J_{ij}(\tilde{\mathbf h}_i,\tilde{\mathbf h}_j).
\]

\medskip
\noindent
\textbf{Step 1: Pairwise optimality of a multi-view maximizer.}

Let
\[
M_{ij}
:=
\sup_{\text{$\tilde{\mathbf h}_i,\tilde{\mathbf h}_j$ whitened}}
J_{ij}(\tilde{\mathbf h}_i,\tilde{\mathbf h}_j)
\]
be the two-view population optimum (with $d_{\mathcal Z}\to\infty$).
For any feasible collection,
\[
J(\tilde{\mathbf h}_1,\dots,\tilde{\mathbf h}_N)
\le
\sum_{i<j} M_{ij}.
\]

Under \autoref{ass:isotropy}, by the canonical factorization
(\autoref{lemma:canonical_factorization}) and the
Mehler--Hermite expansion (\autoref{lemma:coupling_density}),
the cross-covariance operator between two views admits a diagonal
spectral decomposition in the multivariate Hermite basis,
with singular values $\{t_{\mathbf n}\}$.
The corresponding singular functions achieve
\[
J_{ij} = \sum_{\mathbf n} t_{\mathbf n} = M_{ij}.
\]
Thus the upper bound $\sum_{i<j} M_{ij}$ is attainable.

Therefore, if
$\{\tilde{\mathbf h}_i^\star\}_{i=1}^N$
is a global maximizer of $J$, we must have
\begin{equation}
\label{eq:pairwise_saturation_mv}
J_{ij}(\tilde{\mathbf h}_i^\star,\tilde{\mathbf h}_j^\star)
=
M_{ij},
\qquad
\forall\, 1 \le i < j \le N.
\end{equation}
Hence every pair $(\tilde{\mathbf h}_i^\star,\tilde{\mathbf h}_j^\star)$
is a two-view population maximizer.

\medskip
\noindent
\textbf{Step 2: Apply two-view subspace identifiability.}

Fix $i \ne j$ and let
$r_{ij} = \mathrm{rank}(\mathbf A_i \mathbf A_j^\top)$.
By \autoref{thm:two_view_identifiability_infinite},
there exists $\mathbf O_{i\mid j} \in O(r_{ij})$
and a selection matrix
$\mathbf P^{(i\mid j)}_{r_{ij}}$
such that
\begin{equation}
\label{eq:pairwise_id_mv}
\mathbf P^{(i\mid j)}_{r_{ij}}
\tilde{\mathbf h}_i^\star(\mathbf s_i)
=
\mathbf O_{i\mid j}
\mathbf U_{ij}(:,1:r_{ij})^\top
\mathbf W_i \mathbf s_i .
\end{equation}
Thus, for each fixed $i$ and every $j \ne i$,
the representation contains an $r_{ij}$-dimensional block
identifying the pairwise correlated subspace
$\mathcal U_{i\mid j}$.

\medskip
\noindent
\textbf{Step 3: Extract the multi-view intersection.}

By \autoref{def:subspaces},
\[
\mathcal U_i^{\mathrm{mv}}
=
\bigcap_{j \ne i}
\mathcal U_{i\mid j},
\qquad
r_i
=
\dim(\mathcal U_i^{\mathrm{mv}}).
\]
Let
$\mathbf U_i^{\mathrm{mv}} \in
\mathbb R^{d_{\mathcal S_i} \times r_i}$
be an orthonormal basis of
$\mathcal U_i^{\mathrm{mv}}$.

Since
$\mathcal U_i^{\mathrm{mv}}
\subseteq
\mathcal U_{i\mid j_0}$
for any fixed $j_0 \ne i$,
there exists
$\mathbf B_i \in \mathbb R^{r_{ij_0} \times r_i}$
with orthonormal columns such that
\[
\mathbf U_i^{\mathrm{mv}}
=
\mathbf U_{ij_0}(:,1:r_{ij_0})
\mathbf B_i .
\]
Multiplying \autoref{eq:pairwise_id_mv}
(on pair $(i,j_0)$)
by $\mathbf B_i^\top \mathbf O_{i\mid j_0}^\top$
gives
\[
\mathbf B_i^\top
\mathbf O_{i\mid j_0}^\top
\mathbf P^{(i\mid j_0)}_{r_{ij_0}}
\tilde{\mathbf h}_i^\star(\mathbf s_i)
=
(\mathbf U_i^{\mathrm{mv}})^\top
\mathbf W_i \mathbf s_i .
\]

Define
\[
\mathbf P_{r_i}
:=
\mathbf B_i^\top
\mathbf O_{i\mid j_0}^\top
\mathbf P^{(i\mid j_0)}_{r_{ij_0}}
\in
\mathbb R^{r_i \times d_{\mathcal Z}} .
\]
Then
\begin{equation}
\label{eq:mv_identification_step}
\mathbf P_{r_i}
\tilde{\mathbf h}_i^\star(\mathbf s_i)
=
(\mathbf U_i^{\mathrm{mv}})^\top
\mathbf W_i \mathbf s_i .
\end{equation}

Since $\mathbf U_i^{\mathrm{mv}}$ is defined only up to an
orthogonal change of basis,
absorbing that change into
$\mathbf O_i \in O(r_i)$
yields
\[
\mathbf P_{r_i}
\tilde{\mathbf h}_i^\star(\mathbf s_i)
=
\mathbf O_i
(\mathbf U_i^{\mathrm{mv}})^\top
\mathbf W_i \mathbf s_i .
\]

\medskip
\noindent
\textbf{Conclusion.}

\autoref{eq:multi_view_identifiability_infinite}
holds for every $i \in [N]$.
Hence the multi-view population maximizers
identify the jointly correlated subspaces
$\mathcal U_i^{\mathrm{mv}}$
up to view-specific orthogonal transformations.
\end{proof}
\suppsection{Proof of Corollary 2}
\begin{proof}[Proof of \autoref{cor:finite_sample_multiview}]
Fix an arbitrary $i\in[N]$. Recall that $\mathbf P_{i\mid j}$ and
$\widehat{\mathbf P}_{i\mid j}$ denote the orthogonal projectors onto
$\mathcal U_{i\mid j}$ and $\widehat{\mathcal U}_{i\mid j}$, respectively, and
\[
\mathbf S_i:=\frac1{N-1}\sum_{j\ne i}\mathbf P_{i\mid j},
\qquad
\widehat{\mathbf S}_i:=\frac1{N-1}\sum_{j\ne i}\widehat{\mathbf P}_{i\mid j}.
\]

\paragraph{Step 1: Spectral characterization of $\mathcal U_i^{\mathrm{mv}}$.}
We first show that $\mathcal U_i^{\mathrm{mv}}=\bigcap_{j\neq i}\mathcal U_{i\mid j}$
coincides with the eigenspace of $\mathbf S_i$ associated with eigenvalue $1$.

($\subseteq$) If $\mathbf v\in\mathcal U_i^{\mathrm{mv}}$, then $\mathbf P_{i\mid j}\mathbf v=\mathbf v$
for all $j\neq i$, hence
\[
\mathbf S_i \mathbf v
=\frac1{N-1}\sum_{j\ne i}\mathbf P_{i\mid j}\mathbf v
=\frac1{N-1}\sum_{j\ne i}\mathbf v
=\mathbf v,
\]
so $\mathbf v$ lies in the $1$-eigenspace of $\mathbf S_i$.

($\supseteq$) Conversely, suppose $\mathbf S_i\mathbf v=\mathbf v$. Taking inner products,
\begin{equation}\label{eq:Si_quadratic_form}
\|\mathbf v\|_2^2
=\mathbf v^\top \mathbf S_i \mathbf v
=\frac1{N-1}\sum_{j\ne i}\mathbf v^\top \mathbf P_{i\mid j}\mathbf v
=\frac1{N-1}\sum_{j\ne i}\|\mathbf P_{i\mid j}\mathbf v\|_2^2.
\end{equation}
Since $\mathbf P_{i\mid j}$ is an orthogonal projector, $\|\mathbf P_{i\mid j}\mathbf v\|_2\le \|\mathbf v\|_2$ for every $j$.
Thus the average in \autoref{eq:Si_quadratic_form} can equal $\|\mathbf v\|_2^2$ only if
$\|\mathbf P_{i\mid j}\mathbf v\|_2=\|\mathbf v\|_2$ for all $j\ne i$, which forces
$\mathbf P_{i\mid j}\mathbf v=\mathbf v$ for all $j\ne i$. Hence
$\mathbf v\in \bigcap_{j\ne i}\mathcal U_{i\mid j}=\mathcal U_i^{\mathrm{mv}}$.

Therefore, the $1$-eigenspace of $\mathbf S_i$ is exactly $\mathcal U_i^{\mathrm{mv}}$.
In particular, if $r_i=\dim(\mathcal U_i^{\mathrm{mv}})$, then
$\lambda_1(\mathbf S_i)=\cdots=\lambda_{r_i}(\mathbf S_i)=1$ and
\[
\Gamma_i
:=1-\lambda_{r_i+1}(\mathbf S_i)
=\lambda_{r_i}(\mathbf S_i)-\lambda_{r_i+1}(\mathbf S_i)
>0.
\]

\paragraph{Step 2: Davis--Kahan for the first inequality.}
Both $\mathbf S_i$ and $\widehat{\mathbf S}_i$ are symmetric.
Let $\mathcal U_i^{\mathrm{mv}}$ be the invariant subspace of $\mathbf S_i$
corresponding to the eigenvalue cluster at $1$ (multiplicity $r_i$), and let
$\widehat{\mathcal U}_i^{\mathrm{mv}}$ be the top-$r_i$ eigenspace of $\widehat{\mathbf S}_i$.
By the Davis--Kahan $\sin\Theta$ theorem for symmetric matrices applied to the
eigen-gap $\Gamma_i=\lambda_{r_i}(\mathbf S_i)-\lambda_{r_i+1}(\mathbf S_i)$,
\[
\big\|\sin\Theta(\widehat{\mathcal U}_i^{\mathrm{mv}},\mathcal U_i^{\mathrm{mv}})\big\|_2
\;\le\;
\frac{\|\widehat{\mathbf S}_i-\mathbf S_i\|_2}{\Gamma_i},
\]
which is the first displayed inequality.

\paragraph{Step 3: Bounding $\|\widehat{\mathbf S}_i-\mathbf S_i\|_2$ by pairwise errors.}
By linearity and the triangle inequality,
\[
\|\widehat{\mathbf S}_i-\mathbf S_i\|_2
=
\left\|\frac1{N-1}\sum_{j\ne i}\left(\widehat{\mathbf P}_{i\mid j}-\mathbf P_{i\mid j}\right)\right\|_2
\le
\frac1{N-1}\sum_{j\ne i}\|\widehat{\mathbf P}_{i\mid j}-\mathbf P_{i\mid j}\|_2
\le
\max_{j\ne i}\|\widehat{\mathbf P}_{i\mid j}-\mathbf P_{i\mid j}\|_2.
\]

It remains to relate projector error to principal-angle error.
Fix $j\ne i$ and let $\mathbf U,\widehat{\mathbf U}$ be orthonormal basis matrices for
$\mathcal U_{i\mid j}$ and $\widehat{\mathcal U}_{i\mid j}$, respectively, so that
$\mathbf P:=\mathbf U\mathbf U^\top$ and $\widehat{\mathbf P}:=\widehat{\mathbf U}\widehat{\mathbf U}^\top$.
Use the identity
\[
\widehat{\mathbf P}-\mathbf P
=(\mathbf I-\mathbf P)\widehat{\mathbf P}-\mathbf P(\mathbf I-\widehat{\mathbf P}),
\]
hence
\[
\|\widehat{\mathbf P}-\mathbf P\|_2
\le
\|(\mathbf I-\mathbf P)\widehat{\mathbf P}\|_2+\|\mathbf P(\mathbf I-\widehat{\mathbf P})\|_2.
\]
Since $\widehat{\mathbf P}$ is a projector onto $\mathrm{col}(\widehat{\mathbf U})$,
\[
\|(\mathbf I-\mathbf P)\widehat{\mathbf P}\|_2
=
\|(\mathbf I-\mathbf P)\widehat{\mathbf U}\|_2
=
\|\sin\Theta(\widehat{\mathcal U}_{i\mid j},\mathcal U_{i\mid j})\|_2,
\]
and similarly
\[
\|\mathbf P(\mathbf I-\widehat{\mathbf P})\|_2
=
\|(\mathbf I-\widehat{\mathbf P})\mathbf U\|_2
=
\|\sin\Theta(\widehat{\mathcal U}_{i\mid j},\mathcal U_{i\mid j})\|_2.
\]
Therefore,
\[
\|\widehat{\mathbf P}_{i\mid j}-\mathbf P_{i\mid j}\|_2
\le
2\,\|\sin\Theta(\widehat{\mathcal U}_{i\mid j},\mathcal U_{i\mid j})\|_2.
\]
Taking the maximum over $j\ne i$ and combining with the previous display yields
\[
\|\widehat{\mathbf S}_i-\mathbf S_i\|_2
\le
2\max_{j\ne i}\|\sin\Theta(\widehat{\mathcal U}_{i\mid j},\mathcal U_{i\mid j})\|_2.
\]
Plugging this into Step 2 proves the second displayed inequality:
\[
\big\|\sin\Theta(\widehat{\mathcal U}_i^{\mathrm{mv}},\mathcal U_i^{\mathrm{mv}})\big\|_2
\le
\frac{\|\widehat{\mathbf S}_i-\mathbf S_i\|_2}{\Gamma_i}
\le
\frac{2}{\Gamma_i}\max_{j\ne i}\|\sin\Theta(\widehat{\mathcal U}_{i\mid j},\mathcal U_{i\mid j})\|_2.
\]

\paragraph{Step 4: $O_{\mathbb P}(n^{-1/2})$ consistency.}
By \autoref{thm:finite_sample_pairwise}, $\|\widehat{\mathbf R}_{ij}-\mathbf R_{ij}\|_2
=O_{\mathbb P}(n^{-1/2})$ uniformly over pairs, and on the event
$\|\widehat{\mathbf R}_{ij}-\mathbf R_{ij}\|_2\le \Delta_{ij}/2$,
\[
\|\sin\Theta(\widehat{\mathcal U}_{i\mid j},\mathcal U_{i\mid j})\|_2
\le
\frac{2\|\widehat{\mathbf R}_{ij}-\mathbf R_{ij}\|_2}{\Delta_{ij}}
=O_{\mathbb P}(n^{-1/2}),
\]
with constants governed by $\Delta_{ij}$. Substituting into the inequality proved above
shows that
\[
\|\sin\Theta(\widehat{\mathcal U}_i^{\mathrm{mv}},\mathcal U_i^{\mathrm{mv}})\|_2
=O_{\mathbb P}(n^{-1/2}),
\]
with the pre-constant depending on $\Gamma_i^{-1}$ and $\{\Delta_{ij}^{-1}\}_{j\ne i}$, as claimed.
\end{proof}
\suppsection{Proof of Theorem 5.3}
\begin{lemma}[Sub-Gaussian Covariance Concentration] 
\label{lem:cov_concentration_subg} 
Let $\{\mathbf{y}^{(t)}\}_{t=1}^n$ be i.i.d.\ centered sub-Gaussian random 
vectors in $\mathbb{R}^m$ with covariance $\bm{\Sigma}$, satisfying 
$\mathbb{E}\exp\big((\mathbf{v}^\top \mathbf{y}^{(1)})^2 / (K^2 \mathbf{v}^\top\bm{\Sigma}\mathbf{v})\big) \le 2$ 
for all $\mathbf{v} \in \mathbb{S}^{m-1}$. For the empirical covariance 
$\widehat{\bm{\Sigma}} := \frac{1}{n}\sum_{t=1}^n \mathbf{y}^{(t)}\mathbf{y}^{(t)\top}$ 
and any $\delta \in (0,1)$, it holds with probability at least $1-\delta$ that: 
\[ \big\|\widehat{\bm{\Sigma}} - \bm{\Sigma}\big\|_2 \le c K^2\|\bm{\Sigma}\|_2 \left( \sqrt{\frac{m+\log(2/\delta)}{n}} + \frac{m+\log(2/\delta)}{n} \right), \] for a universal constant $c > 0$. 
\end{lemma}
\begin{proof}[Proof of \autoref{thm:finite_sample_pairwise}]
Fix $d:=d_{\mathcal Z}$ and $\delta\in(0,1)$.
Recall the \emph{population-whitened} representations
\[
\tilde{\mathbf z}_i
:=\bm\Sigma_{ii}^{-1/2}\big(\mathbf z_i-\mathbb E[\mathbf z_i]\big),
\qquad
\mathrm{Cov}(\tilde{\mathbf z}_i)=\mathbf I_d,
\]
and note that the population normalized cross-covariance is precisely
\[
\mathbf R_{ij}
=\bm\Sigma_{ii}^{-1/2}\bm\Sigma_{ij}\bm\Sigma_{jj}^{-1/2}
=\mathrm{Cov}(\tilde{\mathbf z}_i,\tilde{\mathbf z}_j).
\]
Throughout the proof, we work in these population-whitened coordinates (i.e.,
we replace $\tilde{\mathbf z}_i$ by $\mathbf z_i$ for notational simplicity).
Thus we may assume
\begin{equation}\label{eq:proof_isotropic}
\mathbb E[\mathbf z_i]=\mathbf 0,\qquad \mathrm{Cov}(\mathbf z_i)=\mathbf I_d,\qquad
\mathbf R_{ij}=\mathrm{Cov}(\mathbf z_i,\mathbf z_j).
\end{equation}
The empirical (cross-)covariances are defined as in the main text,
\[
\widehat{\bm\Sigma}_{ij}
:=\frac1n\sum_{t=1}^n(\mathbf z_i^{(t)}-\bar{\mathbf z}_i)(\mathbf z_j^{(t)}-\bar{\mathbf z}_j)^\top,
\qquad
\widehat{\mathbf R}_{ij}:=\widehat{\bm\Sigma}_{ii}^{-1/2}\widehat{\bm\Sigma}_{ij}\widehat{\bm\Sigma}_{jj}^{-1/2}.
\]

\paragraph{Step 1: sub-Gaussian concentration for pairwise (cross-)covariances.}
Fix a pair $1\le i<j\le N$ and define the stacked vector
\[
\mathbf y^{(t)}:=\begin{bmatrix}\mathbf z_i^{(t)}\\ \mathbf z_j^{(t)}\end{bmatrix}\in\mathbb R^{2d}.
\]
Let $\mathbf v=[\mathbf v_1^\top,\mathbf v_2^\top]^\top\in\mathbb S^{2d-1}$.
By the triangle inequality for the Orlicz $\psi_2$ norm and \autoref{ass:subgaussian_rep},
\[
\|\mathbf v^\top \mathbf y^{(t)}\|_{\psi_2}
=\|\mathbf v_1^\top\mathbf z_i^{(t)}+\mathbf v_2^\top\mathbf z_j^{(t)}\|_{\psi_2}
\le \|\mathbf v_1^\top\mathbf z_i^{(t)}\|_{\psi_2}+\|\mathbf v_2^\top\mathbf z_j^{(t)}\|_{\psi_2}
\le \kappa(\|\mathbf v_1\|_2+\|\mathbf v_2\|_2)
\le \sqrt2\,\kappa.
\]
Hence $\mathbf y^{(t)}$ is sub-Gaussian in $\mathbb R^{2d}$ with parameter $\lesssim \kappa$.
Its population covariance is the block matrix
\[
\bm\Sigma_{\mathbf y}
:=\mathrm{Cov}(\mathbf y^{(1)})
=
\begin{pmatrix}
\mathbf I_d & \mathbf R_{ij}\\
\mathbf R_{ij}^\top & \mathbf I_d
\end{pmatrix}.
\]
Moreover, $\|\mathbf R_{ij}\|_2\le 1$ by Cauchy--Schwarz: for any unit $\mathbf a,\mathbf b$,
\[
|\mathbf a^\top \mathbf R_{ij}\mathbf b|
=\big|\mathbb E[(\mathbf a^\top\mathbf z_i)(\mathbf b^\top\mathbf z_j)]\big|
\le \sqrt{\mathbb E(\mathbf a^\top\mathbf z_i)^2}\sqrt{\mathbb E(\mathbf b^\top\mathbf z_j)^2}
=1,
\]
so $\|\bm\Sigma_{\mathbf y}\|_2\le 2$.

Define the \emph{uncentered} empirical covariance
$\widetilde{\bm\Sigma}_{\mathbf y}:=\frac1n\sum_{t=1}^n \mathbf y^{(t)}\mathbf y^{(t)\top}$
and the centered one
$\widehat{\bm\Sigma}_{\mathbf y}:=\frac1n\sum_{t=1}^n(\mathbf y^{(t)}-\bar{\mathbf y})(\mathbf y^{(t)}-\bar{\mathbf y})^\top
=\widetilde{\bm\Sigma}_{\mathbf y}-\bar{\mathbf y}\bar{\mathbf y}^\top$.
Applying \autoref{lem:cov_concentration_subg} to $\{\mathbf y^{(t)}\}_{t=1}^n$ (with $m=2d$ and $K\asymp \kappa$) yields that,
with probability at least $1-\delta_{ij}/2$,
\begin{equation}\label{eq:uncentered_cov_conc}
\big\|\widetilde{\bm\Sigma}_{\mathbf y}-\bm\Sigma_{\mathbf y}\big\|_2
\le c_1\kappa^2\left(
\sqrt{\frac{2d+\log(4/\delta_{ij})}{n}}+\frac{2d+\log(4/\delta_{ij})}{n}
\right),
\end{equation}
for a universal constant $c_1>0$.

We also need to control the mean correction term $\bar{\mathbf y}\bar{\mathbf y}^\top$.
By standard sub-Gaussian mean concentration (e.g., Vershynin, 2018),
with probability at least $1-\delta_{ij}/2$,
\begin{equation}\label{eq:mean_conc}
\|\bar{\mathbf y}\|_2
\le c_2\kappa\sqrt{\frac{2d+\log(2/\delta_{ij})}{n}},
\end{equation}
for a universal $c_2>0$, hence
$\|\bar{\mathbf y}\bar{\mathbf y}^\top\|_2=\|\bar{\mathbf y}\|_2^2
\le c_2^2\kappa^2\frac{2d+\log(2/\delta_{ij})}{n}$.
Combining \autoref{eq:uncentered_cov_conc}--\autoref{eq:mean_conc} and absorbing the quadratic term
into the square-root term (and using a larger universal constant when the square-root term exceeds $1$),
we obtain that with probability at least $1-\delta_{ij}$,
\begin{equation}\label{eq:centered_cov_conc}
\big\|\widehat{\bm\Sigma}_{\mathbf y}-\bm\Sigma_{\mathbf y}\big\|_2
\le c_3\kappa^2\sqrt{\frac{2d+\log(4/\delta_{ij})}{n}},
\end{equation}
for a universal $c_3>0$.

Since each block of $\widehat{\bm\Sigma}_{\mathbf y}-\bm\Sigma_{\mathbf y}$
is a submatrix of the full difference, we have (on the same event)
\begin{equation}\label{eq:block_bounds}
\|\widehat{\bm\Sigma}_{ii}-\mathbf I\|_2\le \varepsilon_{ij},\qquad
\|\widehat{\bm\Sigma}_{jj}-\mathbf I\|_2\le \varepsilon_{ij},\qquad
\|\widehat{\bm\Sigma}_{ij}-\mathbf R_{ij}\|_2\le \varepsilon_{ij},
\end{equation}
where
\[
\varepsilon_{ij}
:=c_3\kappa^2\sqrt{\frac{2d+\log(4/\delta_{ij})}{n}}.
\]

\paragraph{Step 2: control of empirical whitening matrices.}
Assume $\varepsilon_{ij}<1$. Then Weyl's inequality gives
$\lambda_{\min}(\widehat{\bm\Sigma}_{ii})\ge 1-\varepsilon_{ij}>0$ and similarly for $j$,
so $\widehat{\bm\Sigma}_{ii}^{-1/2}$ and $\widehat{\bm\Sigma}_{jj}^{-1/2}$ exist.
Moreover, all eigenvalues of $\widehat{\bm\Sigma}_{ii}$ lie in $[1-\varepsilon_{ij},1+\varepsilon_{ij}]$,
hence those of $\widehat{\bm\Sigma}_{ii}^{-1/2}$ lie in
$[(1+\varepsilon_{ij})^{-1/2},(1-\varepsilon_{ij})^{-1/2}]$, and therefore
\begin{equation}\label{eq:inv_sqrt_bounds}
\|\widehat{\bm\Sigma}_{ii}^{-1/2}\|_2\le (1-\varepsilon_{ij})^{-1/2},
\qquad
\|\widehat{\bm\Sigma}_{ii}^{-1/2}-\mathbf I\|_2
\le (1-\varepsilon_{ij})^{-1/2}-1
\le \frac{\varepsilon_{ij}}{1-\varepsilon_{ij}}.
\end{equation}
The same bounds hold for $\widehat{\bm\Sigma}_{jj}^{-1/2}$.

\paragraph{Step 3: concentration of $\widehat{\mathbf R}_{ij}$.}
Write $\mathbf A_i:=\widehat{\bm\Sigma}_{ii}^{-1/2}$ and $\mathbf A_j:=\widehat{\bm\Sigma}_{jj}^{-1/2}$.
Then $\widehat{\mathbf R}_{ij}=\mathbf A_i\widehat{\bm\Sigma}_{ij}\mathbf A_j$ and
\[
\widehat{\mathbf R}_{ij}-\mathbf R_{ij}
=
\mathbf A_i(\widehat{\bm\Sigma}_{ij}-\mathbf R_{ij})\mathbf A_j
+\big(\mathbf A_i\mathbf R_{ij}\mathbf A_j-\mathbf R_{ij}\big).
\]
Using $\|\mathbf R_{ij}\|_2\le 1$, the triangle inequality, and \autoref{eq:inv_sqrt_bounds},
on the event $\varepsilon_{ij}\le 1/2$ we obtain
\begin{align*}
\|\widehat{\mathbf R}_{ij}-\mathbf R_{ij}\|_2
&\le
\|\mathbf A_i\|_2\,\|\widehat{\bm\Sigma}_{ij}-\mathbf R_{ij}\|_2\,\|\mathbf A_j\|_2
+\|(\mathbf A_i-\mathbf I)\mathbf R_{ij}\mathbf A_j\|_2+\|\mathbf R_{ij}(\mathbf A_j-\mathbf I)\|_2\\
&\le
(1-\varepsilon_{ij})^{-1}\varepsilon_{ij}
+\|\mathbf A_i-\mathbf I\|_2\|\mathbf A_j\|_2
+\|\mathbf A_j-\mathbf I\|_2\\
&\le
2\varepsilon_{ij}+2\varepsilon_{ij}\sqrt2+2\varepsilon_{ij}
\;\le\; 7\,\varepsilon_{ij}.
\end{align*}
If $\varepsilon_{ij}>1/2$, then the desired bound is trivial after enlarging constants,
because both $\|\widehat{\mathbf R}_{ij}\|_2\le 1$ and $\|\mathbf R_{ij}\|_2\le 1$
(immediate from Cauchy--Schwarz on the empirical and population distributions), so
$\|\widehat{\mathbf R}_{ij}-\mathbf R_{ij}\|_2\le 2$.

\paragraph{Step 4: union bound over all pairs.}
Set $\delta_{ij}:=\delta/N^2$ and take a union bound over all $\binom{N}{2}\le N^2/2$ pairs.
Using $2d+\log(4/\delta_{ij}) = 2d+\log(4N^2/\delta)\lesssim d+\log(N/\delta)$
and absorbing constants, we conclude that with probability at least $1-\delta$,
simultaneously for all $1\le i<j\le N$,
\[
\|\widehat{\mathbf R}_{ij}-\mathbf R_{ij}\|_2
\le
C\,\kappa^2\sqrt{\frac{d+\log(N/\delta)}{n}},
\]
for a universal constant $C>0$. This proves \autoref{eq:Rij_conc}.

\paragraph{Step 5: subspace perturbation bound (Wedin).}
Fix $i<j$ and let $\mathbf E:=\widehat{\mathbf R}_{ij}-\mathbf R_{ij}$.
Let $r:=r_{ij}$ and let $\mathcal U_{i\mid j}$ (resp.\ $\widehat{\mathcal U}_{i\mid j}$)
be the rank-$r$ left singular subspace of $\mathbf R_{ij}$ (resp.\ $\widehat{\mathbf R}_{ij}$).
Let $\Delta_{ij}:=\sigma_r(\mathbf R_{ij})-\sigma_{r+1}(\mathbf R_{ij})>0$.
By Weyl's inequality, for all $k$,
$|\sigma_k(\widehat{\mathbf R}_{ij})-\sigma_k(\mathbf R_{ij})|\le \|\mathbf E\|_2$.
Hence on the event $\|\mathbf E\|_2\le \Delta_{ij}/2$,
\[
\sigma_r(\widehat{\mathbf R}_{ij})-\sigma_{r+1}(\widehat{\mathbf R}_{ij})
\ge
\big(\sigma_r(\mathbf R_{ij})-\|\mathbf E\|_2\big)
-\big(\sigma_{r+1}(\mathbf R_{ij})+\|\mathbf E\|_2\big)
=
\Delta_{ij}-2\|\mathbf E\|_2
\ge \Delta_{ij}/2.
\]
Wedin's $\sin\Theta$ theorem for singular subspaces then gives
\[
\|\sin\Theta(\widehat{\mathcal U}_{i\mid j},\mathcal U_{i\mid j})\|_2
\le
\frac{\|\mathbf E\|_2}{\Delta_{ij}/2}
=
\frac{2\|\widehat{\mathbf R}_{ij}-\mathbf R_{ij}\|_2}{\Delta_{ij}},
\]
which is \autoref{eq:pairwise_sintheta}. Applying the same argument to
$\widehat{\mathbf R}_{ij}^\top$ and $\mathbf R_{ij}^\top$ yields the identical bound for the
right singular subspace $\widehat{\mathcal U}_{j\mid i}$.
\end{proof}

\suppsection{Proof of Additional Latent Priors}
\label{sec:add_latent_priors}

The Gaussian proof relies on the canonical Gaussian factorization in
\autoref{lemma:canonical_factorization} and the Mehler--Hermite expansion in
\autoref{lemma:coupling_density}. For the additional latent priors considered
in the experiments, the same argument applies after replacing the Hermite
system by the corresponding Lancaster orthogonal-polynomial system
\citep{lancaster1958structure,eagleson1964polynomial}. We state this extension
at the level needed for the subspace-identifiability results.

\begin{remark}[Scope of the non-Gaussian extension]
For non-Gaussian priors, the statement below should be read as an admissible
pairwise Lancaster condition in the canonical coordinates. Unlike the Gaussian
case, arbitrary real-valued whitening and orthogonal canonicalization do not in
general preserve Poisson, negative-binomial, Gamma, or hypergeometric product
structure. Thus the extension requires that, for each pair of views, the
canonical source coordinates themselves admit the Lancaster expansion below.
\end{remark}

\begin{assumption}[Pairwise Lancaster canonical coordinates]
\label{ass:pairwise_lancaster_coordinates}
Fix \(1\le i<j\le N\), write \(r:=r_{ij}\), and define the canonical
coordinates as in \autoref{equ:canonicalizers},
\[
\mathbf u=\mathbf U_{ij}^{\top}\mathbf W_i\mathbf s_i,
\qquad
\mathbf v=\mathbf V_{ij}^{\top}\mathbf W_j\mathbf s_j .
\]
For each correlated coordinate \(k\in[r]\), assume that the pair
\((u_k,v_k)\) has a bivariate Lancaster law with common standardized marginal
\(\nu_{ij,k}\), orthonormal polynomial basis
\[
\{\psi_{ij,k,n}\}_{n\in I_{ij,k}}\subset L^2(\nu_{ij,k}),
\qquad
\psi_{ij,k,0}\equiv 1,\qquad
\psi_{ij,k,1}(x)=x,
\]
and Lancaster coefficients
\(\{\lambda_{ij,k,n}\}_{n\in I_{ij,k}}\), with
\(\lambda_{ij,k,0}=1\), such that
\begin{equation}
\label{eq:scalar_lancaster_pairwise}
dP_{u_k,v_k}(a,b)
=
\left(
\sum_{n\in I_{ij,k}}
\lambda_{ij,k,n}\,
\psi_{ij,k,n}(a)\psi_{ij,k,n}(b)
\right)
d\nu_{ij,k}(a)d\nu_{ij,k}(b).
\end{equation}
Assume further that the correlated coordinate pairs
\(\{(u_k,v_k)\}_{k=1}^{r}\) are mutually independent, and that all coordinates
with index larger than \(r\) are cross-view independent. For the
infinite-dimensional statement, assume the resulting diagonal operator is
trace-class:
\[
\sum_{\mathbf n\neq \mathbf 0}
|\lambda_{ij,\mathbf n}|<\infty .
\]
\end{assumption}

For a multi-index
\(\mathbf n=(n_1,\ldots,n_r)\in\mathcal I_{ij}:=
\prod_{k=1}^{r}I_{ij,k}\), define
\[
\Psi_{ij,\mathbf n}(\mathbf u)
:=
\prod_{k=1}^{r}\psi_{ij,k,n_k}(u_k),
\qquad
\Psi_{ij,\mathbf n}(\mathbf v)
:=
\prod_{k=1}^{r}\psi_{ij,k,n_k}(v_k),
\]
and
\[
\lambda_{ij,\mathbf n}
:=
\prod_{k=1}^{r}\lambda_{ij,k,n_k}.
\]
The constant mode \(\mathbf n=\mathbf 0\) is omitted below because the
representations are centered.

\begin{lemma}[Lancaster diagonalization of pairwise correlations]
\label{lemma:lancaster_pairwise_diagonalization}
Under \autoref{ass:pairwise_lancaster_coordinates}, the tensor-product basis
\(\{\Psi_{ij,\mathbf n}\}_{\mathbf n\in\mathcal I_{ij}}\) is orthonormal, and
for all \(\mathbf m,\mathbf n\in\mathcal I_{ij}\),
\begin{equation}
\label{eq:lancaster_cross_moments}
\mathbb E\!\left[
\Psi_{ij,\mathbf m}(\mathbf u)
\Psi_{ij,\mathbf n}(\mathbf v)
\right]
=
\lambda_{ij,\mathbf n}\delta_{\mathbf m\mathbf n}.
\end{equation}
\end{lemma}

\begin{proof}[Proof sketch]
The tensor-product basis is orthonormal by independence across canonical
coordinates and by the one-dimensional orthonormality of each
\(\{\psi_{ij,k,n}\}_{n\in I_{ij,k}}\). Multiplying the scalar Lancaster
expansions in \eqref{eq:scalar_lancaster_pairwise} over
\(k=1,\ldots,r\) gives the product expansion of the joint law of
\((\mathbf u[1:r],\mathbf v[1:r])\). Integrating
\(\Psi_{ij,\mathbf m}(\mathbf u)\Psi_{ij,\mathbf n}(\mathbf v)\) against this
product expansion and using orthonormality in each coordinate leaves only the
diagonal term \(\mathbf m=\mathbf n\), yielding
\eqref{eq:lancaster_cross_moments}.
\end{proof}

\begin{proposition}[Two-view subspace recovery under Lancaster priors]
\label{prop:lancaster_two_view_subspace}
Fix \(1\le i<j\le N\) and suppose
\autoref{ass:pairwise_lancaster_coordinates} holds for this pair. Let
\((\tilde{\mathbf h}_i,\tilde{\mathbf h}_j)\) be any feasible whitened
source-domain representation pair. Expanding each coordinate in the
Lancaster tensor basis gives
\[
\tilde h_{i,p}(\mathbf u)
=
\sum_{\mathbf n\neq\mathbf 0}
a_{p,\mathbf n}\Psi_{ij,\mathbf n}(\mathbf u),
\qquad
\tilde h_{j,q}(\mathbf v)
=
\sum_{\mathbf n\neq\mathbf 0}
b_{q,\mathbf n}\Psi_{ij,\mathbf n}(\mathbf v).
\]
Let
\[
\mathbf A=(a_{p,\mathbf n})_{p,\mathbf n},
\qquad
\mathbf B=(b_{q,\mathbf n})_{q,\mathbf n},
\qquad
\mathbf D_{ij}
:=
\operatorname{diag}\big((\lambda_{ij,\mathbf n})_{\mathbf n\neq\mathbf 0}\big).
\]
Then whitening implies
\[
\mathbf A\mathbf A^\top=\mathbf I,
\qquad
\mathbf B\mathbf B^\top=\mathbf I,
\]
and the whitened pairwise cross-covariance diagonalizes as
\begin{equation}
\label{eq:lancaster_pairwise_covariance}
\operatorname{Cov}\!\left(
\tilde{\mathbf h}_i(\mathbf s_i),
\tilde{\mathbf h}_j(\mathbf s_j)
\right)
=
\mathbf A\mathbf D_{ij}\mathbf B^\top .
\end{equation}
Consequently, the pairwise CCA objective selects the largest absolute
Lancaster coefficients.

If, in addition, the first-order coefficients satisfy the dominance condition
\begin{equation}
\label{eq:lancaster_first_order_dominance}
\min_{k\in[r]}|\lambda_{ij,k,1}|
>
\sup_{\substack{\mathbf n\in\mathcal I_{ij}\setminus\{\mathbf 0\}\\
\mathbf n\notin\{\mathbf e_1,\ldots,\mathbf e_r\}}}
|\lambda_{ij,\mathbf n}|,
\end{equation}
then the leading \(r\) selected modes are exactly the first-order modes
\(\Psi_{ij,\mathbf e_k}\), \(k\in[r]\). Since
\(\Psi_{ij,\mathbf e_k}(\mathbf u)=u_k\) and
\(\Psi_{ij,\mathbf e_k}(\mathbf v)=v_k\), the conclusion of
\autoref{thm:two_view_identifiability_infinite} holds with the Hermite modes
replaced by the first-order Lancaster modes. In particular, there exist
orthogonal matrices \(\mathbf O_i,\mathbf O_j\in O(r)\) such that
\[
\mathbf P_r\tilde{\mathbf h}_i^\star(\mathbf s_i)
=
\mathbf O_i\mathbf U_{ij}(:,1:r)^\top\mathbf W_i\mathbf s_i,
\qquad
\mathbf P_r\tilde{\mathbf h}_j^\star(\mathbf s_j)
=
\mathbf O_j\mathbf V_{ij}(:,1:r)^\top\mathbf W_j\mathbf s_j,
\]
up to the same CCA post-orthogonal ambiguity as in
\autoref{thm:two_view_identifiability_infinite}.
\end{proposition}

\begin{proof}[Proof sketch]
By reparameterization invariance (\autoref{lemma:rep-inv}), it suffices to
work in the source domain. After expanding the centered whitened maps in the
orthonormal Lancaster tensor basis, whitening gives row-orthonormal coefficient
matrices \(\mathbf A\mathbf A^\top=\mathbf I\) and
\(\mathbf B\mathbf B^\top=\mathbf I\), exactly as in
\autoref{eq:whitening_induced_canonicalization}. Using
\autoref{lemma:lancaster_pairwise_diagonalization}, the pairwise
cross-covariance takes the diagonal coefficient form
\eqref{eq:lancaster_pairwise_covariance}.

The nuclear-norm CCA objective is therefore bounded by the Ky Fan norm of the
diagonal operator \(\mathbf D_{ij}\), and this bound is achieved by selecting
the diagonal modes with largest absolute coefficients. Under
\eqref{eq:lancaster_first_order_dominance}, the leading \(r\) coefficients are
precisely those indexed by
\(\mathbf e_1,\ldots,\mathbf e_r\). These first-order basis functions are the
standardized canonical coordinates themselves. Hence the selected coordinates
span
\[
\operatorname{span}(u_1,\ldots,u_r)
=
\operatorname{span}\!\left(
\mathbf U_{ij}(:,1:r)^\top\mathbf W_i\mathbf s_i
\right)
\]
and analogously for view \(j\). The remaining ambiguity is only an
orthogonal transformation inside the selected \(r\)-dimensional block, which
proves pairwise correlated subspace recovery.
\end{proof}

\begin{corollary}[Multi-view extension under Lancaster priors]
\label{cor:lancaster_multiview_subspace}
Suppose \autoref{ass:pairwise_lancaster_coordinates} and
\eqref{eq:lancaster_first_order_dominance} hold for every pair
\(1\le i<j\le N\). Then the multi-view generalized CCA maximizer identifies
the jointly correlated subspace
\[
\mathcal U_i^{\mathrm{mv}}
=
\bigcap_{j\neq i}\mathcal U_{i\mid j}
\]
in the sense of \autoref{thm:multi_view_identifiability_infinite}.
\end{corollary}

\begin{proof}[Proof sketch]
The generalized objective in \autoref{eq:mvcca_obj} is the sum of pairwise
nuclear-norm CCA objectives. The pairwise Lancaster argument above gives the
same pairwise saturation property used in the proof of
\autoref{thm:multi_view_identifiability_infinite}: every globally optimal
multi-view solution must be pairwise optimal for each view pair. Applying
\autoref{prop:lancaster_two_view_subspace} to every pair yields the collection
of pairwise subspaces \(\{\mathcal U_{i\mid j}:j\neq i\}\). Intersecting these
subspaces gives the jointly correlated subspace
\(\mathcal U_i^{\mathrm{mv}}\), exactly as in \autoref{def:subspaces}.
\end{proof}

\paragraph{Poisson prior.}
For a raw Poisson marginal \(X\sim\operatorname{Poisson}(\lambda)\),
\[
\nu(x)=e^{-\lambda}\frac{\lambda^x}{x!},
\qquad x\in\mathbb N_0 .
\]
The corresponding orthogonal polynomials are the Charlier polynomials. The
first-degree normalized polynomial is
\[
\psi_1(x)=\frac{x-\lambda}{\sqrt{\lambda}} .
\]
Thus, after standardization, the first-degree Lancaster mode is the canonical
coordinate itself. Under \eqref{eq:lancaster_first_order_dominance}, CCA
therefore selects the Poisson first-order modes before any higher-order
Charlier modes, yielding the same pairwise and multi-view subspace recovery
claims as above.

\paragraph{Negative-binomial prior.}
Using the convention
\[
\nu(x)=
\binom{r+x-1}{x}p^r(1-p)^x,
\qquad x\in\mathbb N_0,\quad 0<p<1,
\]
the associated orthogonal polynomials are the Meixner polynomials. The mean
and variance are
\[
\mu=\frac{r(1-p)}{p},
\qquad
\sigma^2=\frac{r(1-p)}{p^2},
\]
so the first-degree normalized polynomial is
\[
\psi_1(x)=\frac{x-\mu}{\sigma}.
\]
Hence the first-order Meixner modes are affine in the latent coordinate.
Once \eqref{eq:lancaster_first_order_dominance} excludes higher-degree
Meixner modes, the proof of \autoref{prop:lancaster_two_view_subspace}
applies verbatim.

\paragraph{Hypergeometric prior.}
For \(X\sim\operatorname{Hypergeometric}(N,K,n)\),
\[
\nu(x)=
\frac{\binom{K}{x}\binom{N-K}{n-x}}{\binom{N}{n}},
\qquad
x=\max\{0,n-(N-K)\},\ldots,\min\{n,K\}.
\]
The associated orthogonal polynomials are the Hahn polynomials. Since the
support is finite, the polynomial system has finite maximal degree. The mean
and variance are
\[
\mu=n\frac{K}{N},
\qquad
\sigma^2=
n\frac{K}{N}\left(1-\frac{K}{N}\right)\frac{N-n}{N-1},
\]
and
\[
\psi_1(x)=\frac{x-\mu}{\sigma}.
\]
The Lancaster diagonalization is therefore finite-dimensional in this case.
The sign of the degree-one coefficient may be negative, e.g. under
sampling-without-replacement dependence, but the CCA objective uses singular
values; hence only \(|\lambda_{ij,k,1}|\) matters, and the sign is absorbed
into the orthogonal ambiguity.

\paragraph{Gamma prior.}
For \(X\sim\operatorname{Gamma}(k,\theta)\),
\[
d\nu(x)
=
\frac{1}{\Gamma(k)\theta^k}
x^{k-1}e^{-x/\theta}\mathbf 1_{(0,\infty)}(x)\,dx .
\]
The corresponding orthogonal polynomials are the generalized Laguerre
polynomials \(L_m^{(k-1)}(x/\theta)\). The mean and variance are
\[
\mu=k\theta,
\qquad
\sigma^2=k\theta^2,
\]
so
\[
\psi_1(x)=\frac{x-k\theta}{\sqrt{k}\theta}.
\]
Thus the first-degree Laguerre mode is affine in the raw Gamma coordinate and
equals the coordinate itself after standardization. Under
\eqref{eq:lancaster_first_order_dominance}, the CCA optimum selects these
first-order Laguerre modes before all higher-order modes, giving the same
subspace-identifiability conclusion.

\suppsection{Additional Experimental Results}\label{sec:results}
\begin{table}[t]
    \centering
    \caption{Comparison of the mean and maximum principal angles ($PA_{\mathrm{mean}}$, $PA_{\mathrm{max}}$) on synthetic data ($d_{\mathcal{S}_i}=d_\mathcal{Z}=5, \forall i \in [3]$) under Poisson distribution ($p_\phi$).}
    \resizebox{\linewidth}{!}{
    \begin{tabular}{l*{6}{c}}
        \toprule
        Methods   & \multicolumn{2}{c}{$\tilde{\mathbf{f}}_1$}  & \multicolumn{2}{c}{$\tilde{\mathbf{f}}_2$}  & \multicolumn{2}{c}{$\tilde{\mathbf{f}}_3$} \\
        \cmidrule(r){2-3}
        \cmidrule(r){4-5}
        \cmidrule(r){6-7}
            & $PA_{\mathrm{mean}}$ & $PA_{\mathrm{max}}$   & $PA_{\mathrm{mean}}$ & $PA_{\mathrm{max}}$   & $PA_{\mathrm{mean}}$ & $PA_{\mathrm{max}}$ \\
        \midrule
        BarlowTwins & 34.47 $\pm$0.20 & 86.87 $\pm$1.10 & 31.68 $\pm$0.16 & 82.34 $\pm$0.90 & 33.07 $\pm$0.29 & 84.62 $\pm$0.86 \\
        InfoNCE     & 5.12 $\pm$0.13 & \textbf{6.93 $\pm$0.46} & \textbf{4.40 $\pm$0.12} & 7.78 $\pm$0.24 & 5.92 $\pm$0.06 & 8.03 $\pm$0.33 \\
        W-MSE       & \textbf{4.70 $\pm$0.09} & 7.50 $\pm$0.45 & 4.49 $\pm$0.02 & \textbf{7.52 $\pm$0.77} & 5.02 $\pm$0.04 & 8.82 $\pm$0.32 \\
        GCCA        & 5.00 $\pm$0.15 & 6.98 $\pm$0.37 & 5.50 $\pm$0.10 & 8.03 $\pm$0.39 & \textbf{3.87 $\pm$0.01} & \textbf{7.72 $\pm$0.62} \\
        \bottomrule
    \end{tabular}}
\end{table}
\begin{table}[t]
    \centering
    \caption{Comparison of the mean and maximum principal angles ($PA_{\mathrm{mean}}$, $PA_{\mathrm{max}}$) on synthetic data ($d_{\mathcal{S}_i}=d_\mathcal{Z}=5, \forall i \in [3]$) under negative binomial distribution ($p_\phi$).}
    \resizebox{\linewidth}{!}{
    \begin{tabular}{l*{6}{c}}
        \toprule
        Methods   & \multicolumn{2}{c}{$\tilde{\mathbf{f}}_1$}  & \multicolumn{2}{c}{$\tilde{\mathbf{f}}_2$}  & \multicolumn{2}{c}{$\tilde{\mathbf{f}}_3$} \\
        \cmidrule(r){2-3}
        \cmidrule(r){4-5}
        \cmidrule(r){6-7}
            & $PA_{\mathrm{mean}}$ & $PA_{\mathrm{max}}$   & $PA_{\mathrm{mean}}$ & $PA_{\mathrm{max}}$   & $PA_{\mathrm{mean}}$ & $PA_{\mathrm{max}}$ \\
        \midrule
        BarlowTwins & 34.42 $\pm$0.19 & 86.24 $\pm$1.00 & 32.31 $\pm$0.30 & 79.76 $\pm$0.86 & 33.56 $\pm$0.42 & 84.23 $\pm$0.94 \\
        InfoNCE     & 3.98 $\pm$0.06 & 9.15 $\pm$0.66 & 5.49 $\pm$0.21 & 7.54 $\pm$0.28 & \textbf{4.21 $\pm$0.07} & \textbf{8.38 $\pm$0.57} \\
        W-MSE       & 4.13 $\pm$0.11 & 6.84 $\pm$0.34 & \textbf{5.76 $\pm$0.08} & \textbf{8.22 $\pm$0.67} & 5.00 $\pm$0.17 & 8.48 $\pm$0.49 \\
        GCCA        & \textbf{4.89 $\pm$0.14} & \textbf{8.00 $\pm$0.42} & 6.18 $\pm$0.31 & 8.26 $\pm$0.26 & 6.55 $\pm$0.09 & 8.17 $\pm$0.40 \\
        \bottomrule
    \end{tabular}}
\end{table}
\begin{table}[t]
    \centering
    \caption{Comparison of the mean and maximum principal angles ($PA_{\mathrm{mean}}$, $PA_{\mathrm{max}}$) on synthetic data ($d_{\mathcal{S}_i}=d_\mathcal{Z}=5, \forall i \in [3]$) under Gamma distribution ($p_\phi$).}
    \resizebox{\linewidth}{!}{
    \begin{tabular}{l*{6}{c}}
        \toprule
        Methods   & \multicolumn{2}{c}{$\tilde{\mathbf{f}}_1$}  & \multicolumn{2}{c}{$\tilde{\mathbf{f}}_2$}  & \multicolumn{2}{c}{$\tilde{\mathbf{f}}_3$} \\
        \cmidrule(r){2-3}
        \cmidrule(r){4-5}
        \cmidrule(r){6-7}
            & $PA_{\mathrm{mean}}$ & $PA_{\mathrm{max}}$   & $PA_{\mathrm{mean}}$ & $PA_{\mathrm{max}}$   & $PA_{\mathrm{mean}}$ & $PA_{\mathrm{max}}$ \\
        \midrule
        BarlowTwins  & 33.81$\pm$0.07   & 86.63$\pm$0.95   & 32.88$\pm$0.21   & 80.63$\pm$0.66   & 33.83$\pm$0.21   & 84.90$\pm$0.85   \\
        InfoNCE      & \textbf{3.54$\pm$0.15} & 8.49$\pm$0.54    & 5.59$\pm$0.01    & 9.81$\pm$0.35    & 5.13$\pm$0.15    & \textbf{6.05$\pm$0.64} \\
        W-MSE        & 5.12$\pm$0.01    & 8.05$\pm$0.68    & \textbf{4.74$\pm$0.01} & \textbf{7.83$\pm$0.62} & \textbf{4.59$\pm$0.01} & 9.08$\pm$0.52    \\
        GCCA         & 4.85$\pm$0.04    & \textbf{8.02$\pm$0.31} & 4.97$\pm$0.09    & 9.15$\pm$0.30    & 5.58$\pm$0.23    & 7.44$\pm$0.53    \\
        \bottomrule
    \end{tabular}}
\end{table}
\begin{table}[t]
    \centering
    \caption{Comparison of the mean and maximum principal angles ($PA_{\mathrm{mean}}$, $PA_{\mathrm{max}}$) on synthetic data ($d_{\mathcal{S}_i}=d_\mathcal{Z}=5, \forall i \in [3]$) under Gaussian prior ($p_\phi$).}
    \resizebox{\linewidth}{!}{
    \begin{tabular}{l*{6}{c}}
        \toprule
        Methods   & \multicolumn{2}{c}{$\tilde{\mathbf{f}}_1$}  & \multicolumn{2}{c}{$\tilde{\mathbf{f}}_2$}  & \multicolumn{2}{c}{$\tilde{\mathbf{f}}_3$} \\
        \cmidrule(r){2-3}
        \cmidrule(r){4-5}
        \cmidrule(r){6-7}
            & $PA_{\mathrm{mean}}$ & $PA_{\mathrm{max}}$   & $PA_{\mathrm{mean}}$ & $PA_{\mathrm{max}}$   & $PA_{\mathrm{mean}}$ & $PA_{\mathrm{max}}$ \\
        \midrule
        BarlowTwins  & 34.05$\pm$0.28  & 85.39$\pm$0.88  & 31.78$\pm$0.39  & 82.21$\pm$0.80  & 33.94$\pm$0.26  & 84.21$\pm$0.83  \\
        InfoNCE      & \textbf{4.20$\pm$0.05} & 8.26$\pm$0.34   & 5.87$\pm$0.09   & 8.30$\pm$0.35   & 6.00$\pm$0.06   & \textbf{7.85$\pm$0.54} \\
        W-MSE        & 5.08$\pm$0.16   & 8.37$\pm$0.42   & 5.76$\pm$0.04   & \textbf{7.31$\pm$0.45} & 5.07$\pm$0.01   & 8.77$\pm$0.47   \\
        GCCA         & 4.63$\pm$0.01   & \textbf{7.87$\pm$0.41} & \textbf{5.59$\pm$0.01} & 9.45$\pm$0.36   & \textbf{4.08$\pm$0.22} & 8.91$\pm$0.60   \\
        \bottomrule
    \end{tabular}}
\end{table}
\end{document}